\documentclass[11pt]{article}

\usepackage[margin=1in]{geometry}
\usepackage{amsmath,amssymb,amsthm}
\usepackage{graphicx}
\usepackage{booktabs}
\usepackage[plain,noend]{algorithm2e}
\usepackage{newtxtext}
\usepackage[subscriptcorrection]{newtxmath}

\usepackage[numbers,sort&compress]{natbib}
\makeatletter
\DeclareRobustCommand{\citet}{\@ifstar{\citep}{\citep}}
\DeclareRobustCommand{\Citet}{\@ifstar{\Citep}{\Citep}}
\makeatother

\usepackage[final]{microtype}   
\usepackage[hypertexnames=false]{hyperref}

\theoremstyle{plain}
\newtheorem{theorem}{Theorem}
\newtheorem{lemma}{Lemma}
\newtheorem{proposition}{Proposition}

\newtheorem{assumption}{Assumption}

\theoremstyle{definition}

\newtheorem{example}{Example}

\theoremstyle{remark}
\newtheorem{remark}{Remark}

\renewcommand{\P}[1]{\mathbb{P}\left[#1\right]}
\newcommand{\E}[1]{\mathbb{E}\left[#1\right]}
\newcommand{\I}[1]{\mathbb{I}\left[#1\right]}

\newcommand{\joker}{\circledast}

\usepackage{placeins}

\makeatletter
\renewcommand{\algocf@captiontext}[2]{#1\algocf@typo. \AlCapFnt{}#2}

\def\@algocf@capt@plain{top}
\renewcommand{\algocf@makecaption}[2]{%
  \addtolength{\hsize}{\algomargin}%
  \sbox\@tempboxa{\algocf@captiontext{#1}{#2}}%
  \ifdim\wd\@tempboxa >\hsize%
    \hskip .5\algomargin%
    \parbox[t]{\hsize}{\algocf@captiontext{#1}{#2}}%
  \else%
    \global\@minipagefalse%
    \hbox to\hsize{\box\@tempboxa}%
  \fi%
  \addtolength{\hsize}{-\algomargin}%
}
\makeatother

\title{Conformal Inference for Open-Set and Imbalanced Classification}

\author{
  Tianmin Xie$^{1}$,
  Yanfei Zhou$^{2}$,
  Ziyi Liang$^{3}$,
  Stefano Favaro$^{4}$,
  Matteo Sesia$^{5}$\\
}

\date{
  {\small
  $^{1}$Department of Data Sciences and Operations,\\University of Southern California, Los Angeles, California, U.S.A.\\[0.25em]
  $^{2}$Meta Platforms, Menlo Park, California, U.S.A.\\[0.25em]
  $^{3}$Department of Statistics, University of California, Santa Barbara, California, U.S.A.\\[0.25em]
  $^{4}$Dipartimento di Scienze Economico-Sociali e Matematico-Statistiche,\\Università di Torino and Collegio Carlo Alberto, Torino, Italy\\[0.25em]
  $^{5}$Departments of Data Sciences and Operations, and of Computer Science,\\University of Southern California, Los Angeles, California, U.S.A.\\[4pt]
  \texttt{tianminx@marshall.usc.edu, yanfeizhou@meta.com, ziyiliang@ucsb.edu, stefano.favaro@unito.it, sesia@marshall.usc.edu}
  }
}

\begin{document}
\maketitle

\begin{abstract}
This paper studies conformal prediction for highly imbalanced and open-set classification, where the data are exchangeable but there are many possible classes and not all may be observed. We show that existing conformal methods implicitly assume that all labels are observed, so they may undercover when new labels arise at test time; moreover, they make inefficient use of imbalanced data, sometimes producing uninformative prediction sets.
To address the first limitation, we introduce conformal $p$-values for testing whether a test point has a new label, study their optimality, and connect them to classical work by Good and Turing. Building on these, we develop a conformal prediction method that guarantees marginal coverage under exchangeability alone, together with an asymptotically valid adjustment that reduces conservativeness.
To address the second limitation, we propose a selective sample splitting scheme that prioritizes rare labels for training, preserving finite-sample validity through reweighting.
Experiments with simulated and real data show that our method performs well even with infinitely many possible labels.


\end{abstract}


\bigskip

\section{Introduction}

\subsection{Background and Motivation}

Real-world classification often involves many classes and the possibility that a test point belongs to a new, previously unseen class.
This task, known in machine learning as {\em open-set} classification or {\em recognition} \citep{geng2020recent}, is harder than standard classification, which relies on the {\em closed-set} assumption that all relevant labels appear in the observed data.
A well-known example is face recognition, where individuals must be identified from a gallery of known identities, although some test images may show people never seen before \citep{li2005open}.
Similar problems arise in voice recognition \citep{jleed2020open}, biometric identification \citep{sun2022open}, medicine \citep{cao2023open}, forensic statistics \citep{bra10}, ecology \citep{shen2024open}, cybersecurity \citep{cruz2017open}, and anomaly detection \citep{bergman2020classification}.

In practice, open-set classification often relies on sophisticated models such as support vector machines \citep{junior2021open}, random forests \citep{feng2024open}, and deep neural networks \citep{Bendale2016}, including transformers \citep{bartusiak2022transformer}. 
While powerful, however, such models inevitably make mistakes and are not designed to rigorously quantify uncertainty, which can lead to overconfident use.

To address this issue, we build on {\em conformal prediction} \citep{vovk2005algorithmic}, a flexible uncertainty quantification framework that constructs prediction sets from the output of any ``black-box'' model.
While the validity of conformal prediction is usually attributed to data exchangeability alone, existing methods for classification also implicitly assume a closed-set setting, since their prediction sets can only contain observed labels.
Yet, exchangeability does not rule out the possibility that new labels may arise at test time, as explained in more detail below.
This motivates developing a conformal classification method that produces informative prediction sets with guaranteed finite-sample marginal coverage under exchangeability, even in open-set settings.

\subsection{Preview of Contributions and Paper Outline}

Developing such a method involves several technical challenges.
First, to motivate our approach, in Section~\ref{sec:preliminaries} we analyze existing conformal classification methods when the population distribution has support over many labels, not all of which may appear in the observed data, and we quantify the resulting coverage violation under exchangeability.

Second, in Section~\ref{sec:testing} we introduce a family of ``hypothesis testing'' problems where the hypotheses are indicators of random events: the null concerns how many times the true label of a new test point appears in an exchangeable sample, including the key special case that the label is new. We develop statistics designed to test these hypotheses, we show that they satisfy a super-uniformity property analogous to that enjoyed by the {\em conformal $p$-values} of \citet{jin2023selection}, and we study their optimality. We also uncover a connection to the classical Good--Turing estimator \citep{good1953population}, which motivates calling these statistics {\em conformal Good--Turing $p$-values}.

Third, in Section~\ref{sec:open-set-classification} we leverage conformal Good--Turing $p$-values within a novel classification algorithm that constructs prediction sets with marginal coverage even in open-set scenarios.
The algorithm first tests whether the test point has a previously seen label, and then outputs a set of likely labels, including, if necessary, a ``joker'' symbol representing all unseen labels.
This is natural in categorical problems, where labels carry no meaning beyond whether they are the same or different, and yields informative sets as long as the joker is included parsimoniously.
We refine this method through two extensions: a missing-mass adjustment of the significance levels that reduces conservativeness and is asymptotically valid, and a data-driven strategy to approximate the optimal allocation of the significance level between the testing and prediction steps.
Moreover, we show that, under exchangeability alone, if new labels become rare in the large-sample limit, our method approaches standard closed-set conformal prediction, thus protecting against open-set scenarios at no asymptotic cost.
In addition, we propose a label-frequency-based data splitting strategy that, unlike uniform random splitting, makes efficient use of severely imbalanced data, while preserving coverage through suitable reweighting.
This selective splitting substantially improves the informativeness of our prediction sets in practice.
It is also motivated theoretically: when new labels persist in the large-sample limit, we show that, under additional consistency conditions, our method approaches an ideal oracle, but only when combined with selective splitting.

Finally, in Section~\ref{sec:empirical} we demonstrate on synthetic and real data that our method achieves valid coverage in challenging open-set regimes where existing methods fail, while often producing substantially more informative prediction sets than approaches based on uniform random sample splitting.
Section~\ref{sec:discussion} concludes, and all proofs are in Appendix~\ref{app:proofs}.

\subsection{Related Work} \label{sec:related-work}

The literature on open-set recognition (see \citet{geng2020recent} for a survey) has prioritized classification accuracy over finite-sample guarantees.
Since open-set performance is closely tied to closed-set accuracy \citep{vaze2022openset}, the key additional challenge is to detect when a test point may have a new label.
Many methods ``decline'' to classify inputs that look novel \citep{scheirer2012toward}, for example by modeling tail probabilities of rare labels with extreme value theory \citep{Scheirer2014}, or by adding layers to deep neural networks to estimate the probability of an unknown class \citep{Bendale2016}, later combined with generative modeling \citep{ge2017generative, Yoshihashi2019, Oza2019}.
Few studies quantify uncertainty; an exception is the Bayesian approach of \citet{HolUE2024}.

A natural alternative to our method is to apply standard split-conformal prediction directly to an open-set classifier, treating its ``unknown'' output as an additional class (Appendix~\ref{sec:why-not-open-set-classifier} and Algorithm~\ref{alg:open-set-plug-in}).
This approach achieves marginal coverage, but it treats as novel any label absent from the training set, even if it appears in the calibration data, which makes the prediction sets less informative.
Our method avoids this problem by treating as novel only labels absent from all available data.
Open-set classifiers can still be used within our method, although using them to classify seen labels provides little benefit over strong closed-set classifiers, consistent with \citet{vaze2022openset}; a different use within our framework would be to supply the novelty scores of the feature-based $p$-values introduced in Section~\ref{sec:pvals-feature-based}, which depend on a score only through its ranks.

Relevant prior work in conformal prediction has mostly focused on closed-set classification. Existing methods differ mainly in their nonconformity scores \citep{hechtlinger2018cautious,romano2020classification}, and our framework can leverage any of them.
Methods seeking stronger guarantees than marginal coverage include those targeting label-conditional coverage \citep{sadinle2019least} and, more generally, Mondrian conformal prediction \citep{vovk2005algorithmic}.
These stronger guarantees, however, typically require many samples per class; for settings with many classes and scarce data, \citet{Ding2023manyclasses} relaxes them by clustering similar labels and assigning rare classes to a null cluster. Their method is closed-set but complementary to ours, and the two could be combined. For simplicity, in this paper we focus on marginal coverage.

Beyond conformal classification, our work is related to but distinct from set-valued classification with abstention \citep{hechtlinger2018cautious, Guan2022classificationOOD, wang2023classificationOOD} and conformal out-of-distribution testing or novelty detection \citep{bates2023testing, liang2024integrative, marandon2024adaptive}.
Those works assume that the labeled data come from an ``inlier'' distribution, while test points may come from a different ``outlier'' distribution.
By contrast, we assume that all data are exchangeable, and we ask instead how many times the test label appears in the data.
Our approach is inspired by conformal selection \citep{jin2023selection}, which also tests random hypotheses about unobserved outcomes, albeit different ones.

\section{Preliminaries} \label{sec:preliminaries}

\subsection{Setup and Problem Statement} \label{sec:setup}

We consider $n+1$ paired data points $Z_i = (X_i, Y_i)$, for $i \in [n+1] := \{1, \ldots, n+1\}$, where $X_i$ denotes a feature vector in a (possibly high-dimensional) space $\mathcal{X}$ and $Y_i$ is a categorical label taking values in a dictionary $\mathcal{Y}$, so that $Z_i \in \mathcal{Z} := \mathcal{X} \times \mathcal{Y}$.
Formally, $\mathcal{Y}$ may be any standard Borel space (finite, countably infinite, or uncountable), so that events involving label ties, such as $\{Y_i = Y_j\}$, are measurable.
Labels are treated as purely nominal: our inferential target and all validity guarantees depend on the labels only through which observations share a label, so no metric, ordering, or other structure on $\mathcal{Y}$ is required.
We assume the sequence $Z = (Z_1, \ldots, Z_{n+1})$ is \emph{exchangeable}, meaning that its joint distribution is invariant under permutations.
A classical example with infinitely many possible labels is random sampling from the Dirichlet process \citep{ferguson1973bayesian,bla73}: given $Y_1, \ldots, Y_n$, the next label $Y_{n+1}$ is new with probability $\theta/(\theta+n)$, for a concentration parameter $\theta > 0$, and otherwise equals a previously observed label $y$ with probability $n_y/(\theta+n)$, where $n_y$ is the number of times $y$ has appeared (see Appendix~\ref{app:dp-example} for details, including the feature distribution).
Our methodology, however, does not rely on any particular generative model.

Having observed $Z_{1:n} = \left( (X_1, Y_1), \ldots, (X_n, Y_n) \right)$, the goal is to predict the label $Y_{n+1}$ of a new sample with features $X_{n+1}$.
To account for uncertainty, we seek a reasonably small and informative prediction set $\hat{C}_{\alpha}(X_{n+1}) \subseteq \mathcal{Y}$, implicitly depending on $Z_{1:n}$, guaranteeing {\em marginal coverage}:
\begin{align} \label{eq:def-marginal-coverage}
  \P{Y_{n+1} \in \hat{C}_{\alpha}(X_{n+1}) } \geq 1-\alpha.
\end{align}
This probability is taken with respect to the randomness in all variables, $(X_1, Y_1), \ldots, (X_{n+1}, Y_{n+1})$, while $\alpha \in (0,1)$ is a fixed significance level; e.g., $\alpha = 0.1$.
Since $\mathcal{Y}$ may be unknown and possibly infinite, $\hat{C}_{\alpha}(X_{n+1})$ will be represented in practice using a ``joker'' symbol as a placeholder for all unseen labels, as explained in more detail in Section~\ref{sec:cgtc}.

Marginal coverage is a standard goal in conformal inference, ensuring that prediction sets contain the true label with the desired average frequency across the population. Stronger {\em feature-conditional coverage}, targeting $\P{Y_{n+1} \in \hat{C}_{\alpha}(X_{n+1}) \mid X_{n+1} = x}$ for almost all $x \in \mathcal{X}$, is unattainable without stronger distributional assumptions or infinite data \citep{foygel2021limits}. We therefore focus on marginal coverage, while aiming for good conditional performance by leveraging strong predictive models and well-designed nonconformity scores \citep{sesia2026elements}.

\subsection{From Closed-Set Conformal Classification to Open-Set Scenarios} \label{sec:methods-closed-set}

Existing conformal prediction methods, which assume the label dictionary $\mathcal{Y}$ is finite and known, operate as follows.
For each candidate label $y \in \mathcal{Y}$, they consider the augmented dataset $\mathcal{D}(y) := \{(X_1,Y_1),\ldots,(X_n,Y_n),(X_{n+1},y)\}$ and compute \emph{nonconformity scores}, scalar statistics quantifying how unusual an observation looks relative to the rest of the data: $S_i^{(y)} = s((X_i,Y_i); \mathcal{D}(y))$ for $i \in [n]$ and $S_{n+1}^{(y)} = s((X_{n+1},y); \mathcal{D}(y))$, using a {\em symmetric} score function $s$, i.e., one invariant to the order of the data in $\mathcal{D}(y)$.
A label is then included in the prediction set if its test score is not unusually large relative to the others:
\begin{align} \label{eq:standard-prediction-set}
  \hat{C}_{\alpha}(X_{n+1}; \mathcal{Y}) := \left\{ y \in \mathcal{Y} : S_{n+1}^{(y)} \leq \hat{q}_{1-\alpha}^{(y)} \right\},
\end{align}
where $\hat{q}_{1-\alpha}^{(y)}$ is the $\lceil (1-\alpha)(n+1) \rceil$-th smallest value in the multiset $\{S_1^{(y)}, \ldots, S_n^{(y)}, +\infty\}$.
By symmetry, the scores $S_1^{(Y_{n+1})}, \ldots, S_{n+1}^{(Y_{n+1})}$ evaluated at $y = Y_{n+1}$ are exchangeable, so the rank of $S_{n+1}^{(Y_{n+1})}$ among them is (super-)uniformly distributed, which implies~\eqref{eq:def-marginal-coverage}.

Two key methodological choices involving the score function are worth recalling. The first concerns how the output of a black-box classification model is mapped into scalar nonconformity scores. A classical and simple option is the negative estimated conditional probability of $y$ given $x$; i.e., $S_i^{(y)} = -\hat{\mathbb{P}}[Y = Y_i \mid X=X_i]$ for $i \in [n]$ and $S_{n+1}^{(y)} = -\hat{\mathbb{P}}[Y = y \mid X=X_{n+1}]$, where $\hat{\mathbb{P}}$ is an estimate of the conditional probability, e.g., from the final soft-max layer of a deep neural network. This score approximately minimizes the expected size of the prediction sets \citep{sadinle2019least}.
An alternative applies rank-based transformations to the estimated class probabilities, yielding generalized inverse quantile scores \citep{romano2020classification}. If the model accurately estimates the true conditional class probabilities, these scores lead to prediction sets with not only marginal but also feature-conditional coverage (see Appendix~\ref{app:closed-set-classification}).

The second choice concerns which data points are used for training the model and which for computing nonconformity scores. In the {\em full conformal} approach, the augmented dataset $\mathcal{D}(y)$ is used for both, so the model must be re-trained for each candidate label, which can be computationally demanding. In the more practical {\em split conformal} approach, the model is pre-trained once on an independent training set, so that the scores $s((X_i,Y_i); \mathcal{D}(y))$ and $s((X_{n+1},y); \mathcal{D}(y))$ do not depend on $\mathcal{D}(y)$.
Algorithm~\ref{alg:standard-closed-set} in Appendix~\ref{app:closed-set-classification} outlines split-conformal classification with generalized inverse quantile scores in the closed-set setting.

   For any of these methods, if the scores are almost surely distinct, conformal prediction sets are not overly conservative, as they satisfy the almost-matching upper bound $\P{Y_{n+1} \in \hat{C}_{\alpha}(X_{n+1}; \mathcal{Y})} \leq 1 - \alpha + \min\left\{1/(n+1), \alpha \right\}$.
Our methods and theory apply to any of these frameworks, but our empirical demonstrations focus on split conformal prediction with generalized inverse quantile scores, which is computationally efficient and often produces informative sets with relatively good conditional coverage.

The first question we consider is how existing conformal methods, such as Algorithm~\ref{alg:standard-closed-set}, behave in an open-set scenario.
Since the full label space $\mathcal{Y}$ is unknown, it is natural to replace it with the set of observed labels, $\mathcal{Y}_n := \text{unique}(Y_1, \ldots, Y_n)$, obtaining the plug-in prediction set $\hat{C}_{\alpha}(X_{n+1}; \mathcal{Y}_{n})$; see Algorithm~\ref{alg:standard-plug-in} in Appendix~\ref{app:closed-set-classification}.
What coverage can be expected from it?
Intuitively, its performance must depend on the probability of the event that the next label is new,
\begin{align} \label{eq:event-new-label}
  N_{n+1} := \{\, Y_{n+1} \notin \mathcal{Y}_n \,\}.
\end{align}

\begin{theorem} \label{theorem:coverage-standard}
Assume $\{(X_i,Y_i)\}_{i=1}^{n+1}$ are exchangeable and the scores $\{S_1^{(Y_{n+1})},\ldots,S_{n+1}^{(Y_{n+1})}\}$ are almost surely distinct, with $\mathcal{Y}$ unknown and potentially infinite.
Then,
\begin{align*}
  1 - \alpha - \P{N_{n+1}} \leq
  \P{Y_{n+1} \in \hat{C}_{\alpha}(X_{n+1}; \mathcal{Y}_{n})}
  \leq 1 - \alpha + \min\left\{\frac{1}{n+1}, \alpha - \P{N_{n+1}} \right\}.
\end{align*}
Moreover, $\P{Y_{n+1} \notin \hat{C}_{\alpha}(X_{n+1}; \mathcal{Y}_n), Y_{n+1} \in \mathcal{Y}_n} \leq \alpha$.
\end{theorem}

The distinctness assumption above is mild: it holds for continuously distributed scores and can always be enforced by breaking ties with independent random noise.
If $\P{N_{n+1}} = 0$, Theorem~\ref{theorem:coverage-standard} reduces to the familiar bounds for closed-set conformal classification. 
In general, however, $\hat{C}_{\alpha}(X_{n+1}; \mathcal{Y}_{n})$ may under-cover, and does so if $\P{N_{n+1}} > \alpha$. 

For example, under the Dirichlet process model of Section~\ref{sec:setup}, $\P{N_{n+1}}=\theta/(\theta + n) > 0$ for any finite $n$, so $\P{Y_{n+1} \in \hat{C}_{\alpha}(X_{n+1}; \mathcal{Y}_{n})} \leq 1 - \theta/(\theta + n)$; e.g., the upper bound is $0.5$ with $\theta = n = 100$, which is far below the nominal $0.9$ for $\alpha = 0.1$.

\section{Testing Random Hypotheses about Label Frequencies}
\label{sec:testing}

\subsection{Setup} \label{sec:pvals-setup}

We introduce a family of ``hypotheses'' about the frequency of $Y_{n+1}$ among $Y_{1:n} = (Y_1, \ldots, Y_n)$.
Unlike classical hypotheses, which are statements about fixed population parameters, these hypotheses are random events. As in prior works on conformal prediction \citep{jin2023selection,angelopoulos2024theoretical}, we still call them hypotheses, since the analogy with hypothesis testing helps explain our methodology, even if it is not entirely rigorous \citep{sesia2026elements}.

For $y \in \mathcal{Y}$, let $N(y; Y_{1:n}) := \sum_{i=1}^n \mathbf{1}\{Y_i = y\}$, and for each $k \in \{0\} \cup \mathbb{N}$ define the event
\begin{align} \label{eq:freq-hypothesis}
H_k: N(Y_{n+1}; Y_{1:n}) = k.
\end{align}
For example, $H_0$ holds if and only if $Y_{n+1}$ is a new label.
To test these hypotheses, we seek statistics $\psi_k$ that are small when $H_k$ is false and satisfy
\begin{align} \label{eq:pval-validity}
  \P{\psi_k \leq u, H_k} \leq u, \qquad \forall u \in (0,1).
\end{align}
This condition bounds the joint probability that $H_k$ is true and rejected by $\psi_k$, rather than the conditional probability of rejection given $H_k$. We call such statistics {\em conformal $p$-values}, following the literature on conformal selection \citep{jin2023selection}, which also uses this term for statistics with a similar property.
While~\eqref{eq:pval-validity} is weaker than the super-uniformity of classical $p$-values, or of the conformal $p$-values of \citet{bates2023testing}, which holds under a fixed null hypothesis, it is exactly what is needed to guarantee the coverage of our prediction sets (Section~\ref{sec:open-set-classification}).

\subsection{Feature-Blind Conformal Good--Turing $p$-Values} \label{sec:pvals-feature-blind}

We begin by constructing conformal $p$-values $\psi_k$ that depend only on the observed labels $Y_{1:n}$, postponing the use of features to the next subsection.
Since labels carry no meaning beyond equality (Section~\ref{sec:setup}), it is natural to look at $Y_{1:n}$ only through the lens of the sample frequency profile $\mathcal{F}_n = (M_1, \ldots, M_n)$, where $M_k$ is the number of {\em distinct} labels appearing exactly $k$ times in $Y_{1:n}$, so that $\sum_{k=1}^{n}k M_k = n$; we also set $M_{n+1} := 0$.

For each $k \in \{0\} \cup \mathbb{N}$, we define the {\em feature-blind} conformal Good--Turing $p$-value
\begin{align}
\label{eq:p-value-GT}
\psi_{k}^{\text{GT}} = \frac{(k+1) M_{k+1} + k + 1}{n + 1}.
\end{align}
The following result shows that $\psi_{k}^{\text{GT}}$ is valid for testing $H_k$, and optimal among all valid conformal $p$-values that are deterministic functions of the frequency profile.

\begin{theorem}
\label{theorem:GT-super-uniform}
If $Y_{1:(n+1)}$ are exchangeable, $\psi_{k}^{\text{GT}}$ satisfies~\eqref{eq:pval-validity} for any $k \in \{0\} \cup \mathbb{N}$.
Moreover, if $\psi'_k$ is any deterministic function of $\mathcal{F}_n$ satisfying~\eqref{eq:pval-validity} under every exchangeable distribution of $Y_{1:(n+1)}$, then $\psi'_{k} \geq \psi_{k}^{\text{GT}}$ almost surely whenever $k = 0$ or $M_k \geq 1$.
\end{theorem}

The $p$-value in~\eqref{eq:p-value-GT} is closely connected to the classical Good--Turing estimator \citep{good1953population}, which estimates the total probability mass of species observed exactly $k$ times in a sample of size $n$ as $(k+1)M_{k+1}/n$; for $k = 0$, this gives $M_1/n$ as the probability of observing a new species.
Although derived from different principles, $\psi_{k}^{\text{GT}}$ approaches this estimator as $n \to \infty$ with $k$ fixed, as the additional finite-sample correction term $(k+1)/(n+1)$ vanishes.
Moreover, on the event $H_k$, $\psi_{k}^{\text{GT}}$ coincides with the Good--Turing estimator computed on the augmented sample $Y_{1:(n+1)}$.

The power of conformal Good--Turing $p$-values can be easily increased through randomization.
For each $k \in \{0\} \cup \mathbb{N}$, we define the randomized conformal Good--Turing $p$-value
\begin{align}
\label{eq:p-value-RGT}
\psi^{\text{RGT}}_{k} = \frac{\lceil U G_k \rceil}{n + 1}, \qquad G_k := (k+1)M_{k+1}+k+1,
\end{align}
where $U \sim \mathrm{Uniform}(0,1)$ is independent of everything else.
This is valid, by Theorem~\ref{theorem:RGT-super-uniform} below, and more powerful than $\psi_{k}^{\text{GT}}$, since $\psi^{\text{RGT}}_{k} \leq \psi_{k}^{\text{GT}}$ always and $\P{\psi^{\text{RGT}}_{k} < \psi_{k}^{\text{GT}} \mid \mathcal{F}_n} = 1 - 1/G_k$.

\begin{theorem}
\label{theorem:RGT-super-uniform}
If $Y_{1:(n+1)}$ are exchangeable, $\psi^{\text{RGT}}_{k}$ satisfies~\eqref{eq:pval-validity} for any $k \in \{0\} \cup \mathbb{N}$.
Moreover, if $\psi'_k = g(\mathcal{F}_n, U)$ also satisfies~\eqref{eq:pval-validity} under every exchangeable distribution of $Y_{1:(n+1)}$, then, whenever $k=0$ or $M_k \geq 1$,
\begin{align*}
\P{\psi'_k \leq \frac{j}{n+1} \mid \mathcal{F}_n}
\leq
\P{\psi_k^{\text{RGT}} \leq \frac{j}{n+1} \mid \mathcal{F}_n}
= \frac{j}{G_k}, \qquad \forall j \in \{1,\ldots,G_k\}.
\end{align*}
\end{theorem}
Since $\mathcal{F}_n$ is held fixed, the conditional probabilities above involve only the auxiliary randomization $U$. Thus, no valid randomized conformal $p$-value depending on the data only through $\mathcal{F}_n$ can be more powerful than $\psi_k^{\mathrm{RGT}}$ at any grid level $j/(n+1)$.

\subsection{Feature-Based Conformal Good--Turing $p$-Values} \label{sec:pvals-feature-based}

We now extend the conformal Good--Turing $p$-values to incorporate feature information, which can further increase their power.
For each $k \in \{0\} \cup \mathbb{N}$, let $\mathcal{S}_{k} := \{ i \in [n] : N(Y_i; Y_{1:n}) = k \}$ denote the indices of the points whose label appears exactly $k$ times in $Y_{1:n}$.
Let $\hat{s}_k:\mathcal{X} \to \mathbb{R}$ be a score function, which for simplicity we can think of as fixed for now.
More generally, $\hat{s}_k$ can also be learned from the data, as long as it satisfies an invariance condition stated in Theorem~\ref{thm:XGT} below. Any score satisfying this condition yields a valid $p$-value, but for power, $\hat{s}_k$ should be designed so that larger values of $\hat{s}_k(X_{n+1})$ indicate that $X_{n+1}$ is less consistent with $Y_{n+1}$ having frequency $k$, thus providing stronger evidence against $H_k$.
We define the feature-based conformal Good--Turing $p$-value as
\begin{align}
\label{eq:p-value-XGT}
  \psi^{\text{XGT}}_{k} := \frac{1}{n+1} \left( 1 + \sum_{i \in \mathcal{S}_{k+1}} \I{\hat{s}_k(X_i) \geq \hat{s}_k(X_{n+1})} + \max_{i \in \mathcal{S}_k} \sum_{j \in [n] : Y_j = Y_i} \I{\hat{s}_k(X_j) \geq \hat{s}_k(X_{n+1})} \right),
\end{align}
with the convention that a maximum over an empty set is zero.

For intuition, consider $k=0$, in which case $\mathcal{S}_0 = \emptyset$ and~\eqref{eq:p-value-XGT} simplifies to
\begin{align} \label{eq:pvalue-factorization}
\psi^{\text{XGT}}_{0}
= \underbrace{\frac{M_1+1}{n+1}}_{\textstyle = \, \psi^{\text{GT}}_{0}}
\; \cdot \;
\underbrace{\frac{1 + \sum_{i \in \mathcal{S}_1} \I{\hat{s}_0(X_i) \geq \hat{s}_0(X_{n+1})}}{M_1+1}}_{\textstyle =: \, R_0}
\; \in \; \left\{ \frac{1}{n+1}, \ldots, \frac{M_1+1}{n+1} \right\}.
\end{align}
The first factor is the feature-blind conformal Good--Turing $p$-value from~\eqref{eq:p-value-GT}, which captures the evidence against $H_0$ carried by the label frequencies alone.
The second factor captures the evidence carried by the features, and corresponds to a standard conformal $p$-value of the type used for out-of-distribution testing \citep{bates2023testing}. It compares the test point with the singletons, and satisfies conditional super-uniformity, $\P{R_0 \leq u \mid H_0} \leq u$ for all $u \in (0,1)$, since the test point is exchangeable with the other singletons conditional on $H_0$.
Each factor is valid on its own, and their product combines both sources of evidence while remaining valid in the sense of~\eqref{eq:pval-validity}, as established by Theorem~\ref{thm:XGT} below.
The interpretation of $\psi^{\text{XGT}}_{0}$ is as follows. If $\hat{s}_{0}(X_{n+1})$ exceeds $\hat{s}_{0}(X_i)$ for all $i \in \mathcal{S}_1$, the test point looks less like a rare label than every observed singleton, which is strong evidence against $H_0$, and $\psi^{\text{XGT}}_{0}$ attains its minimum $1/(n+1)$.
If instead $\hat{s}_{0}(X_{n+1})$ is smaller than all these scores, the features carry no evidence against $H_0$, and $\psi^{\text{XGT}}_{0}$ attains its maximum $\psi^{\text{GT}}_{0}$.
Thus, $\psi^{\text{XGT}}_{0}$ behaves like the randomized $p$-value~\eqref{eq:p-value-RGT}, except that ties are broken by feature-based scores, which may be informative, rather than by independent noise.
In fact, if the scores are almost surely distinct, the two have exactly the same type-I error $\P{\psi_0 \leq u, H_0}$, which is nearly sharp whenever the fraction of singletons exceeds $u$ with high probability (Proposition~\ref{prop:unseen-error-exact}); an informative feature score therefore affects only the power.

For $k>0$, $\psi^{\text{XGT}}_{k}$ is more involved, because the score of the test point is also compared with those of points sharing a label observed $k$ times. However, the same intuition still applies: like $\psi^{\text{RGT}}_k$, it takes values on the data-dependent grid $\{1, \ldots, (k+1)M_{k+1}+k+1\}/(n+1)$, with ceiling $\psi^{\text{GT}}_k$, and it can be expected to be more powerful than $\psi^{\text{RGT}}_k$ when the features are informative.

\begin{theorem} \label{thm:XGT}
Assume $\{(X_i, Y_i)\}_{i=1}^{n+1}$ are exchangeable, and fix any $k \in \{0\} \cup \mathbb{N}$. If $\hat{s}_k$ is either fixed or invariant to permutations of $\{(X_i, Y_i)\}_{i=1}^{n+1}$ on the event $H_k$, then $\psi^{\text{XGT}}_{k}$ satisfies~\eqref{eq:pval-validity}.
\end{theorem}

A practical way to construct a data-dependent score satisfying the invariance property required by Theorem~\ref{thm:XGT} is to let it depend only on observations with label frequency outside $\{k, k+1\}$. The data points that play the role of calibration data in~\eqref{eq:p-value-XGT} are then precisely those excluded from the construction of the score.
Moreover, on $H_k$, the latent label $Y_{n+1}$ must be one of those observed exactly $k$ times, so the features $\{X_i\}_{i\in [n] \setminus (\mathcal{S}_k \cup \mathcal{S}_{k+1})}$ of points with label frequency outside $\{k, k+1\}$ never include $X_{n+1}$, and are unchanged by any permutation of the $n+1$ data points.
Concretely, we train a {\em one-class classifier} \citep{moya1993one,pimentel2014review} on the features of the observations with label frequency outside $\{k, k+1\}$, e.g., for $k=0$, on the features of samples whose labels appear at least twice; the resulting score satisfies the assumption of Theorem~\ref{thm:XGT} as long as the learning algorithm is insensitive to the order of its inputs (otherwise, the training data can be randomly shuffled).
Since $X_{n+1}$ is never used for training, the same score function can also be reused across test points.
In the empirical applications of Section~\ref{sec:empirical}, we use the local outlier factor model implemented in \texttt{scikit-learn} \citep{pedregosa2011scikit}, which assigns larger scores to features conforming to its training distribution; larger values of $\hat{s}_k(X_{n+1})$ thus indicate that $X_{n+1}$ resembles points whose labels have frequency outside $\{k, k+1\}$, providing evidence against $H_k$.
In general, any one-class classifier can be used.

Training a one-class classifier only on the features of points with label frequency outside $\{k, k+1\}$ is computationally convenient and works reasonably well in practice, but it is not necessarily the most powerful approach (Section~\ref{sec:asymptotics}).

\subsection{Testing the Composite Hypothesis that $Y_{n+1}$ is a Previously Seen Label} \label{sec:composite}

For the prediction method of Section~\ref{sec:open-set-classification}, it will be useful to test not only $H_0$ but also the complementary null that $Y_{n+1}$ is {\em not} new. This is the composite hypothesis
\begin{align} \label{eq:seen-hypothesis}
H_{\text{seen}}: \bigcup_{k \in K_n} H_k,
\end{align}
where $K_n := \{k \in [n]: M_k > 0\}$ is the set of observed label frequencies.
Rejecting $H_{\text{seen}}$ suggests the test point is likely to have a new label, as opposed to rejecting $H_0$ which indicates the label is likely to have been already seen.
We test $H_{\text{seen}}$ using the combination conformal $p$-value
\begin{align} \label{eq:combination-pval}
  \psi_{\text{seen}} = \max_{k \in K_n} \frac{\psi_k}{c_k},
\end{align}
where each $\psi_k$ is a valid conformal $p$-value for $H_k$, and the weights $c_k$ are suitable constants.
The $\psi_k$ may be feature-based, as in Section~\ref{sec:pvals-feature-based}, although this requires training a separate one-class classifier for each $k \in K_n$; for simplicity, our experiments use the feature-blind $p$-values $\psi^{\text{RGT}}_k$, which are available in closed form.

\begin{theorem} \label{thm:psi-seen-super-uniformity}
Assume $\{(X_i, Y_i)\}_{i=1}^{n+1}$ are exchangeable and each $\psi_k$ satisfies~\eqref{eq:pval-validity}, for $k \in [n]$. If $c_k \geq 0$ and $\sum_{k=1}^{n} c_k \leq 1$, then $\P{\psi_{\text{seen}} \leq u, H_{\text{seen}}} \leq u$ for all $u \in (0,1)$.
\end{theorem}

The weights minimizing $\psi_{\text{seen}}$ given $\{\psi_k\}_{k \in K_n}$, subject to $\sum_{k \in K_n} c_k \leq 1$, satisfy $c_k \propto \psi_k$, but they depend on the data, whereas Theorem~\ref{thm:psi-seen-super-uniformity} requires fixed constants.
Still, this ideal suggests a practical approach to approximately maximize the power of~\eqref{eq:combination-pval}. By~\eqref{eq:p-value-GT}, $\psi_k$ is roughly proportional to $(k+1)M_{k+1}$, and in many real-world categorical data sets $M_{k+1}$, and hence $(k+1)M_{k+1}$, exhibits approximate power-law tail behavior \citep{ferrer2001two,zipf2016human}. This motivates using the fixed power-law weights
\begin{align} \label{eq:combination-pval-power-law}
c_k = \frac{k^{-\beta}}{\sum_{j=1}^n j^{-\beta}}, \quad k \in [n],
\end{align}
for some $\beta > 0$. Since $\psi_{\text{seen}} \geq \sum_{k \in K_n} \psi_k$ for any valid weights, the test can be informative only when rare labels dominate the frequency profile, and otherwise the choice of $\beta$ is immaterial. 
Where the test is informative, numerical experiments show that the best exponent often lies approximately in $[1.6, 2.0]$, with performance degrading sharply for smaller values and only slowly for larger ones (Figures~\ref{fig:app-beta-sensitivity}--\ref{fig:app-beta-sensitivity-rgt} in Appendix~\ref{app:experiments}), so $\beta = 1.6$ is a natural default.

\section{Open-Set Conformal Classification} \label{sec:open-set-classification}

\subsection{Conformal Good--Turing Classification} \label{sec:cgtc}

We now present and study our method, which we call {\em conformal Good--Turing classification}.
Let $\psi_{\text{unseen}}$ and $\psi_{\text{seen}}$ be conformal Good--Turing $p$-values for testing $H_{\text{unseen}} := H_0$, defined in~\eqref{eq:freq-hypothesis}, and $H_{\text{seen}}$, defined in~\eqref{eq:seen-hypothesis}, respectively.
Fix any $\alpha_{\text{class}},\alpha_{\text{unseen}},\alpha_{\text{seen}} \in [0,1]$ such that $\alpha_{\text{class}}+\alpha_{\text{unseen}}+\alpha_{\text{seen}} = \alpha$; the allocation of this {\em budget} is discussed below.

Our prediction sets may include a ``joker'' symbol $\joker$, representing all unseen labels $\mathcal{Y} \setminus \mathcal{Y}_n$; i.e., $Y_{n+1}$ is covered by a set containing $\joker$ whenever $Y_{n+1} \notin \mathcal{Y}_n$.
Intuitively, the joker must be included unless there is sufficient evidence that $Y_{n+1}$ is not new, while the joker alone suffices if there is sufficient evidence that $Y_{n+1}$ is new.
This leads to the open-set prediction set
\begin{align} \label{eq:open-set-predictor}
  \hat{C}_{\alpha}^*(X_{n+1})
  & := \begin{cases}
    \hat{C}_{\alpha_{\text{class}}}(X_{n+1}; \mathcal{Y}_n), & \text{if } \psi_{\text{unseen}} \leq \alpha_{\text{unseen}} \text{ and } \psi_{\text{seen}} > \alpha_{\text{seen}}, \\
    \{\joker\}, & \text{if } \psi_{\text{unseen}} > \alpha_{\text{unseen}} \text{ and } \psi_{\text{seen}} \leq \alpha_{\text{seen}}, \\
    \hat{C}_{\alpha_{\text{class}}}(X_{n+1}; \mathcal{Y}_n) \cup \{\joker\}, & \text{otherwise,}
  \end{cases}
\end{align}
where $\hat{C}_{\alpha_{\text{class}}}(X_{n+1}; \mathcal{Y}_n)$ is any closed-set conformal prediction set at level $\alpha_{\text{class}}$, computed as if $\mathcal{Y} = \mathcal{Y}_n$, using split or full conformal prediction with any nonconformity scores (Section~\ref{sec:preliminaries}).

\begin{theorem} \label{theorem:coverage-cgtc}
Assume $\{(X_i, Y_i)\}_{i=1}^{n+1}$ are exchangeable, $\P{\psi_{\text{unseen}} \leq \alpha_{\text{unseen}}, H_{\text{unseen}}} \leq \alpha_{\text{unseen}}$, $\P{\psi_{\text{seen}} \leq \alpha_{\text{seen}}, H_{\text{seen}}} \leq \alpha_{\text{seen}}$, and $\P{Y_{n+1} \notin \hat{C}_{\alpha_{\text{class}}}(X_{n+1}; \mathcal{Y}_n), Y_{n+1} \in \mathcal{Y}_n} \leq \alpha_{\text{class}}$, as is the case for standard conformal methods by Theorem~\ref{theorem:coverage-standard}.
Then, $\P{Y_{n+1} \in \hat{C}_{\alpha}^*(X_{n+1})} \geq 1-\alpha$.
\end{theorem}

The next result bounds the coverage of previously seen labels from both sides, similar to Theorem~\ref{theorem:coverage-standard}.
\begin{proposition} \label{prop:seen-coverage}
Under the assumptions of Theorem~\ref{theorem:coverage-cgtc}, 
\begin{align*}
  \P{Y_{n+1} \in \hat{C}_{\alpha}^*(X_{n+1}), Y_{n+1} \in \mathcal{Y}_n} \geq 1 - \P{N_{n+1}} - \alpha_{\mathrm{class}} - \alpha_{\mathrm{seen}}.
\end{align*}
If, in addition, $\hat{C}_{\alpha_{\mathrm{class}}}(X_{n+1}; \mathcal{Y}_n)$ is computed as in~\eqref{eq:standard-prediction-set} with almost surely distinct scores $\{S_1^{(Y_{n+1})},\ldots,S_{n+1}^{(Y_{n+1})}\}$, then
\begin{align*}
\P{Y_{n+1} \in \hat{C}_{\alpha}^*(X_{n+1}), Y_{n+1} \in \mathcal{Y}_n} \leq 1 - \alpha_{\mathrm{class}} + \min\left\{\frac{1}{n+1}, \alpha_{\mathrm{class}} - \P{N_{n+1}} \right\}.
\end{align*}
\end{proposition}

The $\alpha_{\mathrm{seen}}$ term in the lower bound accounts for the event that the test of $H_{\mathrm{seen}}$ discards all observed labels.
If $\P{N_{n+1}} = 0$, the joker never contributes to coverage, which lies between $1 - \alpha_{\mathrm{class}} - \alpha_{\mathrm{seen}}$ and $1 - \alpha_{\mathrm{class}} + 1/(n+1)$, exceeding the usual upper bound $1-\alpha + 1/(n+1)$ by at most $\alpha_{\mathrm{unseen}} + \alpha_{\mathrm{seen}}$.
This excess coverage is the price of protection against open-set cases, and it can be minimized by allocating the $\alpha$ budget carefully, as studied below.

If the label space has known finite cardinality $K$, a suitable allocation recovers the plug-in closed-set prediction sets at the tightened level $\alpha - (K+1)/(n+1)$, never using the joker, while providing marginal coverage without requiring the identities of the $K$ labels to be known.

\begin{proposition}
\label{prop:recovery}
Assume $\{(X_i, Y_i)\}_{i=1}^{n+1}$ are exchangeable and $|\mathcal{Y}| = K$, with $(K+1)/(n+1) < \alpha$. If $(\alpha_{\mathrm{class}},\alpha_{\mathrm{unseen}},\alpha_{\mathrm{seen}}) = (\alpha - (K+1)/(n+1), (K+1)/(n+1), 0)$, then $\hat{C}^{*}_{\alpha}(X_{n+1})$ in~\eqref{eq:open-set-predictor}, computed with $\psi_{0}^{\text{GT}}$, $\psi_{0}^{\text{RGT}}$, or $\psi_{0}^{\text{XGT}}$, satisfies $\hat{C}^{*}_{\alpha}(X_{n+1}) = \hat{C}_{\alpha_{\mathrm{class}}}(X_{n+1};\mathcal{Y}_n)$ almost surely.
\end{proposition}

The intuition behind this result is that $\psi_{\text{unseen}} \leq (K+1)/(n+1) = \alpha_{\mathrm{unseen}}$ because $M_1 \leq K$, so  $H_{\mathrm{unseen}}$ is always rejected and the joker is never included in the prediction set.
Meanwhile, $H_{\mathrm{seen}}$ is never rejected at level $\alpha_{\mathrm{seen}} = 0$, since $\psi_{\text{seen}} \geq 1/(n+1)$; the output thus always coincides with the plug-in closed-set prediction set at level $\alpha - (K+1)/(n+1)$.
While this allocation may be rarely applicable, as it requires the label space to be finite with known cardinality, it motivates the more general data-dependent $\alpha$-allocation strategy developed in Section~\ref{sec:alpha-allocation}.

\subsection{Tightening Coverage via Missing-Mass Estimation} \label{sec:plugin-adjustment}

Theorem~\ref{theorem:coverage-cgtc} guarantees coverage at least $1-\alpha$, and in practice conformal Good--Turing classification is often conservative, as shown in Section~\ref{sec:empirical}.
We explain this by deriving a sharper lower bound on its coverage, under a mild additional regularity assumption.
Motivated by this result, we then introduce a data-driven plug-in adjustment of the $\alpha$-allocation, which typically restores coverage close to the nominal level; this adjustment is not covered by a finite-sample guarantee, but is justified asymptotically by Theorem~\ref{prop:gt-adaptive-allocation} below.

Let $\mu_n := \P{N_{n+1}}$ denote the missing mass, i.e., the probability that the test label is new, and let $\mathcal{Y}_{n+1} := \mathcal{Y}_n \cup \{Y_{n+1}\}$. Consider the auxiliary prediction set $\hat{C}_{\alpha_{\text{class}}}(X_{n+1};\mathcal{Y}_{n+1})$, computed like the closed-set component of~\eqref{eq:open-set-predictor} but with $\mathcal{Y}_{n+1}$ in place of $\mathcal{Y}_n$. This set is not computable in practice, since it depends on the unknown $Y_{n+1}$, but it always contains the true label among its candidates, and it coincides with $\hat{C}_{\alpha_{\text{class}}}(X_{n+1};\mathcal{Y}_n)$ whenever $Y_{n+1} \in \mathcal{Y}_n$.

\begin{assumption}
\label{assumption:novel-miscoverage}
For every fixed $\alpha_{\text{class}} \in (0,1)$, if $0 < \mu_n < 1$,
\begin{equation} \label{eq:novel-harder}
\P{Y_{n+1}\notin \hat{C}_{\alpha_{\text{class}}}(X_{n+1};\mathcal{Y}_{n+1}) \mid Y_{n+1} \notin \mathcal{Y}_n}
\geq
\P{Y_{n+1}\notin \hat{C}_{\alpha_{\text{class}}}(X_{n+1};\mathcal{Y}_{n+1}) \mid Y_{n+1} \in \mathcal{Y}_n}.
\end{equation}
\end{assumption}

In words, the auxiliary set misses new labels at least as often as seen ones. This is plausible because new labels are absent from the training data and are therefore harder to classify; for the nonconformity scores considered in this paper, a new label receives zero estimated probability and hence the largest possible score.

\begin{proposition} \label{prop:sharpened-coverage}
Under the assumptions of Theorem~\ref{theorem:coverage-cgtc} and Assumption~\ref{assumption:novel-miscoverage},
\begin{align} \label{eq:sharpened-coverage}
\P{Y_{n+1} \notin \hat{C}_{\alpha}^*(X_{n+1})} \leq \alpha_{\text{class}}(1-\mu_n) + \alpha_{\text{seen}} + \alpha_{\text{unseen}}.
\end{align}
\end{proposition}

Proposition~\ref{prop:sharpened-coverage} reveals a first source of conservativeness: Theorem~\ref{theorem:coverage-cgtc} bounds the probability that the closed-set component misses a seen label by $\alpha_{\text{class}}$, whereas Assumption~\ref{assumption:novel-miscoverage} implies the tighter bound $\alpha_{\text{class}}(1-\mu_n)$, leaving a gap of $\alpha_{\text{class}}\mu_n$.
This suggests refining our method by inflating the classification level in~\eqref{eq:open-set-predictor} from $\alpha_{\text{class}}$ to $\alpha_{\text{class}}/(1-\mu_n)$, so that the intended classification budget is spent fully.
Since $\mu_n$ is unknown in practice, we replace it with the augmented Good--Turing estimate $\widehat\mu_n := (M_1+1)/(n+1)$, which coincides with $\psi^{\text{GT}}_0$ in~\eqref{eq:p-value-GT}.

A second source of conservativeness arises from the test of $H_{\text{unseen}}$. The proof of Theorem~\ref{theorem:coverage-cgtc} uses $\P{\psi_{\text{unseen}} \leq \alpha_{\text{unseen}}, H_{\text{unseen}}} \leq \alpha_{\text{unseen}}$, but each of the $p$-values in~\eqref{eq:p-value-GT},~\eqref{eq:p-value-RGT}, and~\eqref{eq:p-value-XGT} also satisfies $\psi_{\text{unseen}} \leq \widehat\mu_n$.
Therefore, any budget $\alpha_{\text{unseen}} > \widehat{\mu}_n$ is wasted, and the surplus could be spent on the classification component instead.

These two observations lead to the following missing-mass adjustment.
Given any base allocation $(\tilde\alpha_{\text{class}}, \tilde\alpha_{\text{unseen}}, \tilde\alpha_{\text{seen}})$ summing to $\alpha$, define the adjusted levels
\begin{align} \label{eq:plugin-adjustment}
\alpha^{\text{mm}}_{\text{unseen}}
=
\tilde\alpha_{\text{unseen}}\wedge\widehat{\mu}_n,
\qquad
\alpha^{\text{mm}}_{\text{seen}}
=
\tilde\alpha_{\text{seen}},
\qquad
\alpha^{\text{mm}}_{\text{class}}
=
\left(
\frac{\alpha-\alpha^{\text{mm}}_{\text{unseen}}-\tilde\alpha_{\text{seen}}}
{1-\widehat{\mu}_n}
\right)\wedge 1,
\end{align}
with the convention that $\alpha^{\text{mm}}_{\text{class}}=1$ if $\widehat{\mu}_n=1$.
In words, the adjustment caps the unseen-label budget at the estimated missing mass, reallocates any surplus to the classification component, and inflates the classification level by the factor $1/(1-\widehat{\mu}_n)$.
The following result shows that the adjusted predictor is asymptotically valid whenever $\widehat{\mu}_n$ is consistent.

\begin{theorem} \label{prop:gt-adaptive-allocation}
Assume $\{(X_i,Y_i)\}_{i=1}^{n+1}$ are exchangeable, and that Assumption~\ref{assumption:novel-miscoverage} holds.
Fix any base allocation $(\tilde\alpha_{\text{class}}, \tilde\alpha_{\text{unseen}}, \tilde\alpha_{\text{seen}})$ summing to $\alpha$, with $\tilde\alpha_{\text{class}} > 0$, and let $\hat{C}_{\alpha,n}^*(X_{n+1})$ be the prediction set~\eqref{eq:open-set-predictor} obtained with the adjusted levels~\eqref{eq:plugin-adjustment}, using any valid conformal $p$-values for $H_{\text{unseen}}$ and $H_{\text{seen}}$.
If $\mu_n\to q$ for some $q\in[0,1)$ and $\widehat\mu_n-\mu_n \to 0$ in probability, then
\begin{align*}
\limsup_{n\to\infty} \P{Y_{n+1}\notin \hat{C}_{\alpha,n}^*(X_{n+1})} \le \alpha.
\end{align*}
\end{theorem}

This result is an asymptotic counterpart of Theorem~\ref{theorem:coverage-cgtc}, relying on Assumption~\ref{assumption:novel-miscoverage} and on the consistency of $\widehat\mu_n$. Assumption~\ref{assumption:novel-miscoverage} can be relaxed to an asymptotic condition, which also covers the selective splitting of Section~\ref{sec:selective-splitting} (Appendix~\ref{app:proofs}).
The consistency of $\widehat\mu_n$ is known to hold for i.i.d.\ data, where the Good--Turing estimate concentrates around the missing mass at rate $O(n^{-1/2})$ \citep{mcallester2003concentration}, and whenever $q = 0$, i.e., whenever new labels become vanishingly rare as $n$ grows, as in the Dirichlet process example of Section~\ref{sec:setup}.

The following subsection studies in more depth the limiting behavior of conformal Good--Turing prediction sets based on this adjusted allocation. The base allocation remains a free hyperparameter of our method, whose data-driven choice is discussed in Section~\ref{sec:alpha-allocation}.

\subsection{Asymptotic Analysis with the Missing-Mass Adjustment} \label{sec:asymptotics}

We now study the limiting behavior of conformal Good--Turing classification with the missing-mass adjustment of Section~\ref{sec:plugin-adjustment}.
We show that, if new labels become rare as the sample size increases, the joker vanishes and our method reduces to plug-in closed-set conformal prediction, under exchangeability alone.
Under additional consistency assumptions, detailed in Appendix~\ref{app:asymptotics}, we further show that our prediction sets approach those of an ideal oracle.

For simplicity, we focus on base allocations of the form $(\tilde\alpha_{\text{class}}, \tilde\alpha_{\text{unseen}}, \tilde\alpha_{\text{seen}}) = (\alpha-a, a, 0)$, with $a \in (0,\alpha)$. Setting $\tilde\alpha_{\text{seen}} = 0$ means that the prediction sets never consist of the joker alone; this is natural, and performs well empirically (Section~\ref{sec:empirical}), when new labels are expected to be relatively rare.
With this choice, $H_{\text{seen}}$ is never rejected because $\psi_{\text{seen}} \geq 1/(n+1) > 0$, and the prediction set in~\eqref{eq:open-set-predictor} simplifies to
\begin{equation} \label{eq:predictor-zero-seen}
\hat C^*_{\alpha,n}(X_{n+1})=
\hat C_{\alpha^{\text{mm}}_{\text{class}}}(X_{n+1};\mathcal{Y}_n)
\cup\{\joker :\psi_{\text{unseen}}>\alpha^{\text{mm}}_{\text{unseen}}\}.
\end{equation}

We first consider the asymptotically closed-set case where $\mu_n \to 0$, which includes any application with a finite label space as well as the Dirichlet process example of Section~\ref{sec:setup}.

\begin{theorem} \label{thm:asymptotic-vanishing}
Assume $\{(X_i,Y_i)\}_{i=1}^{n+1}$ are exchangeable, and let $\hat{C}_{\alpha,n}^*(X_{n+1})$ be the prediction set~\eqref{eq:predictor-zero-seen} obtained with the adjusted levels~\eqref{eq:plugin-adjustment}, for the base allocation $(\alpha-a, a, 0)$, using $\psi_{\text{unseen}} \in \{\psi^{\text{GT}}_0, \psi^{\text{RGT}}_0, \psi^{\text{XGT}}_0\}$.
If $\mu_n \to 0$ as $n \to \infty$, then
\begin{align*}
\P{\joker\in\hat C^*_{\alpha,n}(X_{n+1})}\to 0
\qquad \text{and} \qquad
\alpha^{\text{mm}}_{\text{class}}\to\alpha \qquad \text{ in probability}.
\end{align*}
If, in addition, the closed-set prediction sets $\hat{C}_{\alpha_{\text{class}}}(X_{n+1};\mathcal{Y}_n)$ are computed as in~\eqref{eq:standard-prediction-set} with almost surely distinct scores $\{S_1^{(Y_{n+1})},\ldots,S_{n+1}^{(Y_{n+1})}\}$, then
\begin{align*}
\lim_{n\to\infty} \P{Y_{n+1}\in\hat C^*_{\alpha,n}(X_{n+1})} = 1-\alpha.
\end{align*}
\end{theorem}

Thus, protection against unseen labels comes at no asymptotic cost when new labels become rare.
Indeed, since $\widehat\mu_n$ vanishes with $\mu_n$, with probability tending to one $\alpha^{\text{mm}}_{\text{unseen}} = \widehat\mu_n \geq \psi_{\text{unseen}}$, so the joker is excluded, while $\alpha^{\text{mm}}_{\text{class}}$ tends to $\alpha$.
Moreover, under a population model with finitely many recurring labels, together with consistency and regularity conditions, our prediction sets converge to the smallest closed-set prediction sets with feature-conditional coverage $1-\alpha$ (Theorem~\ref{thm:oracle-convergence} in Appendix~\ref{app:asymptotics}).

Theorem~\ref{thm:oracle-convergence} also covers the more challenging asymptotically open-set case where $\mu_n \to q > 0$, provided our method is applied with selective sample splitting (Section~\ref{sec:selective-splitting}) and a consistent novelty score $\hat{s}_0$. In this case, our prediction sets approach those of an oracle that splits the significance level $\alpha$ between new and previously seen labels, according to the part of $\alpha$ allocated to the test for new labels. Given this split, the oracle is optimal in two senses (Proposition~\ref{prop:oracle-optimality} in Appendix~\ref{app:asymptotics}): for seen labels, it outputs the smallest sets of recurring labels with the corresponding feature-conditional coverage; for new labels, it includes the joker least often among all rules that leave a new label uncovered with at most the same probability, by including it for the inputs where new labels are most likely. Selective splitting is needed because it keeps singletons out of the calibration set. With uniform random splitting, singletons would enter the calibration set with large nonconformity scores, and the method would become asymptotically conservative.

\subsection{Data-Driven Allocation of the Significance Budget} \label{sec:alpha-allocation}

The informativeness of the prediction sets~\eqref{eq:open-set-predictor} depends on how the budget $\alpha$ is allocated: increasing $\alpha_{\text{class}}$ includes fewer seen labels, increasing $\alpha_{\text{unseen}}$ reduces reliance on the joker, and increasing $\alpha_{\text{seen}}$ allows all seen labels to be excluded when new labels are very common. The best allocation thus depends on the data; e.g., the entire budget should go to $\alpha_{\text{class}}$ if $\mu_n = 0$, and to $\alpha_{\text{seen}}$ if $\mu_n = 1$.
We therefore propose a data-driven allocation heuristic, which may be applied either directly to the method of Section~\ref{sec:cgtc} or to tune the base allocation of the missing-mass adjustment of Section~\ref{sec:plugin-adjustment}.
This approach is not covered by Theorems~\ref{theorem:coverage-cgtc} and~\ref{prop:gt-adaptive-allocation}, which assume a fixed allocation, but we validate it empirically in Section~\ref{sec:empirical}.

Informativeness depends on two distinct criteria: the number of seen labels in the prediction set, and whether the joker is included.
Since the joker has no natural cardinality, as it represents the entire unobserved portion of the label space, its cost relative to an additional seen label may depend on the application and on the practitioner's judgement.
We therefore allow the practitioner to choose the weight $\lambda \in [0,1]$ used to combine the two criteria in the aggregate loss
\begin{align} \label{eq:alpha_allocation_loss}
L\bigl(\hat{C}_{\alpha}^*(X_{n+1}),Y_{n+1};\lambda\bigr)
=
\lambda\,
\frac{\bigl|\hat{C}_{\alpha}^*(X_{n+1}) \setminus \{\joker\}\bigr|}
     {\bigl|\hat{C}_{\alpha}(X_{n+1}; \mathcal{Y}_n)\bigr|}
+
(1-\lambda)\,
\I{\joker\in\hat{C}_{\alpha}^*(X_{n+1}),\,Y_{n+1}\in\mathcal{Y}_n}.
\end{align}
The first term above normalizes the number of seen labels in the prediction set by the size of the plug-in prediction set obtained by allocating the entire budget to classification, so that the two terms are on a comparable scale.
The second term penalizes the joker only when used unnecessarily.

An ideal allocation is that which minimizes the expected loss,
\[
{\arg\min}_{\substack{\alpha_{\text{class}},\alpha_{\text{unseen}},\alpha_{\text{seen}} \geq 0 \\ \alpha_{\text{class}} + \alpha_{\text{unseen}} + \alpha_{\text{seen}} = \alpha}} \; \E{L\bigl(\hat{C}_{\alpha}^*(X_{n+1}),Y_{n+1};\lambda\bigr)},
\]
where, for the missing-mass adjustment, the minimization is over the base allocation, and $\hat{C}_{\alpha}^*(X_{n+1})$ is computed with the adjusted levels~\eqref{eq:plugin-adjustment}.
In practice, we approximate the expectation by $V$-fold cross-validation, applying within each fold the same procedure that will be used at test time, including, for the adjusted method, the computation of $\widehat\mu_n$ and of the adjusted levels.
The details are given in Appendix~\ref{app:plugin-adjustment}, without and with the missing-mass adjustment.

\subsection{Frequency-Based Selective Sample Splitting} \label{sec:selective-splitting}

Our results so far apply to any closed-set conformal predictor $\hat{C}_{\alpha_{\text{class}}}(X_{n+1}; \mathcal{Y}_n)$, but in practice we expect our method to be applied mostly with data splitting, since full conformal prediction is often computationally prohibitive.
This raises the question of how to divide the data into training and calibration sets for the closed-set classification component of our method.

The standard practice is to split the data uniformly at random, but this can be suboptimal under severe class imbalance: if the only observation from a rare class is placed in the calibration set, the model cannot learn to recognize that class at all.
This suggests prioritizing rare observations for training.
Such a data-dependent split breaks the exchangeability between calibration and test data, but validity can be restored by reweighting, using the framework of \citet{tibshirani2019conformal}, originally developed for covariate shift and stated more generally by \citet{liang2026structured}.

Given a non-decreasing function $\pi:\mathbb{N}\to[0,1]$, assign each observation $i \in [n]$ to the calibration set independently, with  probability depending on the frequency $N(Y_i; Y_{1:n})$ of its label:
\[
I_i \overset{\text{ind.}}{\sim} \operatorname{Bernoulli}\bigl(\pi(N(Y_i; Y_{1:n}))\bigr),
\qquad
\mathcal{D}^{\text{cal}} = \{ i \in [n] : I_i = 1\},
\qquad
\mathcal{D}^{\text{train}} = [n] \setminus \mathcal{D}^{\text{cal}}.
\]
Since $\pi$ is non-decreasing, observations with rarer labels are at least as likely to be used for training.
The classification model is fitted on $\mathcal{D}^{\text{train}}$, and the nonconformity scores are then evaluated on the calibration points and, for each candidate label $y$, on the test point.

After this data-dependent split, the calibration points are no longer exchangeable with the test point, so the standard plug-in prediction set $\hat{C}_{\alpha}(X_{n+1}; \mathcal{Y}_n)$ may lose validity even in the closed-set setting.
To correct for this, we reweight the calibration and test points by
\begin{align} \label{eq:conformalization-weights}
w_j(y) = \frac{p^{(j)}(y)}{p^{(n+1)}(y) + \sum_{k \in \mathcal{D}^{\text{cal}}} p^{(k)}(y)},
\qquad j \in \mathcal{D}^{\text{cal}} \cup \{n+1\},
\end{align}
where
\begin{align} \label{eq:conformalization-weights-p}
p^{(j)}(y) := \prod_{i \in \mathcal{D}^{\text{cal}}} \pi\bigl(N(Y_i^{(j,y)}; Y_{1:n}^{(j,y)})\bigr) \prod_{i \in \mathcal{D}^{\text{train}}} \bigl\{1 - \pi\bigl(N(Y_i^{(j,y)}; Y_{1:n}^{(j,y)})\bigr)\bigr\},
\end{align}
and $Y_{1:n}^{(j,y)} \in \mathcal{Y}^{n}$ denotes the label sequence obtained by replacing $Y_j$ with $y$, for any $j \in \mathcal{D}^{\text{cal}}$ and $y \in \mathcal{Y}_n$, and we set $Y_{1:n}^{(n+1,y)} := Y_{1:n}$.
Thus, $p^{(j)}(y)$ is the probability of the observed split if the $j$-th calibration point had been swapped with a test point with label $y$.

The plug-in closed-set prediction set $\hat{C}_{\alpha}(X_{n+1}; \mathcal{Y}_n)$ is then replaced by
\begin{align} \label{eq:weighted-split}
\hat{C}^{\pi}_{\alpha}(X_{n+1}; \mathcal{Y}_n) = \left\{y \in \mathcal{Y}_n : w_{n+1}(y) +\sum_{j\in\mathcal{D}^{\text{cal}}}w_j(y)\I{S_j \geq S_{n+1}^{(y)}} > \alpha\right\}.
\end{align}
If $\pi \equiv \varrho$ is constant, all weights equal $1/(1+|\mathcal{D}^{\text{cal}}|)$, and~\eqref{eq:weighted-split} reduces to~\eqref{eq:standard-prediction-set} with $n$ replaced by $|\mathcal{D}^{\text{cal}}|$. In that case, $|\mathcal{D}^{\text{cal}}| \sim \operatorname{Binomial}(n, \varrho)$ is independent of the data, so conditioning on it recovers standard split conformal prediction with a uniform random split.
In general, these weights do not depend on the test features, so they can be computed once, for every candidate label, in $\mathcal{O}(n + |\mathcal{Y}_n|)$ time. Under the split conformal framework,~\eqref{eq:weighted-split} then has the same computational cost as standard split conformal prediction, up to a logarithmic factor (Appendix~\ref{app:selective-splitting}).

\begin{theorem}
\label{thm:weighted-split-validity}
If $\{(X_i, Y_i)\}_{i=1}^{n+1}$ are exchangeable, then, for any fixed $\pi$ and any $\alpha \in (0,1)$,
\[
\P{Y_{n+1} \notin \hat{C}^{\pi}_{\alpha}(X_{n+1}; \mathcal{Y}_n), Y_{n+1} \in \mathcal{Y}_n} \leq \alpha.
\]
\end{theorem}

Thus, the weighted prediction set~\eqref{eq:weighted-split} satisfies the closed-set requirement of Theorem~\ref{theorem:coverage-cgtc}, for any fixed $\pi$.
More broadly, selective splitting applies to any closed-set classification problem with a known label space $\mathcal{Y}$: by the proof of Theorem~\ref{thm:weighted-split-validity}, computing~\eqref{eq:weighted-split} over $\mathcal{Y}$ instead of $\mathcal{Y}_n$ yields marginal coverage at least $1-\alpha$, which makes it an appealing alternative to uniform random splitting under severe class imbalance.
However, if rare labels are not used for calibration, coverage conditional on a rare test label may be lower than with random splitting \citep{sesia2023conformal}. Our experiments in Section~\ref{sec:empirical} show that this effect is mild in the settings considered here.

Furthermore, the upper bound of Theorem~\ref{thm:weighted-split-validity} can be sharpened from $\alpha$ to $\alpha (1 - \mu_n)$ for any $\pi$ satisfying $\pi(1) = 0$ (always using singleton labels for training), in which case the weighted prediction set~\eqref{eq:weighted-split} is valid conditional on the test label having been seen (Lemma~\ref{lemma:seen-calibration} in Appendix~\ref{app:proofs}). 
This is akin to Proposition~\ref{prop:sharpened-coverage} in Section~\ref{sec:plugin-adjustment},  and also implies that the missing mass adjustment remains asymptotically valid (Theorem~\ref{prop:gt-adaptive-allocation}) under selective sample splitting, without requiring Assumption~\ref{assumption:novel-miscoverage} (Lemma~\ref{lemma:seen-calibration}).
These results provide additional motivation for applying conformal Good--Turing classification with frequency-based selective sample splitting.

In practice, we set $\pi(1)=0$ and $\pi(k) = \min\{n^{\text{cal}}/(n - M_1),\, 1\}$ for $k \geq 2$, where $n^{\text{cal}}$ is a target calibration size, so that the calibration set then has expected size exactly $n^{\text{cal}}$ if $n^{\text{cal}} \leq n - M_1$. 
Since $M_1$ is random, this choice is not strictly covered by Theorem~\ref{thm:weighted-split-validity}, but it works well in our experiments.
Unequal weights reduce the effective size of the calibration set, but only slightly for small calibration proportions: with the $10\%$ proportion used in Section~\ref{sec:empirical}, the effective calibration size is at least $91\%$ of the actual one in all our experiments (Appendix~\ref{app:selective-splitting}).

\section{Empirical Demonstrations} \label{sec:empirical}

\subsection{Setup}  \label{sec:empirical-setup}

We evaluate our method on simulated and real data, and compare it with the standard plug-in closed-set conformal approach of Section~\ref{sec:methods-closed-set}, which assumes $\mathcal{Y} = \mathcal{Y}_n$.
Both methods are implemented with either uniform random splitting (random splitting, for short), assigning $90\%$ of the samples to training and $10\%$ to calibration, or the selective splitting of Section~\ref{sec:selective-splitting}, with a target calibration size of $10\%$ of the samples.
This small calibration share reflects the different difficulty of the two tasks: calibration estimates a single quantile, whereas training must learn many rare classes from few examples. All methods use $\alpha = 0.1$ in the synthetic experiments and $\alpha = 0.2$ in the real data experiments.

All methods use the same $k$-nearest neighbor classifier, with $k=5$ in the synthetic experiments and $k=10$ in the CelebA experiments, and the same randomized generalized inverse quantile scores \citep{romano2020classification}.
Our method tests $H_{\text{unseen}}$ with the feature-based $p$-values of Section~\ref{sec:pvals-feature-based}, using a local outlier factor model as the one-class classifier, and $H_{\text{seen}}$ with the randomized $p$-values of Section~\ref{sec:pvals-feature-blind}, combined as in~\eqref{eq:combination-pval} with the power-law weights~\eqref{eq:combination-pval-power-law} and $\beta$ selected by cross-validation.
The base allocation of $\alpha$ is selected by the cross-validation strategy of Section~\ref{sec:alpha-allocation}, with $\lambda = 0.5$, and then adjusted using the estimated missing mass as in Section~\ref{sec:plugin-adjustment}; we refer to this method as CGTC+, and to the unadjusted method as CGTC.
Appendix~\ref{app:experiments} describes implementation details and presents additional results.

\subsection{Numerical Experiments with Synthetic Data} \label{sec:synthetic-exp}

We begin by analyzing synthetic data from the Dirichlet process model of Section~\ref{sec:setup}. The labels $Y$ follow a Dirichlet process with base distribution $P_0 = \text{Uniform}(0,1)$ and concentration parameter $\theta$ ranging from $12$ to $1{,}500$, which governs the probability of new labels. Conditional on $Y$, the corresponding feature vector $X \in \mathbb{R}^{3}$ is Gaussian with independent coordinates, element-wise mean $Y$, and standard deviation $10^{-4}$. This distribution gives rise to a classification problem with countably many imbalanced Gaussian classes. The sample size is $n = 2{,}000$. 

Figure~\ref{fig:dp-main-four-panel} summarizes the results as a function of $\theta$, averaging over $1{,}000$ test points and $50$ independent repetitions.
As predicted by Theorem~\ref{theorem:coverage-standard}, standard plug-in conformal classification undercovers as new labels become more common, whereas CGTC+ maintains valid coverage by including the joker when needed. With selective splitting, its coverage stays close to the nominal level over the whole range of $\theta$, while random splitting leads to much larger prediction sets. Without the missing-mass adjustment, CGTC remains valid but is more conservative (Figures~\ref{fig:app-dp-orig-four-panel}--\ref{fig:app-dp-orig-pvalues} in Appendix~\ref{app:experiments}).
The rate at which CGTC+ includes the joker increases with $\theta$, as new labels become more common, and the data-driven allocation adapts accordingly, assigning more of the budget to the test of $H_{\text{unseen}}$ (Figure~\ref{fig:app-dp-alpha-allocation}).
The gap between this rate and the true proportion of new labels is explained by the limited power of the test of $H_{\text{unseen}}$: some test points with previously seen labels cannot be distinguished from new ones, so the joker is sometimes included unnecessarily.

\begin{figure}[!htb]
    \centering
    \includegraphics[width=\textwidth]{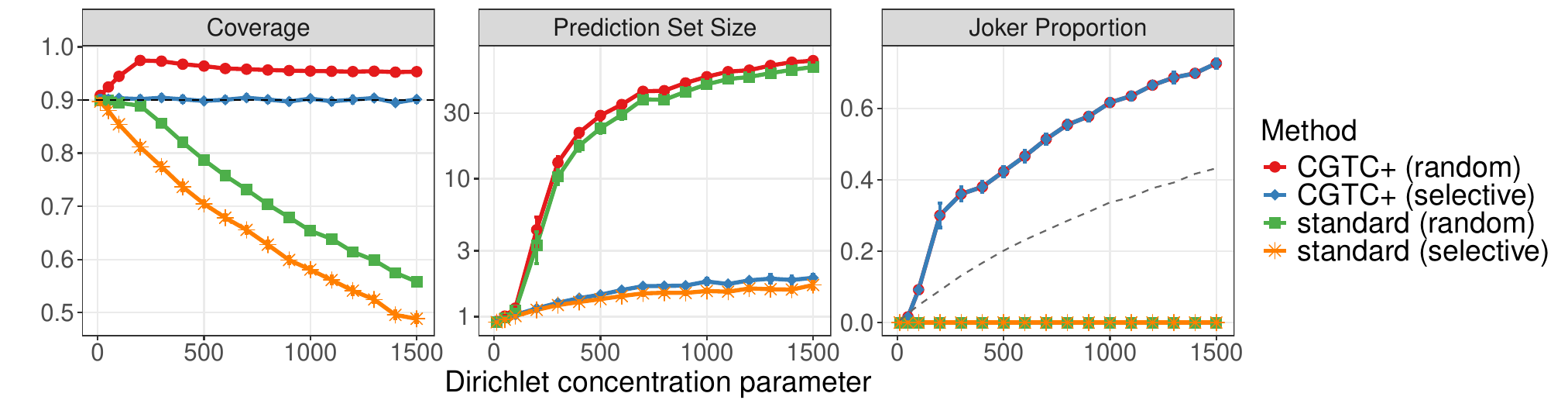}
    \caption{Performance of conformal prediction sets on synthetic data from a Dirichlet process model, as a function of the concentration parameter $\theta$; larger values of $\theta$ correspond to more frequent new labels. The nominal coverage is $90\%$. CGTC+ maintains valid coverage by including the joker when needed and, with selective splitting, produces smaller prediction sets. Set sizes count seen labels only; the use of the joker is reported separately in the third panel, where the dashed line shows the true proportion of unseen test labels. Smaller prediction sets and lower use of joker are preferable, subject to coverage remaining valid. Error bars represent $\pm 1.96$ standard errors.}
    \label{fig:dp-main-four-panel}
\end{figure}

Figure~\ref{fig:dp-main-conditional-coverage} stratifies coverage by the frequency of the test label in the reference sample, using four bins, only for evaluation; the ``very rare'' bin contains labels that are unseen or appear once. The standard method consistently undercovers very rare labels, because it can only output previously seen labels.
By contrast, CGTC+ attains coverage near or above the nominal level in all bins, except for very rare labels at the smallest values of $\theta$.
This is consistent with valid marginal coverage (Figure~\ref{fig:dp-main-four-panel}) and hard to avoid: at small $\theta$, the estimated missing mass is so small that the joker is always excluded, so unseen labels, which make up about half of the very rare bin, are never covered (Figure~\ref{fig:app-dp-freq0freq1}); but this bin then contains only about $2\theta/(\theta+n) \approx 1\%$ of the test points (Figure~\ref{fig:app-dp-bin-mass}).
With random splitting, the coverage of CGTC+ is close to one in most bins, because its prediction sets are too large, whereas selective splitting keeps it close to the nominal level.

\begin{figure}[!htb]
    \centering
    \includegraphics[width=\textwidth]{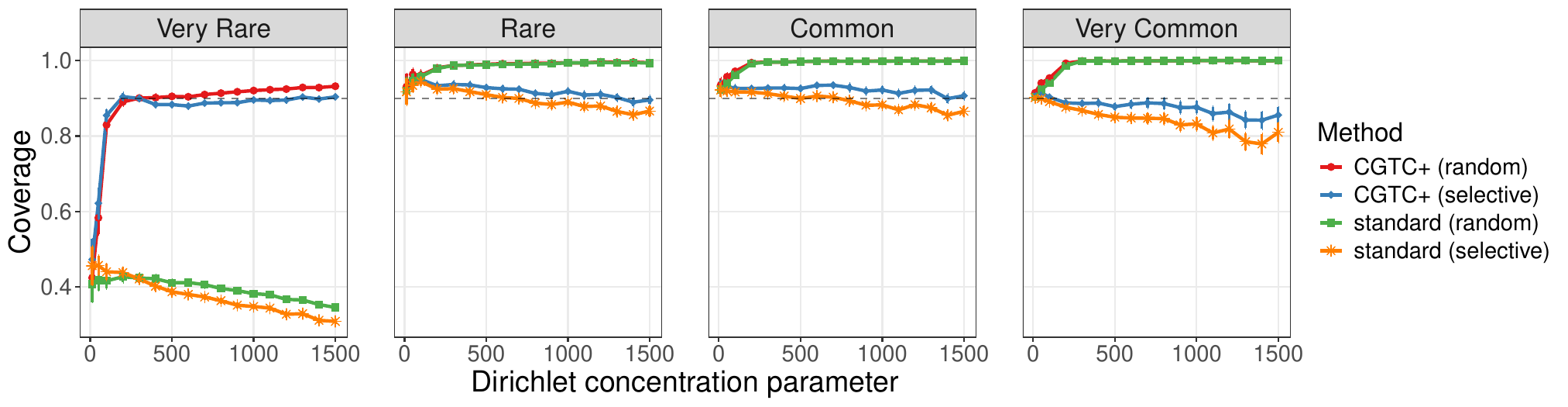}
\caption{Coverage on synthetic data from a Dirichlet process model, as in Figure~\ref{fig:dp-main-four-panel}, stratified by the frequency of the test label. The standard conformal approach undercovers very rare labels, whereas CGTC+ is valid for both rare and common labels, except for very rare labels at the smallest values of $\theta$, which make up about $1\%$ of the test points.}
    \label{fig:dp-main-conditional-coverage}
\end{figure}

Figure~\ref{fig:dp-main-pvalues-comparison} shows the benefit of using features to test $H_{\text{unseen}}$.
The $p$-value $\psi_{0}^{\text{XGT}}$ rejects $H_{\text{unseen}}$ most often, and thus enables our method to include the joker least often, whereas the feature-blind $\psi_{0}^{\text{GT}}$ and $\psi_{0}^{\text{RGT}}$ are much less powerful, with randomization offering only modest gains; Figure~\ref{fig:app-dp-pvalue-full} shows that all three $p$-values lead to valid coverage. 
In these experiments, the joker inclusion rate depends only on this test, because the data-driven allocation selects $\alpha_{\text{seen}} = 0$ (Figure~\ref{fig:app-dp-alpha-allocation}), so that $H_{\text{seen}}$ is never rejected.
Qualitatively similar conclusions also follow if different one-class classifiers are used to compute the feature-based $p$-value (Figures~\ref{fig:app-dp-iforest-four-panel}--\ref{fig:app-dp-ocsvm-pjoker}).

\begin{figure}[!htb]
    \centering
    \includegraphics[width=0.55\textwidth]{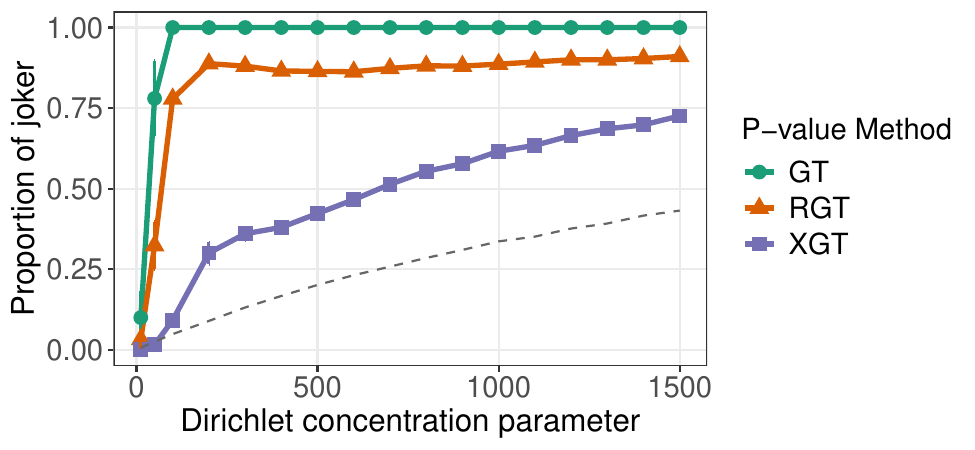}
    \caption{Fraction of prediction sets produced by our method that include the joker, on synthetic data from a Dirichlet process model as in Figure~\ref{fig:dp-main-four-panel}. Different curves correspond to different constructions of the conformal $p$-values used for testing $H_{\text{unseen}}$. The feature-based $p$-value (XGT), used in Figure~\ref{fig:dp-main-four-panel}, includes the joker most parsimoniously, tracking the true proportion of unseen test labels (dashed line) most closely. The deterministic feature-blind $p$-value (GT) is overly conservative, and its randomized variant (RGT) improves on it only slightly.}
    \label{fig:dp-main-pvalues-comparison}
\end{figure}

In summary, conformal Good--Turing classification achieves valid coverage in open-set settings, where standard methods fail, and selective splitting makes its prediction sets substantially more informative.
These conclusions hold over a wide range of reference sample sizes (Figures~\ref{fig:app-dp-four-panel-nref100}--\ref{fig:app-dp-propjoker-nref100}) and calibration proportions from $5\%$ to $50\%$, although small proportions work best: with $n$ fixed, each additional calibration point is one fewer training point, so the prediction sets grow, while coverage becomes only slightly less variable (Figures~\ref{fig:app-dp-vary-cal-mm}--\ref{fig:app-dp-cov-boxplot}).

\subsection{Numerical Experiments with Real Data} \label{sec:real-exp}

We next apply all methods to the CelebFaces Attributes (CelebA) dataset \citep{Liu2015CelebA}, which contains 202{,}599 facial images from 10{,}177 celebrity identities. We detect and align faces with a pre-trained multitask cascaded convolutional network \citep{Zhang2016MTCNN} and extract features with FaceNet \citep{Schroff2015FaceNet}. To reduce the computational cost of repeated experiments, we randomly sample 2{,}000 identities and retain all their images, then split this subset into training, calibration, and test sets as described in Section~\ref{sec:empirical-setup}. All methods target $80\%$ marginal coverage.

Figure~\ref{fig:real-main-four-panel} reports results as a function of the reference sample size $n$, i.e., the total number of training and calibration points, from 2{,}000 to 6{,}000, averaged over 1{,}000 test points and 20 independent repetitions.
CGTC+ achieves valid marginal coverage at all sizes, whereas the standard method undercovers when $n$ is small, because many test identities are then unseen.
With selective splitting, the coverage of CGTC+ stays essentially at the nominal level, its prediction sets are substantially smaller than with random splitting, and its use of the joker decreases with $n$, along with the true proportion of unseen identities.
As with synthetic data, CGTC+ remains valid across label frequencies (Figure~\ref{fig:real-app-conditional-coverage} in Appendix~\ref{app:experiments}), the data-driven allocation assigns more of the budget to the test of $H_{\text{unseen}}$ when new labels are more frequent (Figure~\ref{fig:real-app-alpha-allocation}), and the feature-based $p$-value includes the joker least often (Figure~\ref{fig:real-prop-joker-mixed-labels-80}).

\begin{figure}[!htb]
    \centering
    \includegraphics[width=1\textwidth]{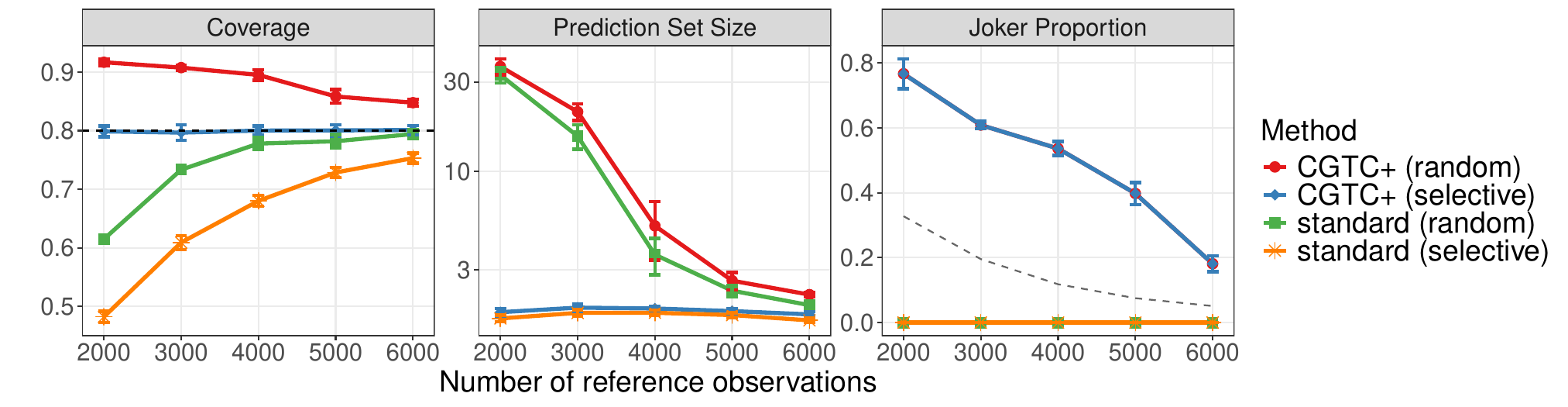}
    \caption{Performance of conformal prediction sets on CelebA face recognition data, as a function of the reference sample size. The nominal coverage level is $1-\alpha = 0.8$. Other details are as in Figure~\ref{fig:dp-main-four-panel}.}
    \label{fig:real-main-four-panel}
\end{figure}

Figure~\ref{fig:real-mm-plugin-vs-original} compares our method with (CGTC+) and without (CGTC) the missing-mass adjustment of Section~\ref{sec:plugin-adjustment}, both with selective splitting.
CGTC remains valid but is conservative, especially at small reference sample sizes (Figures~\ref{fig:app-real-orig-four-panel}--\ref{fig:app-real-orig-cond-cov}), whereas CGTC+ achieves coverage close to the nominal level, with smaller prediction sets that include the joker less often.
These conclusions also hold when the total number of identities varies (Figure~\ref{fig:app-real-four-panel-nlab}), and coverage remains close to the nominal level across values of the preference parameter $\lambda$ in~\eqref{eq:alpha_allocation_loss}, which trades off the number of seen labels against the use of the joker (Figure~\ref{fig:app-lambda-tradeoff}).

\begin{figure}[!htb]
    \centering
    \includegraphics[width=0.9\textwidth]{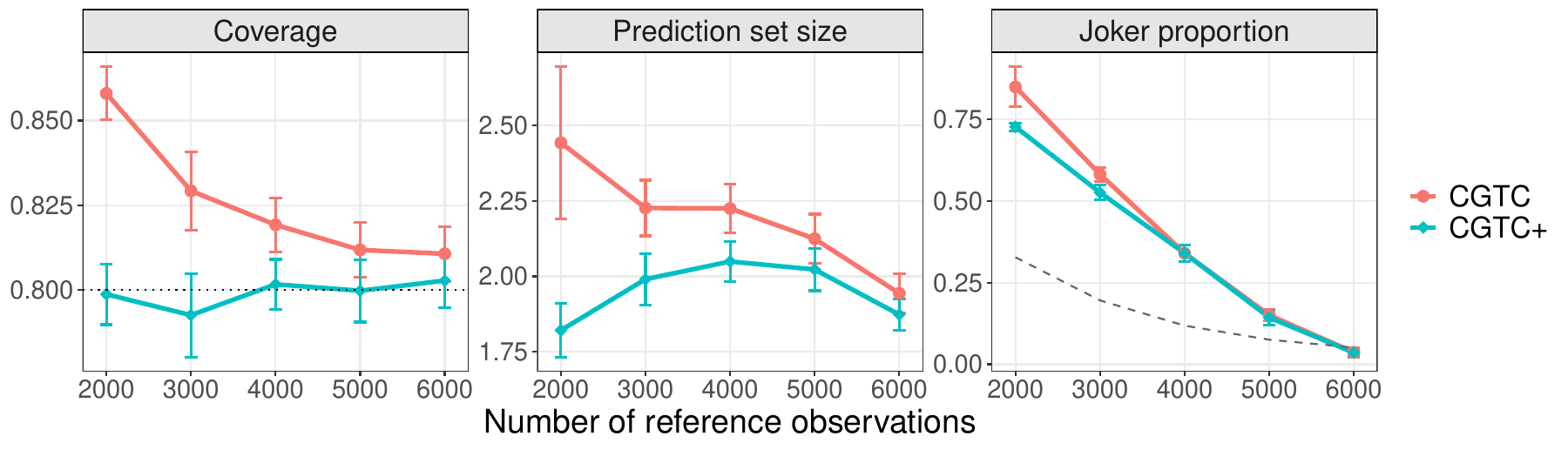}
    \caption{Comparison of CGTC+ and CGTC, both with selective splitting, on CelebA data, as a function of the reference sample size: marginal coverage (left), average number of seen labels in the prediction sets (middle), and proportion of prediction sets including the joker (right). The dotted line marks the nominal $80\%$ coverage, and the dashed line shows the true proportion of unseen test labels. Error bars represent $\pm 1.96$ standard errors.}
    \label{fig:real-mm-plugin-vs-original}
\end{figure}

Finally, Appendix~\ref{app:openset-findings} compares our method with the direct conformalization of open-set classifiers (Appendix~\ref{sec:why-not-open-set-classifier}). Even the best of these benchmarks, recalibrated with the Good--Turing estimate, produces much larger and hence less informative prediction sets, because its unknown class implicitly includes every identity observed only in the calibration set.

\section{Discussion} \label{sec:discussion}

This paper shows that existing conformal classification methods implicitly assume, beyond exchangeability, that all labels are observed in the data, and that marginal coverage may fail otherwise.
To address this limitation, we extend conformal prediction to open-set classification, revealing a connection to \citet{good1953population}.
Our method complements existing machine learning approaches, since it can be used with any classifier, including open-set classifiers.
Some of our methodology may also be useful beyond open-set classification; for example, selective sample splitting applies to any closed-set classification task with severe class imbalance.

One limitation of this work is that the version of our method with finite-sample guarantees, which uses a fixed allocation of $\alpha$, is not the one that performs best in practice. The best-performing version, which relies on the missing-mass adjustment and on a cross-validated allocation, is justified only asymptotically or empirically, and developing data-driven allocations with finite-sample guarantees remains an open problem.

There are several other directions for future work. First, relaxing exchangeability \citep{barber2023conformal} would allow the method to handle distribution shift, which is another reason why test points may have new labels; for example, under label shift, classes that are rare or absent in the reference sample may become common at test time, which could be addressed through additional reweighting \citep{tibshirani2019conformal}. Second, our method could be adapted to online settings \citep{vovk2005algorithmic,angelopoulos2024online}, where the set of observed labels grows over time, and to structured label spaces, such as ordinal \citep{lu2022improving} or hierarchical \citep{mortier2025conformal} labels. Third, selective splitting could be extended to cross-validation-based conformal methods \citep{vovk2015cross,barber2021predictive}, for example by assigning observations to folds based on the frequency of their labels, so that rare labels are never held out, with suitable reweighting.

\section*{Software Availability}

A software package implementing the methods described in this paper and containing code needed to reproduce all numerical experiments is available online at \url{https://github.com/TianminX/open-set-conformal-classification}. 

\section*{Acknowledgments}

T.~X., Y.~Z., and M.~S. were partly supported by a USC-Capital One CREDIF award. M.~S. was also partly supported by a Google Research Scholar Award.


\section*{Declaration of the use of generative AI and AI-assisted technologies}

The authors used ChatGPT (GPT-5.5) and Claude (Opus 5.5) to edit language, proofread, assist in writing code, and facilitate literature searches. After using these tools, the authors reviewed and edited the content as necessary and take full responsibility for the content of the publication.

\bibliography{ref}

@article{liang2026structured,
  title={Structured conformal inference for matrix completion with applications to group recommender systems},
  author={Liang, Ziyi and Xie, Tianmin and Tong, Xin and Sesia, Matteo},
  journal={Journal of the American Statistical Association},
  pages={1--12},
  year={2026},
  publisher={Taylor \& Francis}
}

@article{mortier2025conformal,
  title={Conformal prediction in hierarchical classification},
  author={Mortier, Thomas and Javanmardi, Alireza and Sale, Yusuf and H{\"u}llermeier, Eyke and Waegeman, Willem},
  journal={arXiv preprint arXiv:2501.19038},
  year={2025}
}

@inproceedings{lu2022improving,
  title={Improving trustworthiness of AI disease severity rating in medical imaging with ordinal conformal prediction sets},
  author={Lu, Charles and Angelopoulos, Anastasios N and Pomerantz, Stuart},
  booktitle={International Conference on Medical Image Computing and Computer-Assisted Intervention},
  pages={545--554},
  year={2022},
  organization={Springer}
}

@article{mcallester2003concentration,
  title={Concentration inequalities for the missing mass and for histogram rule error},
  author={McAllester, David and Ortiz, Luis},
  journal={Journal of Machine Learning Research},
  volume={4},
  number={Oct},
  pages={895--911},
  year={2003}
}

@article{jin2023selection,
  title={Selection by prediction with conformal p-values},
  author={Jin, Ying and Cand{\`e}s, Emmanuel J},
  journal={Journal of Machine Learning Research},
  volume={24},
  number={244},
  pages={1--41},
  year={2023}
}

@article{pedregosa2011scikit,
  title={Scikit-learn: Machine learning in {Python}},
  author={Pedregosa, Fabian and Varoquaux, Ga{\"e}l and Gramfort, Alexandre and Michel, Vincent and Thirion, Bertrand and Grisel, Olivier and Blondel, Mathieu and Prettenhofer, Peter and Weiss, Ron and Dubourg, Vincent and others},
  journal={Journal of Machine Learning Research},
  volume={12},
  pages={2825--2830},
  year={2011},
  publisher={JMLR.org}
}

@article{angelopoulos2024online,
  title={Online conformal prediction with decaying step sizes},
  author={Angelopoulos, Anastasios N and Barber, Rina Foygel and Bates, Stephen},
  journal={arXiv preprint arXiv:2402.01139},
  year={2024}
}

@article{barber2023conformal,
  title={Conformal prediction beyond exchangeability},
  author={Barber, Rina Foygel and Cand{\`e}s, Emmanuel J and Ramdas, Aaditya and Tibshirani, Ryan J},
  journal={The Annals of Statistics},
  volume={51},
  number={2},
  pages={816--845},
  year={2023},
  publisher={Institute of Mathematical Statistics}
}

@article{bla73,
  title={Ferguson distributions via {P\'olya} urn schemes},
  author={Blackwell, David and MacQueen, James B},
  journal={The Annals of Statistics},
  volume={1},
  number={2},
  pages={353--355},
  year={1973},
  publisher={Institute of Mathematical Statistics}
}

@article{lo84,
  title={On a class of {Bayesian} nonparametric estimates: {I}. {D}ensity estimates},
  author={Lo, Albert Y},
  journal={The Annals of Statistics},
  volume={12},
  number={1},
  pages={351--357},
  year={1984},
  publisher={Institute of Mathematical Statistics}
}

@article{bra10,
  title={Fundamental problem of forensic mathematics - the evidential value of a rare haplotype},
  author={Brenner, Charles H},
  journal={Forensic Science International: Genetics},
  volume={4},
  number={5},
  pages={281--291},
  year={2010},
  publisher={Elsevier}
}

@article{tibshirani2019conformal,
  title={Conformal prediction under covariate shift},
  author={Tibshirani, Ryan J and Foygel Barber, Rina and Cand{\`e}s, Emmanuel and Ramdas, Aaditya},
  journal={Advances in Neural Information Processing Systems},
  volume={32},
  year={2019}
}

@article{sesia2023conformal,
  title={Conformal frequency estimation using discrete sketched data with coverage for distinct queries},
  author={Sesia, Matteo and Favaro, Stefano and Dobriban, Edgar},
  journal={Journal of Machine Learning Research},
  volume={24},
  number={348},
  pages={1--80},
  year={2023}
}

@book{zipf2016human,
  title={Human behavior and the principle of least effort: An introduction to human ecology},
  author={Zipf, George Kingsley},
  year={2016},
  publisher={Ravenio books}
}

@article{ferrer2001two,
  title={Two regimes in the frequency of words and the origins of complex Lexicons: Zipf's law revisited},
  author={Ferrer i Cancho, Ramon and Sol{\'e}, Ricard V},
  journal={Journal of Quantitative Linguistics},
  volume={8},
  number={3},
  pages={165--173},
  year={2001},
  publisher={Taylor \& Francis}
}

@article{pimentel2014review,
  title={A review of novelty detection},
  author={Pimentel, Marco AF and Clifton, David A and Clifton, Lei and Tarassenko, Lionel},
  journal={Signal processing},
  volume={99},
  pages={215--249},
  year={2014},
  publisher={Elsevier}
}

@techreport{moya1993one,
  title={One-class classifier networks for target recognition applications},
  author={Moya, Mary M and Koch, Mark W and Hostetler, Larry D},
  year={1993},
  institution={Sandia National Labs., Albuquerque, NM (United States)}
}

@article{ferguson1973bayesian,
  title={A {Bayesian} analysis of some nonparametric problems},
  author={Ferguson, Thomas S},
  journal={The Annals of Statistics},
  pages={209--230},
  year={1973},
  publisher={JSTOR}
}

@article{marandon2024adaptive,
  title={Adaptive novelty detection with false discovery rate guarantee},
  author={Marandon, Ariane and Lei, Lihua and Mary, David and Roquain, Etienne},
  journal={The Annals of Statistics},
  volume={52},
  number={1},
  pages={157--183},
  year={2024},
  publisher={Institute of Mathematical Statistics}
}

@article{liang2024integrative,
  title={Integrative conformal p-values for out-of-distribution testing with labelled outliers},
  author={Liang, Ziyi and Sesia, Matteo and Sun, Wenguang},
  journal={Journal of the Royal Statistical Society Series B: Statistical Methodology},
  volume={86},
  number={3},
  pages={671--693},
  year={2024},
  publisher={Oxford University Press UK}
}

@article{bates2023testing,
  title={Testing for outliers with conformal p-values},
  author={Bates, Stephen and Cand{\`e}s, Emmanuel and Lei, Lihua and Romano, Yaniv and Sesia, Matteo},
  journal={The Annals of Statistics},
  volume={51},
  number={1},
  pages={149--178},
  year={2023},
  publisher={Institute of Mathematical Statistics}
}

@article{sadinle2019least,
  title={Least ambiguous set-valued classifiers with bounded error levels},
  author={Sadinle, Mauricio and Lei, Jing and Wasserman, Larry},
  journal={Journal of the American Statistical Association},
  volume={114},
  number={525},
  pages={223--234},
  year={2019},
  publisher={Taylor \& Francis}
}

@article{angelopoulos2024theoretical,
  title={Theoretical foundations of conformal prediction},
  author={Angelopoulos, Anastasios N and Barber, Rina Foygel and Bates, Stephen},
  journal={arXiv preprint arXiv:2411.11824},
  year={2024}
}

@article{foygel2021limits,
  title={The limits of distribution-free conditional predictive inference},
  author={Barber, Rina Foygel and Cand{\`e}s, Emmanuel J and Ramdas, Aaditya and Tibshirani, Ryan J},
  journal={Information and Inference: A Journal of the IMA},
  volume={10},
  number={2},
  pages={455--482},
  year={2021},
  publisher={Oxford University Press}
}

@inproceedings{bartusiak2022transformer,
  title={Transformer-based speech synthesizer attribution in an open set scenario},
  author={Bartusiak, Emily R and Delp, Edward J},
  booktitle={2022 21st IEEE International Conference on Machine Learning and Applications (ICMLA)},
  pages={329--336},
  year={2022},
  organization={IEEE}
}

@article{junior2021open,
  title={Open-set support vector machines},
  author={J{\'u}nior, Pedro Ribeiro Mendes and Boult, Terrance E and Wainer, Jacques and Rocha, Anderson},
  journal={IEEE Transactions on Systems, Man, and Cybernetics: Systems},
  volume={52},
  number={6},
  pages={3785--3798},
  year={2021},
  publisher={IEEE}
}

@inproceedings{feng2024open,
  title={Open set recognition for random forest},
  author={Feng, Guanchao and Desai, Dhruv and Pasquali, Stefano and Mehta, Dhagash},
  booktitle={Proceedings of the 5th ACM International Conference on AI in Finance},
  pages={45--53},
  year={2024}
}

@article{bergman2020classification,
  title={Classification-based anomaly detection for general data},
  author={Bergman, Liron and Hoshen, Yedid},
  journal={arXiv preprint arXiv:2005.02359},
  year={2020}
}

@article{jleed2020open,
  title={Open set audio recognition for multi-class classification with rejection},
  author={Jleed, Hitham and Bouchard, Martin},
  journal={IEEE Access},
  volume={8},
  pages={146523--146534},
  year={2020},
  publisher={IEEE}
}

@inproceedings{cruz2017open,
  title={Open set intrusion recognition for fine-grained attack categorization},
  author={Cruz, Steve and Coleman, Cora and Rudd, Ethan M and Boult, Terrance E},
  booktitle={2017 IEEE International Symposium on Technologies for Homeland Security (HST)},
  pages={1--6},
  year={2017},
  organization={IEEE}
}

@article{shen2024open,
  title={An open set model for pest identification},
  author={Shen, Yefeng and Hossain, Md Zakir and Ahmed, Khandaker Asif and Rahman, Shafin},
  journal={Computational Biology and Chemistry},
  volume={108},
  pages={108002},
  year={2024},
  publisher={Elsevier}
}

@article{cao2023open,
  title={Open-set recognition of breast cancer treatments},
  author={Cao, Alexander and Klabjan, Diego and Luo, Yuan},
  journal={Artificial intelligence in medicine},
  volume={135},
  pages={102451},
  year={2023},
  publisher={Elsevier}
}

@article{sun2022open,
  title={Open-set iris recognition based on deep learning},
  author={Sun, Jie and Zhao, Shipeng and Miao, Sheng and Wang, Xuan and Yu, Yanan},
  journal={IET image processing},
  volume={16},
  number={9},
  pages={2361--2372},
  year={2022},
  publisher={Wiley Online Library}
}

@inproceedings{jain2014multi,
  title={Multi-class open set recognition using probability of inclusion},
  author={Jain, Lalit P and Scheirer, Walter J and Boult, Terrance E},
  booktitle={European conference on computer vision},
  pages={393--409},
  year={2014},
  organization={Springer}
}

@article{li2005open,
  title={Open set face recognition using transduction},
  author={Li, Fayin and Wechsler, Harry},
  journal={IEEE Transactions on Pattern Analysis and Machine Intelligence},
  volume={27},
  number={11},
  pages={1686--1697},
  year={2005},
  publisher={IEEE}
}

@book{pitman2006combinatorial,
  title={Combinatorial stochastic processes: {\'E}cole d'{\'E}t{\'e} de Probabilit{\'e}s de {Saint-Flour} {XXXII}, 2002},
  author={Pitman, Jim},
  year={2006},
  publisher={Springer}
}

@article{wang2023classificationOOD,
  author  = {Zhou Wang and Xingye Qiao},
  title   = {Set-valued Classification with Out-of-distribution Detection for Many Classes},
  journal = {Journal of Machine Learning Research},
  year    = {2023},
  volume  = {24},
  number  = {375},
  pages   = {1--39}
}

@Article{Guan2022classificationOOD,
  author={Leying Guan and Robert Tibshirani},
  title={Prediction and outlier detection in classification problems},
  journal={Journal of the Royal Statistical Society Series B: Statistical Methodology},
  year={2022},
  volume={84},
  number={2},
  pages={524--546}
}

@inproceedings{Ding2023manyclasses,
 author = {Ding, Tiffany and Angelopoulos, Anastasios and Bates, Stephen and Jordan, Michael and Tibshirani, Ryan J},
 booktitle = {Advances in Neural Information Processing Systems},
 editor = {A. Oh and T. Naumann and A. Globerson and K. Saenko and M. Hardt and S. Levine},
 pages = {64555--64576},
 publisher = {Curran Associates, Inc.},
 title = {Class-Conditional Conformal Prediction with Many Classes},
 volume = {36},
 year = {2023}
}

@inproceedings{romano2020classification,
  title={Classification with valid and adaptive coverage},
  author={Romano, Yaniv and Sesia, Matteo and Cand{\`e}s, Emmanuel J.},
  booktitle={Advances in Neural Information Processing Systems},
  volume={33},
  pages={3581--3591},
  year={2020}
}

@article{hechtlinger2018cautious,
    title={Cautious Deep Learning},
    author={Yotam Hechtlinger and Barnabás Póczos and Larry Wasserman},
    year={2018},
  journal={arXiv preprint arXiv:1805.09460}
}

@book{vovk2005algorithmic,
  title={Algorithmic learning in a random world},
  author={Vovk, Vladimir and Gammerman, Alexander and Shafer, Glenn},
  volume={29},
  year={2005},
  publisher={Springer}
}

@article{good1953population,
  title={The population frequencies of species and the estimation of population parameters},
  author={Good, Irving J},
  journal={Biometrika},
  volume={40},
  number={3-4},
  pages={237--264},
  year={1953},
  publisher={Oxford University Press}
}

@article{scheirer2012toward,
  title={Toward open set recognition},
  author={Scheirer, Walter J and de Rezende Rocha, Anderson and Sapkota, Archana and Boult, Terrance E},
  journal={IEEE Transactions on Pattern Analysis and Machine Intelligence},
  volume={35},
  number={7},
  pages={1757--1772},
  year={2012},
  publisher={IEEE}
}

@article{geng2020recent,
  title={Recent advances in open set recognition: A survey},
  author={Geng, Chuanxing and Huang, Sheng-jun and Chen, Songcan},
  journal={IEEE Transactions on Pattern Analysis and Machine Intelligence},
  volume={43},
  number={10},
  pages={3614--3631},
  year={2020},
  publisher={IEEE}
}

@article{salehi2021unified,
  title={A unified survey on anomaly, novelty, open-set, and out-of-distribution detection: Solutions and future challenges},
  author={Salehi, Mohammadreza and Mirzaei, Hossein and Hendrycks, Dan and Li, Yixuan and Rohban, Mohammad Hossein and Sabokrou, Mohammad},
  journal={arXiv preprint arXiv:2110.14051},
  year={2021}
}

@article{Scheirer2014,
  author  = {Walter J. Scheirer and Lalit P. Jain and Terrance E. Boult},
  title   = {Probability Models for Open Set Recognition},
  journal = {IEEE Transactions on Pattern Analysis and Machine Intelligence},
  volume  = {36},
  number  = {11},
  pages   = {2317--2324},
  year    = {2014},
  doi     = {10.1109/TPAMI.2014.2321392}
}

@inproceedings{Bendale2016,
  author    = {Abhijit Bendale and Terrance E. Boult},
  title     = {Towards Open Set Deep Networks},
  booktitle = {Proceedings of the IEEE Conference on Computer Vision and Pattern Recognition (CVPR)},
  year      = {2016},
  pages     = {1563--1572},
  doi       = {10.1109/CVPR.2016.173}
}

@inproceedings{Yoshihashi2019,
  author    = {Ryota Yoshihashi and Wen Shao and Rei Kawakami and Shaodi You and Makoto Iida and Takeshi Naemura},
  title     = {Classification-Reconstruction Learning for Open-Set Recognition},
  booktitle = {Proceedings of the IEEE/CVF Conference on Computer Vision and Pattern Recognition (CVPR)},
  year      = {2019},
  pages     = {4016--4025}
}

@inproceedings{Oza2019,
  author    = {Poojan Oza and Vishal M. Patel},
  title     = {{C2AE}: Class Conditioned Auto-Encoder for Open-Set Recognition},
  booktitle = {Proceedings of the IEEE/CVF Conference on Computer Vision and Pattern Recognition (CVPR)},
  year      = {2019},
  pages     = {2307--2316}
}

@inproceedings{Dhamija2018,
  author    = {Akshay Raj Dhamija and Manuel G{\"u}nther and Terrance E. Boult},
  title     = {Reducing Network Agnostophobia},
  booktitle = {Advances in Neural Information Processing Systems (NeurIPS)},
  year      = {2018}
}

@article{HolUE2024,
  title = {Holistic Uncertainty Estimation For Open-Set Recognition},
  author={Erlygin, Leonid and Zaytsev, Alexey},
  journal={arXiv preprint arXiv:2408.14229},
  year={2024}
}

@article{Zhang2016MTCNN,
  author    = {Kaipeng Zhang and Zhanpeng Zhang and Zhifeng Li and Yu Qiao},
  title     = {Joint Face Detection and Alignment Using Multitask Cascaded Convolutional Networks},
  journal   = {IEEE Signal Processing Letters},
  volume    = {23},
  number    = {10},
  pages     = {1499--1503},
  year      = {2016},
  month     = {Oct},
  doi       = {10.1109/LSP.2016.2603342},
  issn      = {1070-9908}
}

@inproceedings{Schroff2015FaceNet,
  author    = {Florian Schroff and Dmitry Kalenichenko and James Philbin},
  title     = {{FaceNet}: A Unified Embedding for Face Recognition and Clustering},
  booktitle = {Proceedings of the IEEE Conference on Computer Vision and Pattern Recognition (CVPR)},
  pages     = {815--823},
  year      = {2015},
  doi       = {10.1109/CVPR.2015.7298682}
}

@inproceedings{Liu2015CelebA,
  author    = {Ziwei Liu and Ping Luo and Xiaogang Wang and Xiaoou Tang},
  title     = {Deep Learning Face Attributes in the Wild},
  booktitle = {Proceedings of the IEEE International Conference on Computer Vision (ICCV)},
  pages     = {3730--3738},
  year      = {2015},
  doi       = {10.1109/ICCV.2015.425}
}

@inproceedings{
hendrycks2017baseline,
title={A Baseline for Detecting Misclassified and Out-of-Distribution Examples in Neural Networks},
author={Dan Hendrycks and Kevin Gimpel},
booktitle={International Conference on Learning Representations},
year={2017},
url={https://openreview.net/forum?id=Hkg4TI9xl}
}

@inproceedings{
liang2018enhancing,
title={Enhancing The Reliability of Out-of-distribution Image Detection in Neural Networks},
author={Shiyu Liang and Yixuan Li and R. Srikant},
booktitle={International Conference on Learning Representations},
year={2018},
url={https://openreview.net/forum?id=H1VGkIxRZ},
}

@inproceedings{zhou2021learning,
  title={Learning placeholders for open-set recognition},
  author={Zhou, Da-Wei and Ye, Han-Jia and Zhan, De-Chuan},
  booktitle={Proceedings of the IEEE/CVF conference on computer vision and pattern recognition},
  pages={4401--4410},
  year={2021}
}

@inproceedings{ge2017generative,
  title={Generative {OpenMax} for Multi-Class Open Set Classification},
  author={Ge, Zongyuan and Demyanov, Sergey and Chen, Zetao and Garnavi, Rahil},
  booktitle={BMVC 2017},
  year={2017}
}

@inproceedings{
vaze2022openset,
title={Open-Set Recognition: A Good Closed-Set Classifier is All You Need},
author={Sagar Vaze and Kai Han and Andrea Vedaldi and Andrew Zisserman},
booktitle={International Conference on Learning Representations},
year={2022},
url={https://openreview.net/forum?id=5hLP5JY9S2d}
}

@article{sesia2026elements,
      title={Elements of Conformal Prediction}, 
      author={Matteo Sesia and Stefano Favaro},
      year={2026},
  journal={arXiv preprint arXiv:2603.23923}
}

@article{sethuraman1994constructive,
  title={A constructive definition of {Dirichlet} priors},
  author={Sethuraman, Jayaram},
  journal={Statistica sinica},
  pages={639--650},
  year={1994},
  publisher={JSTOR}
}

@inproceedings{sun2022knn,
  title={Out-of-distribution detection with deep nearest neighbors},
  author={Sun, Yiyou and Ming, Yifei and Zhu, Xiaojin and Li, Yixuan},
  booktitle={International conference on machine learning},
  pages={20827--20840},
  year={2022},
  organization={PMLR}
}

@article{liu2020energy,
  title={Energy-based out-of-distribution detection},
  author={Liu, Weitang and Wang, Xiaoyun and Owens, John and Li, Yixuan},
  journal={Advances in Neural Information Processing Systems},
  volume={33},
  pages={21464--21475},
  year={2020}
}

@article{lee2018simple,
  title={A simple unified framework for detecting out-of-distribution samples and adversarial attacks},
  author={Lee, Kimin and Lee, Kibok and Lee, Honglak and Shin, Jinwoo},
  journal={Advances in Neural Information Processing Systems},
  volume={31},
  year={2018}
}

@article{rudd2018extreme,
  author={Rudd, Ethan M. and Jain, Lalit P. and Scheirer, Walter J. and Boult, Terrance E.},
  journal={IEEE Transactions on Pattern Analysis and Machine Intelligence}, 
  title={The Extreme Value Machine}, 
  year={2018},
  volume={40},
  number={3},
  pages={762-768},
  doi={10.1109/TPAMI.2017.2707495}}

@article{chen2021adversarial,
  title={Adversarial reciprocal points learning for open set recognition},
  author={Chen, Guangyao and Peng, Peixi and Wang, Xiangqian and Tian, Yonghong},
  journal={IEEE Transactions on Pattern Analysis and Machine Intelligence},
  volume={44},
  number={11},
  pages={8065--8081},
  year={2021},
  publisher={IEEE}
}

@inproceedings{kong2021opengan,
  title={{OpenGAN}: Open-set recognition via open data generation},
  author={Kong, Shu and Ramanan, Deva},
  booktitle={Proceedings of the IEEE/CVF international conference on computer vision},
  pages={813--822},
  year={2021}
}

@article{vovk2015cross,
  title={Cross-conformal predictors},
  author={Vovk, Vladimir},
  journal={Annals of Mathematics and Artificial Intelligence},
  volume={74},
  number={1},
  pages={9--28},
  year={2015},
  publisher={Springer}
}

@article{barber2021predictive,
  title={Predictive inference with the jackknife+},
  author={Barber, Rina Foygel and Cand{\`e}s, Emmanuel J and Ramdas, Aaditya and Tibshirani, Ryan J},
  journal={The Annals of Statistics},
  volume={49},
  number={1},
  pages={486--507},
  year={2021},
  publisher={JSTOR}
}

@book{kish1965survey,
  title={Survey Sampling},
  author={Kish, Leslie},
  year={1965},
  publisher={Wiley},
  address={New York}
}

@article{kish1992weighting,
  title={Weighting for unequal {$P_i$}},
  author={Kish, Leslie},
  journal={Journal of Official Statistics},
  volume={8},
  number={2},
  pages={183--200},
  year={1992}
}

@techreport{kong1992note,
  title={A note on importance sampling using standardized weights},
  author={Kong, Augustine},
  institution={Department of Statistics, University of Chicago},
  type={Technical Report},
  number={348},
  year={1992}
}

@book{owen2013monte,
  title={Monte {C}arlo Theory, Methods and Examples},
  author={Owen, Art B.},
  year={2013},
  note={Available at https://artowen.su.domains/mc/}
}

@inproceedings{vovk2012conditional,
  title={Conditional validity of inductive conformal predictors},
  author={Vovk, Vladimir},
  booktitle={Proceedings of the Asian Conference on Machine Learning},
  series={Proceedings of Machine Learning Research},
  volume={25},
  pages={475--490},
  year={2012}
}

\appendix
\renewcommand{\thesection}{A\arabic{section}}
\renewcommand{\theequation}{A\arabic{equation}}
\renewcommand{\thetheorem}{A\arabic{theorem}}
\renewcommand{\thecorollary}{A\arabic{corollary}}
\renewcommand{\theproposition}{A\arabic{proposition}}
\renewcommand{\theassumption}{A\arabic{assumption}}
\renewcommand{\thelemma}{A\arabic{lemma}}
\renewcommand{\thetable}{A\arabic{table}}
\renewcommand{\thefigure}{A\arabic{figure}}
\renewcommand{\thealgocf}{A\arabic{algocf}}
\setcounter{figure}{0}
\setcounter{table}{0}
\setcounter{proposition}{0}
\setcounter{assumption}{0}
\setcounter{theorem}{0}
\setcounter{lemma}{0}

\section{Why Not Simply Conformalize an Open-Set Classifier?} \label{sec:why-not-open-set-classifier}

The method proposed in this paper can leverage any classifier to compute nonconformity scores, and is therefore complementary to prior work on open-set recognition. That being said, a natural alternative is to conformalize an open-set classifier directly, treating its ``unknown'' class as an ordinary label. This appendix reviews open-set classifiers and shows that their direct conformalization achieves marginal coverage, but at the cost of producing less informative prediction sets than ours. We also note that the unknown-class probability estimates produced by existing open-set classifiers are often miscalibrated, and show how the classical Good--Turing estimate of the missing mass can be used to recalibrate them, outside our conformal prediction framework and thus without finite-sample guarantees. Appendix~\ref{app:openset-findings} compares the two approaches empirically.

\subsection{Overview of Open-Set Classification Models}
\label{app:osr-overview}

Open-set classifiers seek to recognize inputs from the $K$ classes observed in the training data while flagging inputs from new classes \citep{scheirer2012toward, geng2020recent, salehi2021unified}. They typically combine a closed-set classifier, which estimates the conditional probabilities of the training labels, with a feature-dependent \emph{novelty score} $s_{\text{new}}(x)$, which ranks inputs by the evidence that they belong to none of the training classes. This score may be reported either as a statistic to be thresholded at a user-chosen operating point, or as an estimated probability $\widehat{p}_{\joker_{\mathrm{train}}}(x)$ of a $(K{+}1)$-st ``unknown'' class. Existing methods differ mainly in the evidence from which they compute their novelty score and can be categorized into at least three broad families.

\emph{Confidence-based methods} derive the novelty score from the closed-set classifier itself, flagging inputs for which the classifier is uncertain. Examples include one minus the maximum softmax probability \citep{hendrycks2017baseline, vaze2022openset}, its refinement by temperature scaling and input perturbations in ODIN \citep{liang2018enhancing}, and the energy score computed from the logits \citep{liu2020energy}. However, this heuristic is not always well justified, since the classifier may be uncertain even about inputs from seen classes, for example because the conditional distribution of the label given the features has intrinsic noise, or because some classes have few training examples.

\emph{Distance- and density-based methods} derive the novelty score from distances or densities in some feature space, flagging inputs that lie far from all known classes \citep{scheirer2012toward}. Some use simple global measures, such as the distance to the $k$-th nearest training point \citep{sun2022knn} or a one-class classifier fitted to the training features. Others model each known class separately, using for example Mahalanobis distances to class centroids \citep{lee2018simple}, autoencoder reconstructions \citep{Yoshihashi2019, Oza2019}, or extreme-value models of the distances to each class \citep{Scheirer2014, jain2014multi, rudd2018extreme}, as in OpenMax \citep{Bendale2016}. A limitation of per-class models is that they require many examples per class, which is not the case in our CelebA data, where most identities have only $2$ or $3$ images.

\emph{Methods with a trained unknown class} derive the novelty score from a classifier trained to recognize artificial examples. These examples may be synthesized by generative or adversarial models \citep{ge2017generative, kong2021opengan, chen2021adversarial}, interpolated between training points of different classes, as in PROSER \citep{zhou2021learning}, or drawn from background data \citep{Dhamija2018}. The resulting unknown-class probability estimate depends on how these artificial examples are generated and weighted during training.

All these methods are designed and evaluated mainly for their ability to distinguish known from unknown inputs, typically through the area under the ROC curve \citep{Dhamija2018, vaze2022openset}. This criterion depends only on how inputs are ranked by $s_{\text{new}}(x)$, and their estimated unknown-class probabilities are not necessarily well calibrated. Therefore, these models may be useful to provide score functions for our feature-based conformal Good--Turing $p$-values (Section~\ref{sec:pvals-feature-based}), but they cannot replace our methodology, which provides the calibration they lack.

\subsection{Direct Conformalization of Open-Set Classifiers} \label{app:direct-conformalization}

Split conformal prediction divides the reference sample $[n]$ into a training set $\mathcal{D}^{\text{train}}$, used to fit the classifier, and a calibration set $\mathcal{D}^{\text{cal}}$. Let $\mathcal{Y}_{\mathrm{train}} := \mathrm{unique}(\{Y_i\}_{i \in \mathcal{D}^{\text{train}}})$ denote the labels observed in the training set, and let $\joker_{\mathrm{train}}$ be the unknown class of an open-set classifier fitted on it, representing $\mathcal{Y} \setminus \mathcal{Y}_{\mathrm{train}}$.
The direct construction treats $\joker_{\mathrm{train}}$ as an ordinary label, assigned to all calibration points whose labels are not in $\mathcal{Y}_{\mathrm{train}}$, and applies split conformal classification to output a subset of $\mathcal{Y}_{\mathrm{train}} \cup \{\joker_{\mathrm{train}}\}$, as outlined in Algorithm~\ref{alg:open-set-plug-in}. We denote the resulting prediction set by $\hat{C}^{\dagger}_{\alpha}(X_{n+1} \mid \mathcal{Y}_{\mathrm{train}})$.

\begin{algorithm}[!htb]
  \caption{Split Conformal Classification with an Open-Set Classifier}
  \label{alg:open-set-plug-in}
  \SetKwInOut{Input}{Input{}}
  \Input{%
    Data $\{(X_i, Y_i)\}_{i=1}^{n}$;\\
    Test features $X_{n+1}$; \\
    Open-set classification algorithm $\mathcal{A}$; \\
    Miscoverage level $\alpha \in (0,1)$.
  }
  Randomly split $[n]$ into disjoint subsets $\mathcal{D}^{\text{train}}$ and $\mathcal{D}^{\text{cal}}$\;

  Define $\mathcal{Y}_{\mathrm{train}} := \mathrm{unique}(\{Y_i\}_{i \in \mathcal{D}^{\text{train}}})$, and let $\joker_{\mathrm{train}}$ represent $\mathcal{Y} \setminus \mathcal{Y}_{\mathrm{train}}$\;

  Fit $\mathcal{A}$ on $\{(X_i, Y_i)\}_{i \in \mathcal{D}^{\text{train}}}$, estimating $\widehat{p}_y(x)$ for all $y \in \mathcal{Y}_{\mathrm{train}} \cup \{\joker_{\mathrm{train}}\}$\;

  For each $i \in \mathcal{D}^{\text{cal}}$, set $\tilde{Y}_i = Y_i$ if $Y_i \in \mathcal{Y}_{\mathrm{train}}$ and $\tilde{Y}_i = \joker_{\mathrm{train}}$ otherwise, and compute the score $S_i = s(X_i, \tilde{Y}_i)$, e.g., as in~\eqref{eq:adaptive-scores} with label space $\mathcal{Y}_{\mathrm{train}} \cup \{\joker_{\mathrm{train}}\}$\;

  Compute $\hat{q}_{1-\alpha}$, the $\lceil (1-\alpha)(|\mathcal{D}^{\text{cal}}|+1) \rceil$-th smallest value among $\{S_i\}_{i \in \mathcal{D}^{\text{cal}}} \cup \{+\infty\}$\;

  \Return{$\hat{C}^{\dagger}_{\alpha}(X_{n+1} \mid \mathcal{Y}_{\mathrm{train}}) = \{y \in \mathcal{Y}_{\mathrm{train}} \cup \{\joker_{\mathrm{train}}\} : s(X_{n+1}, y) \leq \hat{q}_{1-\alpha}\}$}.
\end{algorithm}

This procedure guarantees marginal coverage at level $1-\alpha$, like our method, although the joker $\joker_{\mathrm{train}}$ has a different meaning that makes the prediction set $\hat{C}^{\dagger}_{\alpha}(X_{n+1} \mid \mathcal{Y}_{\mathrm{train}})$ less informative than ours.
To see this, let $\tilde{Y}_{n+1} := Y_{n+1}$ if $Y_{n+1} \in \mathcal{Y}_{\mathrm{train}}$ and $\tilde{Y}_{n+1} := \joker_{\mathrm{train}}$ otherwise. Conditionally on the training set, the pairs $\{(X_i, \tilde{Y}_i)\}_{i \in \mathcal{D}^{\text{cal}}}$ and $(X_{n+1}, \tilde{Y}_{n+1})$ are exchangeable, so the standard conformal prediction argument gives $\mathbb{P}[\tilde{Y}_{n+1} \in \hat{C}^{\dagger}_{\alpha}(X_{n+1} \mid \mathcal{Y}_{\mathrm{train}})] \geq 1-\alpha$.
The coverage event can be equivalently written as
\begin{align} \label{eq:training-level-coverage}
  \bigl\{ Y_{n+1} \in \hat{C}^{\dagger}_{\alpha}(X_{n+1} \mid \mathcal{Y}_{\mathrm{train}}) \bigr\}
  \;\cup\;
  \bigl\{ \joker_{\mathrm{train}} \in \hat{C}^{\dagger}_{\alpha}(X_{n+1} \mid \mathcal{Y}_{\mathrm{train}})
          \ \text{and}\ Y_{n+1} \notin \mathcal{Y}_{\mathrm{train}} \bigr\},
\end{align}
which highlights that $\joker_{\mathrm{train}}$ stands for every label outside the training set.
In other words, $\joker_{\mathrm{train}}$ lumps together the genuinely new labels captured by our joker $\joker$ and the \emph{calibration-only} labels in $\mathcal{Y}_n \setminus \mathcal{Y}_{\mathrm{train}}$.
This makes the prediction sets less informative, because including $\joker_{\mathrm{train}}$ is equivalent to including every calibration-only label, without saying which one may be correct.

To compare the two methods on a common footing, any set $\hat{C}^{\dagger} \subseteq \mathcal{Y}_{\mathrm{train}} \cup \{\joker_{\mathrm{train}}\}$ can be \emph{decoded} into a subset of $\mathcal{Y}_n \cup \{\joker\}$ by replacing $\joker_{\mathrm{train}}$ with what it stands for, namely all calibration-only labels plus our joker:
\begin{align} \label{eq:decode-map}
  \mathrm{dec}(\hat{C}^{\dagger})
  :=
  \begin{cases}
    \hat{C}^{\dagger}, & \text{if } \joker_{\mathrm{train}} \notin \hat{C}^{\dagger}, \\
    \bigl(\hat{C}^{\dagger} \cap \mathcal{Y}_{\mathrm{train}}\bigr) \cup \bigl(\mathcal{Y}_n \setminus \mathcal{Y}_{\mathrm{train}}\bigr) \cup \{\joker\}, & \text{if } \joker_{\mathrm{train}} \in \hat{C}^{\dagger}.
  \end{cases}
\end{align}
While logically equivalent to $\hat{C}^{\dagger}$, the decoded set $\mathrm{dec}(\hat{C}^{\dagger}) \subseteq \mathcal{Y}_{n} \cup \{\joker\}$ quantifies informativeness on the same scale as our prediction sets. Our experiments therefore report, where relevant, both the nominal size of $\hat{C}^{\dagger}_{\alpha}$ and its decoded size $|\mathrm{dec}(\hat{C}^{\dagger}_{\alpha})|$, excluding the joker $\joker$ from the count.

\subsection{Good--Turing Recalibration of Open-Set Classifiers}
\label{app:gt-recalibration}

As explained in Appendix~\ref{app:osr-overview}, the unknown-class probabilities $\widehat{p}_{\joker_{\mathrm{train}}}(x)$ estimated by open-set classifiers are often miscalibrated. This matters for direct conformalization, because the nonconformity scores compare $\widehat{p}_{\joker_{\mathrm{train}}}(x)$ with the estimated probabilities of the seen labels for each fixed $x$. Conformal prediction produces the most informative prediction sets when these probabilities are well calibrated \citep{romano2020classification}.
If $\widehat{p}_{\joker_{\mathrm{train}}}(x)$ is too large, the unknown class enters nearly every prediction set, and if it is too small, it is almost never included.

One way to address this is to rescale the estimated unknown-class probabilities so that their average over inputs matches the Good--Turing estimate of the probability that a test label is absent from the training set, while preserving the ordering of different inputs $x$ by their novelty scores $s_{\text{new}}(x)$. This estimate is $\widehat{\mu}_{\mathrm{T}} = M_1^{\mathrm{train}}/n_{\mathrm{train}}$, where $M_1^{\mathrm{train}}$ is the number of labels appearing exactly once in the training set and $n_{\mathrm{train}}$ is its size \citep{good1953population}.
The seen-class probabilities are then renormalized to sum to $1 - \widehat{p}_{\joker_{\mathrm{train}}}(x)$.

A simple reference approach, which we call the \emph{naive baseline}, ignores the features and sets $\widehat{p}_{\joker_{\mathrm{train}}}(x) = \widehat{\mu}_{\mathrm{T}}$ for all $x$; this matches the average but discards the ordering of inputs.
To preserve the ordering, we split the training set at random into two parts, containing $80\%$ and $20\%$ of the data, fit the open-set classifier on the first part, and compute its novelty scores $s_i$ on the second part, since scores computed on the data used for fitting would be biased.
The classifier used for prediction is then refit on the full training set.
The idea above can be implemented in different ways based on these held-out scores; we consider two.

\emph{Isotonic recalibration} (prefix Iso) estimates the probability that a point with score $s$ has a new label. For each held-out point $i$, let $z_i = 1$ if its label is absent from the first part of the training set, and $z_i = 0$ otherwise. We fit by isotonic regression the non-decreasing function $\widehat{g}$ minimizing $\sum_i \{z_i - \widehat{g}(s_i)\}^2$, and set $\widehat{p}_{\joker_{\mathrm{train}}}(x) = \min\{ c\, \widehat{g}(s_{\text{new}}(x)),\, 0.9\}$, where $c = \widehat{\mu}_{\mathrm{T}}/\bar{g}$ and $\bar{g}$ is the average of $\widehat{g}(s_i)$ over the held-out points. Since $\widehat{g}$ is non-decreasing, this preserves the ordering of inputs up to ties, and since it depends on the score only through its ranking, it is invariant to monotone transformations of the score.

\emph{Multiplicative recalibration} (prefix Recal) is simpler: it sets $\widehat{p}_{\joker_{\mathrm{train}}}(x) = \min\{ (\widehat{\mu}_{\mathrm{T}}/\bar{s})\, s_{\text{new}}(x),\, 0.9\}$, where $\bar{s}$ is the average held-out score. Unlike isotonic recalibration, this preserves the ratios between scores, so a score that is well above zero for most inputs with seen labels remains so after recalibration.

Both recalibrations use only the training set, so the coverage guarantee of Algorithm~\ref{alg:open-set-plug-in} is unaffected. In our experiments (Appendix~\ref{app:openset-findings}), isotonic recalibration leads to the best results among direct conformalizations, but their decoded prediction sets remain larger than ours, unless new labels are rare.

\clearpage

\section{The Dirichlet Process Model} \label{app:dp-example}

\begin{example}[Dirichlet Process Model]
\label{ex:dp-model}
A classical example of an exchangeable sequence with infinitely many possible labels is given by random sampling from the Dirichlet process \citep{ferguson1973bayesian,bla73}.
Suppose the label space is $\mathcal{Y} = \mathbb{R}$. The Dirichlet process depends on a non-atomic base distribution $P_0$ on $\mathbb{R}$ (e.g., a standard normal) that governs how new labels are drawn, and a concentration parameter $\theta > 0$ that controls how often new labels appear. Given $n$ previous labels $Y_1,\dots,Y_n$, the next label $Y_{n+1}$ is drawn from $P_0$ with probability $\theta/(\theta+n)$, or set equal to an existing label $y \in \{Y_1,\dots,Y_n\}$ with probability $n_y/(\theta+n)$, where $n_y$ is the number of times $y$ has appeared. Conditional on the labels, features are sampled independently from some distribution $P_{X \mid Y}$; for example, $P_{X \mid Y=y} = \text{Normal}(y,\sigma^2)$, for some $\sigma^2>0$ \citep{lo84}. This construction naturally models the open-set case, since the probability of seeing a new label, $\theta/(\theta+n)$, is always positive.
\end{example}

Although the labels in Example~\ref{ex:dp-model} are real numbers, their numeric values carry no structural meaning for our problem: since $P_0$ is non-atomic, the distinct values are almost surely distinct tags whose pattern of ties is that of the Chinese restaurant process \citep{pitman2006combinatorial}. By the stick-breaking representation \citep{sethuraman1994constructive}, the data-generating model of Section~\ref{sec:synthetic-exp} is, in distribution, a classification model over the dictionary $\mathbb{N}$: class proportions $(\rho_1, \rho_2, \ldots)$ are drawn from the Griffiths--Engen--McCloskey distribution $\mathrm{GEM}(\theta)$, i.e., $\rho_1 = B_1$ and $\rho_j = B_j \prod_{i<j}(1 - B_i)$ for $j \geq 2$, with $B_1, B_2, \ldots$ independent $\mathrm{Beta}(1, \theta)$ random variables; class centroids $V_1, V_2, \ldots$ are drawn independently from $\text{Uniform}(0,1)$; the labels are drawn independently from the categorical distribution with weights $(\rho_j)_{j \geq 1}$; and, given $Y_i = j$, the features are Gaussian with all mean coordinates equal to $V_j$. This is a countably infinite Gaussian-mixture classification model with imbalanced classes and probability $\theta/(\theta+n)$ of a new class at test time, in which the real-valued label merely names class $j$ by its centroid $V_j$. All methods of this paper depend on the labels only through their identities, and every prediction set is a subset of $\mathcal{Y}_n \cup \{\joker\}$, so every result of Section~\ref{sec:synthetic-exp} is invariant to this naming. The real-valued representation is kept because it is classical and convenient for specifying feature distributions.

This is a classification problem, not a regression problem, because the target of inference is the identity of $Y_{n+1}$ as a previously observed class or a new one, while its numeric value is meaningless beyond establishing identity. Under Example~\ref{ex:dp-model}, any interval of positive length contains finitely many observed labels and a continuum of potential new ones, so a method designed for regression would not be able to separate the event that $Y_{n+1}$ equals a specific seen label from the event $N_{n+1}$ that it is new. Intersecting a regression prediction interval with $\mathcal{Y}_n$ forfeits coverage by up to the missing mass, as in Theorem~\ref{theorem:coverage-standard}, while always appending the joker says nothing about novelty.
Including the joker selectively requires a test of the unseen-label event, which is precisely what the conformal Good--Turing $p$-values provide.

An interval-based heuristic, such as declaring novelty when no observed centroid falls within a regression band, may perform well in the low-noise geometry of Section~\ref{sec:synthetic-exp}, but it has no finite-sample guarantee and need not carry over to noisier settings or to feature spaces such as the CelebA embeddings of Section~\ref{sec:real-exp}. Such rules can, however, be useful as novelty scores within the feature-based conformal Good--Turing $p$-values of Section~\ref{sec:pvals-feature-based}.

\clearpage

\section{Review of Generalized Inverse Quantile Scores} \label{app:closed-set-classification}

We review standard closed-set conformal classification using the \emph{generalized inverse quantile} nonconformity scores of~\citet{romano2020classification}, whose randomized variant is used in all numerical experiments of this paper.
These are designed to improve conditional coverage, relative to the simpler probability scores described in Section~\ref{sec:methods-closed-set}, by accounting for heteroscedasticity in the conditional distribution of $Y \mid X$.

\subsection{Deterministic Scores} \label{app:adaptive-scores}
\label{app:closed-set-split}

Let $K := |\mathcal{Y}_n|$. For any $x \in \mathcal{X}$ and $y \in \mathcal{Y}_n$, let $\hat{p}(x, y)$ denote an estimate of $\P{Y = y \mid X = x}$ produced by a classification model, e.g., the output of the final softmax layer of a neural network. For any $x \in \mathcal{X}$, let $\hat{p}_{(1)}(x) \geq \ldots \geq \hat{p}_{(K)}(x)$ denote these estimates in decreasing order, and let $\hat{r}(x, y)$ denote the rank of $\hat{p}(y \mid x)$ among them, with ties broken arbitrarily. The generalized cumulative distribution function is the total probability of the labels ranked at or above $y$:
\begin{align*}
  \hat{\Pi}(x, \hat{p}, y)
  := \sum_{j=1}^{\hat{r}(x, y)} \hat{p}_{(j)}(x).
\end{align*}
In the notation of Section~\ref{sec:methods-closed-set}, the scores are
\begin{align} \label{eq:adaptive-scores}
  & S_i^{(y)} := \hat{\Pi}(X_i,\hat{p}^{(y)},Y_i), \quad i \in [n],
  && S_{n+1}^{(y)} := \hat{\Pi}(X_{n+1},\hat{p}^{(y)},y),
\end{align}
where $\hat{p}^{(y)}$ is obtained from a model trained on the augmented dataset $\mathcal{D}(y)$ in full conformal prediction, and from a model trained once on the training set in split conformal prediction, in which case it does not depend on $y$.

In split conformal prediction, the reference sample $[n]$ is divided into a training set $\mathcal{D}^{\text{train}}$, used to fit $\hat{p}$, and a calibration set $\mathcal{D}^{\text{cal}}$. Since $\hat{p}$ does not depend on $y$, the scores~\eqref{eq:adaptive-scores} reduce to $S_i := \hat{\Pi}(X_i,\hat{p},Y_i)$ for $i \in \mathcal{D}^{\text{cal}}$ and $S_{n+1}^{(y)} := \hat{\Pi}(X_{n+1},\hat{p},y)$, and the threshold in~\eqref{eq:standard-prediction-set} becomes
\begin{align} \label{eq:def-tau-hat}
  \hat{q}_{1-\alpha} := \text{the } \lceil (1-\alpha)(|\mathcal{D}^{\text{cal}}|+1) \rceil\text{-th smallest value among } \{S_i\}_{i \in \mathcal{D}^{\text{cal}}} \cup \{+\infty\},
\end{align}
which also does not depend on $y$. The prediction set~\eqref{eq:standard-prediction-set} is then
\begin{align} \label{eq:pred-sets-aps-split}
  \hat{C}_{\alpha}(X_{n+1}; \mathcal{Y}_n) = \bigl\{ y \in \mathcal{Y}_n : S_{n+1}^{(y)} \leq \hat{q}_{1-\alpha} \bigr\},
\end{align}
i.e., it contains the most likely labels, in decreasing order of $\hat{p}(X_{n+1}, \cdot)$, until their cumulative probability exceeds $\hat{q}_{1-\alpha}$.
Algorithm~\ref{alg:standard-plug-in} summarizes this procedure. In a closed-set problem with a known finite label space $\mathcal{Y}$, the same definitions apply with $\mathcal{Y}_n$ replaced by $\mathcal{Y}$, as in Algorithm~\ref{alg:standard-closed-set}.

\begin{algorithm}[!htb]
  \caption{Split Conformal Classification with Known Label Space}
  \label{alg:standard-closed-set}
  \SetKwInOut{Input}{Input{}}
  \Input{%
    Data $\{(X_i, Y_i)\}_{i=1}^{n}$, split into $\mathcal{D}^{\text{train}}$ and $\mathcal{D}^{\text{cal}}$;\\
    Known finite label space $\mathcal{Y}$; \\
    Test features $X_{n+1}$; \\
    Classification algorithm $\mathcal{A}$; \\
    Miscoverage level $\alpha \in (0,1)$.
  }
  Fit $\mathcal{A}$ on $\{(X_i, Y_i)\}_{i \in \mathcal{D}^{\text{train}}}$, obtaining estimated probabilities $\hat{p}(x, y)$ for all $y \in \mathcal{Y}$\;

  Compute $S_i = \hat{\Pi}(X_i, \hat{p}, Y_i)$ for $i \in \mathcal{D}^{\text{cal}}$, and $S_{n+1}^{(y)} = \hat{\Pi}(X_{n+1}, \hat{p}, y)$ for $y \in \mathcal{Y}$, as in~\eqref{eq:adaptive-scores}\;

  Compute the threshold $\hat{q}_{1-\alpha}$ as in~\eqref{eq:def-tau-hat}\;

  \Return{$\hat{C}_{\alpha}(X_{n+1}; \mathcal{Y}) = \bigl\{ y \in \mathcal{Y} : S_{n+1}^{(y)} \leq \hat{q}_{1-\alpha} \bigr\}$}.
\end{algorithm}

\begin{algorithm}[!htb]
  \caption{Split Conformal Classification with Plug-In Label Space}
  \label{alg:standard-plug-in}
  \SetKwInOut{Input}{Input{}}
  \Input{%
    Data $\{(X_i, Y_i)\}_{i=1}^{n}$;\\
    Test features $X_{n+1}$; \\
    Classification algorithm $\mathcal{A}$; \\
    Miscoverage level $\alpha \in (0,1)$.
  }
  Define the observed labels $\mathcal{Y}_n := \mathrm{unique}(Y_1,\ldots,Y_n)$\;

  Randomly split $[n]$ into disjoint subsets $\mathcal{D}^{\text{train}}$ and $\mathcal{D}^{\text{cal}}$\;

  \Return{$\hat{C}_{\alpha}(X_{n+1}; \mathcal{Y}_n)$, computed by Algorithm~\ref{alg:standard-closed-set} with $\mathcal{Y} = \mathcal{Y}_n$}.
\end{algorithm}

\FloatBarrier

\subsection{Randomized Scores}

The experiments of Section~\ref{sec:empirical} use a randomized variant of the generalized inverse quantile scores \citep{romano2020classification}, which randomizes the inclusion of the boundary label. Randomization yields smaller prediction sets on average, while preserving marginal coverage. Let $U_1, \ldots, U_{n+1}$ be independent uniform random variables on $[0,1]$, independent of the data, and define
\begin{align} \label{eq:giq-set}
  \hat{\Pi}^{\mathrm{rand}}(x, \hat{p}, y, u) := \sum_{j=1}^{\hat{r}(x, y)-1} \hat{p}_{(j)}(x) + (1-u)\, \hat{p}_{(\hat{r}(x,y))}(x).
\end{align}
The randomized scores are $S_i := \hat{\Pi}^{\mathrm{rand}}(X_i, \hat{p}, Y_i, U_i)$ for $i \in \mathcal{D}^{\text{cal}}$ and $S_{n+1}^{(y)} := \hat{\Pi}^{\mathrm{rand}}(X_{n+1}, \hat{p}, y, U_{n+1})$, with the same $U_{n+1}$ for all candidate labels $y$. The prediction set is then computed as in~\eqref{eq:pred-sets-aps-split}, with $\hat{q}_{1-\alpha}$ obtained from these scores as in~\eqref{eq:def-tau-hat}; Algorithms~\ref{alg:standard-plug-in} and~\ref{alg:standard-closed-set} apply unchanged.
These prediction sets may occasionally be empty, which is necessary in general to ensure coverage close to $1-\alpha$ in the closed-set setting; forcing the most likely label into every set may help interpretation while preserving validity, at the cost of making coverage conservative. Our experiments allow empty sets so that coverage can be compared more directly to the nominal level $1-\alpha$.

\clearpage

\section{Additional Details on Conformal Good--Turing Classification} \label{app:plugin-adjustment}

\subsection{Choice of the Exponent $\beta$ for Testing $H_{\text{seen}}$}
\label{app:beta-choice}

This appendix complements Section~\ref{sec:composite}.
The combination $p$-value $\psi_{\text{seen}}$ in~\eqref{eq:combination-pval} depends on the weights $c_k$, which we design to follow the power law~\eqref{eq:combination-pval-power-law} with exponent $\beta$. For any weights satisfying the assumptions of Theorem~\ref{thm:psi-seen-super-uniformity},
\begin{align} \label{eq:psi-seen-oracle-bound}
\psi_{\text{seen}}
= \max_{k \in K_n} \frac{\psi_k}{c_k}
\;\geq\; \frac{\sum_{k \in K_n} c_k \left(\psi_k/c_k\right)}{\sum_{k \in K_n} c_k}
\;\geq\; \sum_{k \in K_n} \psi_k,
\end{align}
since a maximum dominates any weighted average and $\sum_{k \in K_n} c_k \leq 1$. The right-hand side is attained by the infeasible weights $c_k \propto \psi_k$. Hence, no valid choice of the weights makes the test informative when $\sum_{k \in K_n} \psi_k \geq 1$, and this sum is small only when very rare labels dominate the frequency profile. On dense samples, the test of $H_{\text{seen}}$ is therefore powerless for every $\beta$, and the allocation of Section~\ref{sec:alpha-allocation} accordingly sets $\alpha_{\text{seen}} = 0$ outside heavily open-set regimes.

When the test is informative, the numerical study in Appendix~\ref{app:experiments} shows that $\beta = 1.6$ is a sensible default, although the best exponent varies somewhat with the sparsity of the sample. In practice, we therefore select $\beta$ by $5$-fold cross-validation, using an approach similar to that used to select the allocation of $\alpha$ in Section~\ref{sec:alpha-allocation}. For each value of $\beta$ on a grid, we recompute the $p$-values $\psi_k$ on the complement of each fold, average $\min\{\psi_{\text{seen}}(\beta), 1\}$ across folds, and select the value of $\beta$ with the smallest average. Ties occur when the test is uninformative, and are broken in favor of $\beta = 1.6$.
This data-dependent choice of $\beta$ is not covered by Theorem~\ref{thm:psi-seen-super-uniformity}, which assumes fixed weights, but the experiments of Section~\ref{sec:empirical} suggest that it does not compromise validity in practice.

\subsection{Data-Driven Allocation of the Significance Budget} \label{app:alpha-allocation}
\label{app:alpha-allocation-mm}

This appendix complements Section~\ref{sec:alpha-allocation}.
Algorithm~\ref{alg:alpha-tuning-loss-optimization} implements the cross-validated allocation of Section~\ref{sec:alpha-allocation}, both for the method of Section~\ref{sec:cgtc} and for its missing-mass adjustment of Section~\ref{sec:plugin-adjustment}. Each held-out fold $\mathcal{D}_k$ is predicted using the remaining folds $\mathcal{D}_{-k}$ as reference data, which take the place of $Z_{1:n}$, and their labels the place of $\mathcal{Y}_n$, in~\eqref{eq:open-set-predictor} and in the loss~\eqref{eq:alpha_allocation_loss}.
With the missing-mass adjustment, the algorithm selects the base allocation, and each fold is evaluated at the adjusted levels computed from its own reference data $\mathcal{D}_{-k}$, as the method would be deployed there; the selected base allocation is then deployed at the levels~\eqref{eq:plugin-adjustment}, computed from the full reference sample.
The grids used in our experiments are given in Appendix~\ref{app:hyperparameters}.

\begin{algorithm}[!htb]
  \caption{Data-Driven Allocation of the Significance Budget}
  \label{alg:alpha-tuning-loss-optimization}
  \DontPrintSemicolon
  \SetArgSty{textup}
  \SetKwInput{Input}{Input}
  \Input{Data $Z_{1:n}$; budget $\alpha \in (0,1)$; preference parameter $\lambda\in[0,1]$; number of folds $V$; finite grids $\mathcal{G}_{\text{seen}}\subset[0,\alpha)$ and $\mathcal{G}(a)\subset[0,a]$ for $a \in (0, \alpha]$; whether to apply the missing-mass adjustment.}
  Randomly partition $Z_{1:n}$ into folds $\mathcal{D}_1,\dots,\mathcal{D}_V$, and let $\mathcal{D}_{-k}=\bigcup_{j\neq k}\mathcal{D}_j$\;
  \ForEach{$\alpha_{\text{seen}}\in\mathcal{G}_{\text{seen}}$}{
    \ForEach{$\alpha_{\text{unseen}}\in\mathcal{G}(\alpha-\alpha_{\text{seen}})$}{
      $\alpha_{\text{class}}\gets\alpha-\alpha_{\text{unseen}}-\alpha_{\text{seen}}$\;
      \For{$k=1,\dots,V$}{
        \eIf{the missing-mass adjustment is applied}{
          Compute the adjusted levels~\eqref{eq:plugin-adjustment} for the base allocation $(\alpha_{\text{class}},\alpha_{\text{unseen}},\alpha_{\text{seen}})$, with $\widehat{\mu}_n$ computed from $\mathcal{D}_{-k}$\;
        }{
          Use the levels $(\alpha_{\text{class}},\alpha_{\text{unseen}},\alpha_{\text{seen}})$ as they are\;
        }
        Compute $\hat{C}^{*}_{\alpha}(X_i)$ for $(X_i,Y_i)\in\mathcal{D}_k$ by~\eqref{eq:open-set-predictor}, with reference data $\mathcal{D}_{-k}$ and these levels\;
        $\hat{L}_k\gets|\mathcal{D}_k|^{-1}\sum_{(X_i,Y_i)\in\mathcal{D}_k}L\bigl(\hat{C}^{*}_{\alpha}(X_i),Y_i;\lambda\bigr)$, with $L$ as in~\eqref{eq:alpha_allocation_loss}\;
      }
      $\bar{L}(\alpha_{\text{unseen}},\alpha_{\text{seen}})\gets V^{-1}\sum_{k=1}^{V}\hat{L}_k$\;
    }
  }
  $(\alpha^{*}_{\text{unseen}},\alpha^{*}_{\text{seen}})\gets\operatorname{arg\,min}\bar{L}$\;
  \Return{$(\alpha-\alpha^{*}_{\text{unseen}}-\alpha^{*}_{\text{seen}},\,\alpha^{*}_{\text{unseen}},\,\alpha^{*}_{\text{seen}})$, used as is or, with the missing-mass adjustment, as the base allocation in~\eqref{eq:plugin-adjustment}}\;
\end{algorithm}

\FloatBarrier
\clearpage

\section{Extended Asymptotic Power Analysis} \label{app:asymptotics}

This appendix extends the asymptotic analysis of Section~\ref{sec:asymptotics}. There, Theorem~\ref{thm:asymptotic-vanishing} shows under exchangeability alone that our prediction sets approach those of standard closed-set conformal prediction when new labels become vanishingly rare.
With a consistent classifier and the generalized inverse quantile scores of \citet{romano2020classification}, these standard prediction sets in turn approach those of an oracle that outputs the smallest prediction sets with feature-conditional coverage $1-\alpha$.

Here, we study the more challenging case where new labels do not become rare in large samples. Under the population model and consistency conditions described below, we show that our prediction sets, computed with the selective sample splitting of Section~\ref{sec:selective-splitting}, then approach those of a new oracle.
This new oracle extends the conditional-coverage oracle to the open-set setting: it outputs the smallest set of recurring labels with a given feature-conditional coverage, and adds the joker for the inputs where new labels are most likely.
As in Section~\ref{sec:asymptotics}, we focus for simplicity on the base allocation $(\tilde\alpha_{\text{class}}, \tilde\alpha_{\text{unseen}}, \tilde\alpha_{\text{seen}}) = (\alpha-a,a,0)$ for a fixed $a\in(0,\alpha)$, so the predictor takes the form in~\eqref{eq:predictor-zero-seen}.

\subsection{Population Model and Oracle}

This analysis assumes the following working model.
The observations are i.i.d.\ from a distribution $P$, with a finite set of \emph{recurring labels} $\mathcal{A} := \{y:\P{Y=y}>0\}$ and \emph{novelty rate} $q := \P{Y\notin\mathcal{A}} < 1-\alpha$. Labels outside $\mathcal{A}$ almost surely never repeat. For $x \in \mathcal{X}$, define $p(y \mid x) := \P{Y=y\mid X=x}$ for $y \in \mathcal{A}$, the conditional novelty probability $z(x) := \P{Y\notin\mathcal{A}\mid X=x}$, and $m(x) := 1-z(x) = \P{Y\in\mathcal{A}\mid X=x}$.

We now define the oracle targeted by our method. Recall that $a$ is the part of the budget $\alpha$ allocated to the test for new labels. For the oracle, this budget limits the probability of leaving a non-recurring test label uncovered, which can never exceed $q$, so we define the \emph{novelty budget} $\nu := \min\{q,a\}$. The rest of the budget, $\alpha - \nu$, is the probability with which the oracle may miss a recurring test label, so the probability with which it must cover one is the \emph{seen-label coverage target}
\[
\tau_\nu := (1-q) - (\alpha - \nu) = 1-\alpha-q+\nu.
\]

The oracle covers recurring labels with a closed-set prediction set. For each $x$, order the $K := |\mathcal{A}|$ recurring labels as $y_{(1)},\ldots,y_{(K)}$ by decreasing probability $p_{(r)}(x) := p(y_{(r)}\mid x)$, breaking ties arbitrarily, and define, for any $\tau \in (0,1]$, the \emph{closed-set oracle}
\begin{equation} \label{eq:oracle-seen-set-main}
C^{\mathrm{or}}_\tau(x) := \Bigl\{y_{(r)}:\ p_{(r)}(x)>0,\
\sum_{j<r}p_{(j)}(x)+(1-U)\,p_{(r)}(x)<\tau\Bigr\},
\end{equation}
where $U\sim\mathrm{Uniform}(0,1)$ is independent of the data.
In a closed-set problem, where the probabilities $p(y \mid x)$ sum to one over $\mathcal{A}$, this is the oracle of \citet{romano2020classification}: it includes the most likely labels, randomizing the inclusion of the boundary label, and is the smallest randomized prediction set with feature-conditional coverage $\tau$ at every $x$. In our setting, the probabilities $p(y \mid x)$ sum to $m(x) < 1$ over $\mathcal{A}$, and are deliberately not renormalized.

The oracle covers non-recurring labels with the joker, which it omits for the inputs where a non-recurring label is least likely. To determine these inputs, define
\[
\Phi(t) := \E{z(X)\,\I{z(X)\leq t}},
\]
the probability that the label is non-recurring and $z(X) \leq t$. The oracle omits the joker if and only if $\Phi(z(x)) \leq \nu$. This leads to the \emph{open-set oracle}
\begin{equation} \label{eq:open-set-oracle}
C^{\mathrm{or},\star}_{\alpha}(x) := C^{\mathrm{or}}_{\tau_\nu}(x) \cup \{\joker:\Phi(z(x))>\nu\},
\end{equation}
where $\{\joker : \Phi(z(x)) > \nu\}$ equals $\{\joker\}$ if $\Phi(z(x)) > \nu$ and is empty otherwise.

In words, the open-set oracle uses the budget $a$ to split the total miscoverage $\alpha$ into two parts: the probability $\nu$ of leaving a non-recurring test label uncovered, and the probability $\alpha - \nu$ of missing a recurring one. A larger $a$, and hence a larger $\nu$, leads to fewer jokers but larger sets of recurring labels. 
Proposition~\ref{prop:oracle-optimality} below shows that, under the following two conditions, the oracle has exactly these miscoverage probabilities, and is optimal in two senses.
\begin{list}{}{\settowidth{\labelwidth}{(C1)}\setlength{\leftmargin}{\labelwidth}\addtolength{\leftmargin}{\labelsep}}
\item[(C1)] $m(X) \geq \tau_\nu$ almost surely.
\item[(C2)] The distribution of $z(X)$ has no atoms in $(0,1]$.
\end{list}
Condition (C1) requires that, at every input, the label be recurring with conditional probability at least $\tau_\nu$, so that the closed-set component of the oracle can reach its coverage target. This is a fairly strong assumption, which may not always be realistic, but it keeps the definition of the oracle and its analysis simple. Condition (C2), which is milder, ensures that the oracle spends its whole novelty budget $\nu$. Intuitively, it requires that the conditional probability $z(x)$ of a non-recurring label, given the features, not take the same positive value over any region of the feature space with positive probability. This is typically the case when the features are continuous and informative about novelty, but not, for example, when they carry no information about novelty at all, in which case $z(x) = q$ for all $x$.

\begin{proposition} \label{prop:oracle-optimality}
Fix $a \in (0, \alpha)$, and assume the population model above and Conditions (C1) and (C2). Then:
\begin{list}{}{\settowidth{\labelwidth}{(iii)}\setlength{\leftmargin}{\labelwidth}\addtolength{\leftmargin}{\labelsep}}
\item[(i)] The open-set oracle satisfies
\begin{align*}
\P{Y_{n+1} \in \mathcal{A},\ Y_{n+1} \notin C^{\mathrm{or},\star}_{\alpha}(X_{n+1})} & = \alpha - \nu, \\
\P{Y_{n+1} \notin \mathcal{A},\ \joker \notin C^{\mathrm{or},\star}_{\alpha}(X_{n+1})} & = \nu,
\end{align*}
so that its total marginal miscoverage is $\alpha$, and it has conditional coverage
\begin{equation} \label{eq:c-nu}
\P{Y_{n+1} \in C^{\mathrm{or},\star}_{\alpha}(X_{n+1}) \mid X_{n+1} = x} = c_\nu(x) := \tau_\nu+z(x)\,\I{\Phi(z(x))>\nu},
\end{equation}
for almost every $x$, where a non-recurring label is covered if and only if the set contains the joker.
\item[(ii)] For almost every $x$, the closed-set oracle $C^{\mathrm{or}}_{\tau_\nu}(x)$ has the smallest expected cardinality among all randomized subsets $C(x) \subseteq \mathcal{A}$ satisfying $\P{Y_{n+1} \in C(X_{n+1}) \mid X_{n+1} = x} \geq \tau_\nu$.
\item[(iii)] The open-set oracle includes the joker least often among all randomized rules for including it that satisfy $\P{Y_{n+1} \notin \mathcal{A},\ \joker \notin C(X_{n+1})} \leq \nu$.
\end{list}
\end{proposition}

Part (iii) of Proposition~\ref{prop:oracle-optimality} is analogous to the Neyman--Pearson lemma: since including the joker at $x$ covers a non-recurring label only with probability $z(x)$, the joker is used most efficiently where $z(x)$ is largest.

When $q = 0$, our open-set oracle reduces to the closed-set oracle of \citet{romano2020classification}, which has conditional coverage $1-\alpha$ at every input. When $q > 0$, the oracle still has conditional coverage $1-\alpha$ at every input if $a \geq q$, since then $\tau_\nu = 1-\alpha$, so the closed-set component alone covers the test label with conditional probability $1-\alpha$, and the joker is never needed.
If $a < q$, instead, the closed-set component covers the test label with conditional probability only $\tau_\nu < 1-\alpha$, and the joker makes up the difference only on average: at inputs where it is included, the conditional coverage is $\tau_\nu + z(x)$, which may exceed $1-\alpha$, while elsewhere it remains $\tau_\nu$.

\subsection{Regularity Conditions and Implementation}

We assume the following additional conditions, on a fixed set $\mathcal{X}_0$ with $P_X(\mathcal{X}_0)=1$:
\begin{list}{}{\settowidth{\labelwidth}{(C3)}\setlength{\leftmargin}{\labelwidth}\addtolength{\leftmargin}{\labelsep}}
\item[(C3)] For some $\eta>0$, every $x\in\mathcal{X}_0$ satisfies $m(x)\geq\tau_\nu+\eta$, $\min_{y\in\mathcal{A}}p(y \mid x)\geq\eta$, and $|p(y \mid x)-p(y'\mid x)|\geq\eta$ for all distinct $y,y'\in\mathcal{A}$.
\item[(C4)] The classifier, fitted on the training set, returns a probability vector $\hat{p}(x,\cdot)$ on $\mathcal{Y}_n$ and is uniformly consistent:
\begin{equation} \label{eq:consistent-classifier}
\sup_{x\in\mathcal{X}_0}\Bigl[
\max_{y\in\mathcal{A}}|\hat{p}(y \mid x)-p(y \mid x)|
+\max_{y\in\mathcal{Y}_n\setminus\mathcal{A}}\hat{p}(y \mid x)
\Bigr]\to 0 \text{ in probability},
\end{equation}
where a maximum over an empty set is zero.
\item[(C5)] The novelty score function $\hat{s}_0$ of Section~\ref{sec:pvals-feature-based} satisfies, for some continuous strictly increasing function $h$,
\begin{equation} \label{eq:consistent-score}
\sup_{x\in\mathcal{X}_0}|\hat{s}_0(x)-h(m(x))| \to 0 \text{ in probability}.
\end{equation}
\end{list}
These conditions are stronger than necessary, but they keep the result and its proof simple. Condition (C3) strengthens Condition (C1) with a margin $\eta$, and also requires the probabilities of the recurring labels to be bounded away from zero and from each other, so that small estimation errors do not change the order in which labels enter the prediction set.
Condition (C4) requires the classifier to be consistent for the population probabilities $p(y \mid x)$. Under selective splitting, whether an observation is used for training depends on its label, so the conditional distribution of the labels given the features differs between the training set and the population; consistency for $p(y \mid x)$ would therefore require suitably reweighting the training observations. Our experiments do not use reweighting for training, so they are not exactly covered by this result, but the method still performs well in practice.

Finally, Condition (C5) requires $\hat{s}_0$ to rank inputs in the same order as $m(x)$, the probability that the label is recurring. This condition is less well aligned than the others with how we apply our method in practice, because the one-class classifiers we use target a different score.
Let $f_{\mathcal{A}}$ and $f_{\mathcal{N}}$ denote the densities of the features given a recurring label and given a non-recurring label, respectively. By Bayes' rule,
\begin{align*}
m(x)
= \frac{(1-q)\, f_{\mathcal{A}}(x)}{(1-q)\, f_{\mathcal{A}}(x) + q\, f_{\mathcal{N}}(x)}
= \biggl\{1 + \frac{q}{1-q} \cdot \frac{f_{\mathcal{N}}(x)}{f_{\mathcal{A}}(x)}\biggr\}^{-1},
\end{align*}
so $m(x)$ ranks inputs by the ratio $f_{\mathcal{A}}(x)/f_{\mathcal{N}}(x)$. A one-class classifier trained on the features of observations with recurring labels, as proposed in Section~\ref{sec:pvals-feature-based}, typically estimates only $f_{\mathcal{A}}$, so this condition is difficult to verify.
In principle, this ratio could be estimated by training a classifier to distinguish the features of observations with recurring labels from those of observations with labels appearing once, which approximately follow $f_{\mathcal{N}}$. Including the test point among the latter would preserve the invariance required by Theorem~\ref{thm:XGT}, as in a full-conformal version of our approach, but it would require retraining the score for each test point.

\subsection{Main Asymptotic Result}

The result below applies to the following implementation of our method. We test $H_{\text{unseen}}$ with the feature-based $p$-value~\eqref{eq:p-value-XGT}, and construct the closed-set component by selective splitting (Section~\ref{sec:selective-splitting}) with $\pi(1)=0$ and $\pi(k)=\varrho\in(0,1)$ for $k\geq2$, using the weighted prediction set~\eqref{eq:weighted-split} and the randomized generalized inverse quantile scores (Appendix~\ref{app:closed-set-classification}). For the comparison with the oracle, the closed-set oracle~\eqref{eq:oracle-seen-set-main} uses the same randomization variable as the test point, $U=U_{n+1}$.

\begin{theorem} \label{thm:oracle-convergence}
Under the population model, Conditions (C1)--(C5), and the implementation described above, the prediction set~\eqref{eq:predictor-zero-seen} satisfies
\[
\P{\hat C^*_{\alpha,n}(X_{n+1})\neq C^{\mathrm{or},\star}_{\alpha}(X_{n+1})}\to0.
\]
Consequently, $\P{Y_{n+1}\in\hat C^*_{\alpha,n}(X_{n+1})}\to1-\alpha$ and
\[
\E{\left|\P{Y_{n+1}\in\hat C^*_{\alpha,n}(X_{n+1})\mid X_{n+1}}-c_\nu(X_{n+1})\right|}\to0,
\]
where $c_\nu$ is the conditional coverage of the open-set oracle given in~\eqref{eq:c-nu}, and the conditional probability averages over the reference sample and the randomization.
\end{theorem}

In the asymptotically closed-set case with $q = 0 = \nu$, Theorem~\ref{thm:oracle-convergence} strengthens Theorem~\ref{thm:asymptotic-vanishing}. The latter shows, under exchangeability alone, that the joker vanishes and that the coverage of our prediction sets converges to $1-\alpha$. Under the additional conditions above, Theorem~\ref{thm:oracle-convergence} shows that our prediction sets also converge to the closed-set oracle $C^{\mathrm{or}}_{1-\alpha}(x)$, which is the smallest prediction set with feature-conditional coverage $1-\alpha$. In this case, every label is recurring, so, with probability tending to one, there are no singletons, selective splitting reduces to uniform random splitting with proportion $\varrho$, and the weights~\eqref{eq:conformalization-weights} are uniform. Hence, the result also applies to standard split conformal prediction with a uniform random split.

When new labels persist in the large-sample limit ($q > 0$), Theorem~\ref{thm:asymptotic-vanishing} does not apply, and the selective splitting of Section~\ref{sec:selective-splitting} becomes essential. It keeps singletons, which are the observations with non-recurring labels, out of the calibration set, so the calibration scores follow the distribution of the oracle scores of recurring labels. With uniform random splitting, a fraction $q$ of the calibration points would instead have non-recurring labels, which never appear in the training set and therefore receive large nonconformity scores. This would inflate the calibrated threshold, so the seen-label sets would target a coverage larger than $\tau_\nu$, and the method would become asymptotically conservative. Finally, the trade-off between the use of the joker and the number of seen labels, governed by the budget $a$ as described after~\eqref{eq:open-set-oracle}, is balanced in practice by our data-driven allocation, through the loss~\eqref{eq:alpha_allocation_loss}.

\FloatBarrier

\clearpage

\section{Additional Details on Selective Sample Splitting} \label{app:selective-splitting}

This appendix complements Section~\ref{sec:selective-splitting}.

\subsection{Efficient Computation of the Conformalization Weights}
\label{app:computation-shortcut}

The weights $w_j(y)$ in~\eqref{eq:conformalization-weights} can be computed efficiently as follows.
Dividing the numerator and denominator of~\eqref{eq:conformalization-weights} by $p^{(n+1)}(y)$ gives
\begin{align} \label{eq:p-tilde}
w_j(y) = \frac{\tilde{p}^{(j)}(y)}{1 + \sum_{k \in \mathcal{D}^{\text{cal}}} \tilde{p}^{(k)}(y)},
\qquad \text{where} \qquad
\tilde{p}^{(j)}(y) := \frac{p^{(j)}(y)}{p^{(n+1)}(y)},
\end{align}
so it suffices to compute the ratios $\tilde{p}^{(j)}(y)$.
If $Y_j = y$, the swap changes nothing, and $\tilde{p}^{(j)}(y) = 1$.
Otherwise, the swap changes only the frequencies of the labels $Y_j$ and $y$: for every observation $i$ whose label is neither $Y_j$ nor $y$, we have $N(Y_i^{(j,y)}; Y_{1:n}^{(j,y)}) = N(Y_i; Y_{1:n})$, so the factors of these observations are the same in $p^{(j)}(y)$ and $p^{(n+1)}(y)$, and cancel in the ratio.
To compute the remaining factors, write $N_\ell := N(\ell; Y_{1:n})$, $f_{\text{cal}}(\ell) := \sum_{i \in \mathcal{D}^{\text{cal}}} \I{Y_i = \ell}$, and $f_{\text{train}}(\ell) := N_\ell - f_{\text{cal}}(\ell)$ for each $\ell \in \mathcal{Y}_n$. After the swap, label $Y_j$ has frequency $N_{Y_j} - 1$ and one fewer calibration observation, while label $y$ has frequency $N_y + 1$ and one more calibration observation. Therefore,
\begin{align*}
  \tilde{p}^{(j)}(y)
  = \underbrace{\frac{\pi(N_{Y_j}-1)^{f_{\text{cal}}(Y_j)-1}\, \{1-\pi(N_{Y_j}-1)\}^{f_{\text{train}}(Y_j)}}{\pi(N_{Y_j})^{f_{\text{cal}}(Y_j)}\, \{1-\pi(N_{Y_j})\}^{f_{\text{train}}(Y_j)}}}_{\textstyle =: a(Y_j)}
  \cdot
  \underbrace{\frac{\pi(N_y+1)^{f_{\text{cal}}(y)+1}\, \{1-\pi(N_y+1)\}^{f_{\text{train}}(y)}}{\pi(N_y)^{f_{\text{cal}}(y)}\, \{1-\pi(N_y)\}^{f_{\text{train}}(y)}}}_{\textstyle =: b(y)},
\end{align*}
with the convention $0^0 = 1$. The factor $a(Y_j)$ accounts for label $Y_j$ losing one calibration observation, and $b(y)$ for label $y$ gaining one.

We now compute the denominator of $w_j(y)$. Each of the $f_{\text{cal}}(y)$ calibration points with label $y$ contributes $\tilde{p}^{(k)}(y) = 1$, and every other calibration point $k$ contributes $\tilde{p}^{(k)}(y) = a(Y_k)\, b(y)$. Let
\begin{align*}
A := \sum_{\ell \in \mathcal{Y}_n : f_{\text{cal}}(\ell) > 0} f_{\text{cal}}(\ell)\, a(\ell)
\end{align*}
denote the sum of $a(Y_k)$ over all calibration points. With this notation,
\begin{align*}
1 + \sum_{k \in \mathcal{D}^{\text{cal}}} \tilde{p}^{(k)}(y)
& = 1 + \sum_{k \in \mathcal{D}^{\text{cal}} : Y_k = y} 1 + \sum_{k \in \mathcal{D}^{\text{cal}} : Y_k \neq y} a(Y_k)\, b(y) \\
& = 1 + f_{\text{cal}}(y) + b(y) \bigl\{ A - f_{\text{cal}}(y)\, a(y) \bigr\} =: D(y),
\end{align*}
with $f_{\text{cal}}(y)\, a(y) := 0$ when $f_{\text{cal}}(y) = 0$.

Therefore, the weights simplify to
\begin{align*}
w_j(y) =
\begin{cases}
1/D(y), & \text{if } j = n+1 \text{ or } Y_j = y, \\
a(Y_j)\, b(y)/D(y), & \text{otherwise}.
\end{cases}
\end{align*}

The constant $A$ and the quantities $a(y)$, $b(y)$, and $D(y)$, for $y \in \mathcal{Y}_n$, do not depend on the features or on the test point, so they can be computed once in $\mathcal{O}(n + |\mathcal{Y}_n|)$ time and reused for all test points.

Substituting these weights into~\eqref{eq:weighted-split} gives
\begin{align*}
& \hat{C}^{\pi}_{\alpha}(X_{n+1}; \mathcal{Y}_n) \\
& \; = \left\{ y \in \mathcal{Y}_n :
\frac{1}{D(y)}
+ \sum_{j \in \mathcal{D}^{\text{cal}}: Y_j = y} \frac{1}{D(y)}\, \I{S_j \geq S_{n+1}^{(y)}}
+ \sum_{j \in \mathcal{D}^{\text{cal}}: Y_j \neq y} \frac{a(Y_j)\, b(y)}{D(y)}\, \I{S_j \geq S_{n+1}^{(y)}}
> \alpha \right\} \\
& \; = \left\{ y \in \mathcal{Y}_n :
1 + \sum_{j \in \mathcal{D}^{\text{cal}}: Y_j = y} \I{S_j \geq S_{n+1}^{(y)}}
+ b(y) \sum_{j \in \mathcal{D}^{\text{cal}}: Y_j \neq y} a(Y_j)\, \I{S_j \geq S_{n+1}^{(y)}}
> \alpha\, D(y) \right\}.
\end{align*}
In split conformal prediction, the calibration scores $S_j$ do not depend on $y$ or on the test point, so both sums can be evaluated in $\mathcal{O}(\log n)$ time by binary search, after sorting the calibration scores, both overall and within each label, and storing cumulative sums of $a(Y_j)$.
Therefore, after a one-time preprocessing step costing $\mathcal{O}(n \log n)$, the prediction set $\hat{C}^{\pi}_{\alpha}(X_{n+1}; \mathcal{Y}_n)$ costs $\mathcal{O}(|\mathcal{Y}_n| \log n)$ per test point, compared with $\mathcal{O}(|\mathcal{Y}_n|)$ for standard split conformal prediction.

\subsection{Effective Sample Size of the Weights}
\label{app:ess-selective}

For a fixed number of calibration points, unequal weights reduce the effective sample size relative to uniform weights. A common way to quantify this effect, originally from survey sampling and also used in importance sampling, is Kish's effective sample size \citep{kish1965survey, kish1992weighting, kong1992note, owen2013monte}:
\begin{equation} \label{eq:kish-ess}
n_{\mathrm{eff}}(y) := \frac{\bigl(\sum_{j \in \mathcal{D}^{\text{cal}}} w_j(y)\bigr)^2}{\sum_{j \in \mathcal{D}^{\text{cal}}} w_j(y)^2} \leq |\mathcal{D}^{\text{cal}}|,
\end{equation}
with equality if and only if the weights are uniform, by the Cauchy--Schwarz inequality.

To interpret this quantity in our context, consider the coverage conditional on the training and calibration data, $\text{Coverage}(\mathcal{D}^{\text{train}}, \mathcal{D}^{\text{cal}}) := \P{Y_{n+1} \in \hat{C}^{\pi}_{\alpha}(X_{n+1}; \mathcal{Y}_n) \mid \mathcal{D}^{\text{train}}, \mathcal{D}^{\text{cal}}}$, which is a random variable \citep[e.g.,][]{sesia2026elements}. Marginal coverage is the average of $\text{Coverage}(\mathcal{D}^{\text{train}}, \mathcal{D}^{\text{cal}})$ over hypothetical repeated draws of the data, and Theorem~\ref{thm:weighted-split-validity} guarantees that this average is at least $1-\alpha$ for any weights. By contrast, $\text{Coverage}(\mathcal{D}^{\text{train}}, \mathcal{D}^{\text{cal}})$ is the coverage actually achieved, on future test points, by the prediction sets computed from the one data set at hand, and it fluctuates around its average from one data set to another. These fluctuations are smaller for larger calibration sets, and the effective sample size measures how much unequal weights increase them.
Precisely, for fixed weights, rescaled to sum to one, and i.i.d.\ scores with distribution function $F$, the weighted empirical distribution function of the scores has variance $F(t)\{1-F(t)\}/n_{\mathrm{eff}}(y)$ at each $t$, as opposed to $F(t)\{1-F(t)\}/|\mathcal{D}^{\text{cal}}|$ with uniform weights. Although these assumptions do not hold exactly under selective splitting, this suggests that, as a rule of thumb, $\text{Coverage}(\mathcal{D}^{\text{train}}, \mathcal{D}^{\text{cal}})$ varies across data sets approximately like the conditional coverage of an unweighted prediction set with $n_{\mathrm{eff}}(y)$ calibration points.

We now show that, for the selective splitting rule recommended in Section~\ref{sec:selective-splitting}, the effective sample size is close to the number of calibration points, so that unequal weights cost little in practice. Throughout this subsection, $\pi(1) = 0$ and $\pi(k) = \varrho \in (0,1)$ for $k \geq 2$, so every $j \in \mathcal{D}^{\text{cal}}$ has $N_{Y_j} \geq 2$. We call a calibration point $j$ with $N_{Y_j} = 2$ \emph{paired} if the other observation with label $Y_j$ is also in $\mathcal{D}^{\text{cal}}$, and \emph{split} if it is in $\mathcal{D}^{\text{train}}$. Let $R$ be the number of paired calibration points, and recall that $M_2$ is the number of labels appearing exactly twice in $Y_{1:n}$.

\begin{proposition} \label{prop:ess-weights}
Assume $\pi(1) = 0$ and $\pi(k) = \varrho$ for all $k \geq 2$, for some $\varrho \in (0,1)$. Then, for every $y \in \mathcal{Y}_n$ and $j \in \mathcal{D}^{\text{cal}}$ with $Y_j \neq y$, the ratio $\tilde p^{(j)}(y)$ defined in~\eqref{eq:p-tilde} satisfies
\[
\tilde p^{(j)}(y) =
\begin{cases}
0 & \text{if $j$ is paired},\\
1/(1-\varrho) & \text{if $j$ is split and $N_y \geq 2$},\\
1-\varrho & \text{if $N_{Y_j} \geq 3$ and $N_y = 1$},\\
1 & \text{if $j$ is split and $N_y = 1$, or if $N_{Y_j} \geq 3$ and $N_y \geq 2$},
\end{cases}
\]
while $\tilde p^{(j)}(y) = 1$ if $Y_j = y$. Consequently, at most $R$ of the weights $\{w_j(y)\}_{j \in \mathcal{D}^{\text{cal}}}$ vanish, and
\begin{equation} \label{eq:ess-bound}
n_{\mathrm{eff}}(y) \geq \Bigl\{1 - \frac{\varrho^2}{4(1-\varrho)^2}\Bigr\} \bigl(|\mathcal{D}^{\text{cal}}| - R\bigr).
\end{equation}
Moreover,
\begin{equation} \label{eq:ess-bound-expected}
\E{n_{\mathrm{eff}}(y) \mid Y_{1:n}} \geq \Bigl[1 - \frac{\varrho^2}{4(1-\varrho)^2}\Bigr]\, \varrho\, \bigl(n - M_1 - 2\varrho M_2\bigr).
\end{equation}
\end{proposition}
The proof is in Appendix~\ref{app:selective-proof}. The intuition for why paired points receive zero weight for every candidate label other than their own is that swapping one of them with the test point would leave its twin as a singleton in the calibration set, which is impossible under $\pi(1) = 0$.

In practice, the user fixes a target calibration size $n^{\text{cal}}$, and the rule recommended in Section~\ref{sec:selective-splitting} sets $\varrho = n^{\text{cal}}/(n - M_1)$, conditional on $Y_{1:n}$; that is, $\varrho$ is the fraction of observations with non-singleton labels assigned to calibration. Then $\varrho \,(n - M_1) = n^{\text{cal}}$, and since $2M_2 \leq n - M_1$, \eqref{eq:ess-bound-expected} implies
\begin{align*}
\E{n_{\mathrm{eff}}(y) \mid Y_{1:n}} \geq \Bigl[1 - \frac{\varrho^2}{4(1-\varrho)^2}\Bigr] (1 - \varrho)\, n^{\text{cal}}.
\end{align*}
The factor in square brackets accounts for the unequal nonzero weights, and the factor $1 - \varrho$ for the paired points with zero weight.

Table~\ref{tab:ess-selective} reports $\varrho$ and $n_{\mathrm{eff}}(y)$, averaged over $y \in \mathcal{Y}_n$, in the settings of our experiments. In each setting, $n^{\text{cal}}$ is fixed as a proportion of $n$, and $\varrho = n^{\text{cal}}/(n - M_1)$ is computed from the observed labels; since singletons are always used for training, $\varrho$ exceeds the target proportion $n^{\text{cal}}/n$. The loss in effective sample size grows with $\varrho$. For the Dirichlet process with $\theta = 1000$, raising the target proportion from $10\%$ to $50\%$ raises $\varrho$ from $0.15$ to $0.75$, and lowers $n_{\mathrm{eff}}(y)/|\mathcal{D}^{\text{cal}}|$ from $94.5\%$ to $50.3\%$. For CelebA, with a $10\%$ target, $\varrho$ decreases from $0.15$ to $0.11$ as the reference sample grows from $2{,}000$ to $6{,}000$ images and singletons become rarer, while $n_{\mathrm{eff}}(y)/|\mathcal{D}^{\text{cal}}|$ increases from $91.2\%$ to $98.6\%$.
Overall, selective splitting is quite efficient as long as the calibration proportion is not too large: at the $10\%$ target used in our experiments, $n_{\mathrm{eff}}(y)$ is at least $91\%$ of $|\mathcal{D}^{\text{cal}}|$ in every setting.

\begin{table}[!htb]
\centering
\caption{Effective sample size of the selective-splitting weights on the label profiles of the experiments of Section~\ref{sec:empirical}. In each setting, the target calibration size $n^{\text{cal}}$ is fixed as a proportion of the reference sample size $n$, and the selection probability $\varrho = n^{\text{cal}}/(n - M_1)$ is computed from the observed labels, as recommended in Section~\ref{sec:selective-splitting}. The table reports $\varrho$, the realized calibration size $|\mathcal{D}^{\text{cal}}|$, the share of zero weights (the paired points of Proposition~\ref{prop:ess-weights}), and $n_{\mathrm{eff}}/|\mathcal{D}^{\text{cal}}|$ from~\eqref{eq:kish-ess}, averaged over $y \in \mathcal{Y}_n$. All values are averaged over $20$ replicates of the reference sample and the split. Dirichlet process: $n = 2{,}000$; CelebA: $2{,}000$ identities.}
\label{tab:ess-selective}
\small
\setlength{\tabcolsep}{4pt}
\begin{tabular}{llrrrrr}
\toprule
Data & Setting & $n^{\text{cal}}/n$ & $\varrho$ & $|\mathcal{D}^{\text{cal}}|$ & Zero weights & $n_{\mathrm{eff}}/|\mathcal{D}^{\text{cal}}|$ \\
\midrule
Dirichlet process & $\theta = 12$   & 10\% & 0.101 & 199 & 0.0\% & 100.0\% \\
                  & $\theta = 100$  & 10\% & 0.105 & 197 & 0.4\% & 99.5\% \\
                  & $\theta = 300$  & 10\% & 0.115 & 196 & 1.4\% & 98.5\% \\
                  & $\theta = 500$  & 10\% & 0.125 & 202 & 2.5\% & 97.2\% \\
                  & $\theta = 700$  & 10\% & 0.135 & 200 & 3.2\% & 96.4\% \\
                  & $\theta = 1000$ & 10\% & 0.150 & 201 & 5.0\% & 94.5\% \\
\midrule
Dirichlet process & $\theta = 1000$ & 5\%  & 0.075 & 98   & 2.2\%  & 97.7\% \\
                  & $\theta = 1000$ & 10\% & 0.150 & 208  & 4.9\%  & 94.5\% \\
                  & $\theta = 1000$ & 20\% & 0.301 & 394  & 9.9\%  & 87.6\% \\
                  & $\theta = 1000$ & 50\% & 0.753 & 1000 & 24.3\% & 50.3\% \\
\midrule
CelebA & $n = 2{,}000$ & 10\% & 0.149 & 194 & 8.2\% & 91.2\% \\
       & $n = 3{,}000$ & 10\% & 0.124 & 297 & 4.8\% & 94.8\% \\
       & $n = 4{,}000$ & 10\% & 0.114 & 396 & 3.3\% & 96.4\% \\
       & $n = 5{,}000$ & 10\% & 0.108 & 506 & 1.9\% & 97.9\% \\
       & $n = 6{,}000$ & 10\% & 0.106 & 596 & 1.3\% & 98.6\% \\
\bottomrule
\end{tabular}
\end{table}

\FloatBarrier
\clearpage

\section{Mathematical Proofs} \label{app:proofs}

Throughout the proofs, ``on the event $H_k$'' means intersection with that event, so every probability statement involving a random hypothesis is a joint statement of the form~\eqref{eq:pval-validity}. Conditioning on a random element that is not an event, such as the random set $\mathcal{Y}_{n+1}$ or the frequency profile $\mathcal{F}_{n+1}$, means conditioning on the sigma-algebra it generates, so the corresponding statements hold for almost every realized value. The symbol $\{\cdot\}$ denotes a multiset.

\subsection{From Closed-Set Conformal Classification to Open-Set Scenarios}

\begin{lemma} \label{lemma:coverage-standard-oracle}
Assume $\{(X_i,Y_i)\}_{i=1}^{n+1}$ are exchangeable. Let $\hat{C}_{\alpha}(X_{n+1}; \mathcal{Y}_{n+1})$ denote the conformal prediction set computed as in~\eqref{eq:standard-prediction-set}, using a symmetric score function $s$ and replacing the full label dictionary $\mathcal{Y}$ with $\mathcal{Y}_{n+1} := \text{unique}(Y_1,\ldots,Y_{n+1})$.
Assume also that the scores $\{S_1^{(Y_{n+1})},\ldots,S_{n+1}^{(Y_{n+1})}\}$ are almost surely distinct.
Then,
\begin{align*}
  1 - \alpha \leq \P{Y_{n+1} \in \hat{C}_{\alpha}(X_{n+1}; \mathcal{Y}_{n+1})}
  & \leq 1 - \alpha + \frac{1}{n+1}.
\end{align*}
\end{lemma}

\begin{proof}[of Lemma~\ref{lemma:coverage-standard-oracle}]
With $\mathcal{Y}_{n+1}$ replaced by a fixed dictionary $\mathcal{Y}$ containing all labels, this is a standard result \citep[e.g.,][]{angelopoulos2024theoretical}. The same proof carries over with $\mathcal{Y}_{n+1}$ because this set contains $Y_{n+1}$ and is invariant to permutations of the data, so the scores $S_1^{(Y_{n+1})}, \ldots, S_{n+1}^{(Y_{n+1})}$ remain exchangeable conditional on $\mathcal{Y}_{n+1}$.
\end{proof}

\begin{proof}[of Theorem~\ref{theorem:coverage-standard}]
By definition of $\hat{C}_{\alpha}(X_{n+1}; \mathcal{Y}_{n})$,
\begin{align*}
   \P{Y_{n+1} \in \hat{C}_{\alpha}(X_{n+1}; \mathcal{Y}_{n})}
  &  = \P{Y_{n+1} \in \hat{C}_{\alpha}(X_{n+1}; \mathcal{Y}_{n}), Y_{n+1} \in \mathcal{Y}_{n}} \\
  &  = \P{Y_{n+1} \in \hat{C}_{\alpha}(X_{n+1}; \mathcal{Y}_{n+1}), Y_{n+1} \in \mathcal{Y}_{n}} \\
  &  \leq \min\left\{ \P{Y_{n+1} \in \hat{C}_{\alpha}(X_{n+1}; \mathcal{Y}_{n+1})}, \P{Y_{n+1} \in \mathcal{Y}_{n}} \right\} \\
  &  \leq \min\left\{1 - \alpha + \frac{1}{n+1}, 1-\P{N_{n+1}} \right\} \\
  &  = 1 - \alpha - \P{N_{n+1}} + \min\left\{\P{N_{n+1}} + \frac{1}{n+1}, \alpha  \right\},
\end{align*}
where the last inequality follows from Lemma~\ref{lemma:coverage-standard-oracle}.
Moreover,
\begin{align*}
  & \P{Y_{n+1} \notin \hat{C}_{\alpha}(X_{n+1}; \mathcal{Y}_{n})} \\
  & \qquad = \P{Y_{n+1} \notin \hat{C}_{\alpha}(X_{n+1}; \mathcal{Y}_{n}), Y_{n+1} \in \mathcal{Y}_{n}} + \P{Y_{n+1} \notin \hat{C}_{\alpha}(X_{n+1}; \mathcal{Y}_{n}), Y_{n+1} \notin \mathcal{Y}_{n}} \\
  & \qquad = \P{Y_{n+1} \notin \hat{C}_{\alpha}(X_{n+1}; \mathcal{Y}_{n+1}), Y_{n+1} \in \mathcal{Y}_{n}} + \P{Y_{n+1} \notin \mathcal{Y}_{n}} \\
  & \qquad \leq \P{Y_{n+1} \notin \hat{C}_{\alpha}(X_{n+1}; \mathcal{Y}_{n+1})} + \P{Y_{n+1} \notin \mathcal{Y}_{n}} \\
  & \qquad \leq \alpha +  \P{N_{n+1}},
\end{align*}
where the last inequality follows from Lemma~\ref{lemma:coverage-standard-oracle}.
Combining these two results gives
\begin{align*}
  1 - \alpha - \P{N_{n+1}} \leq
  \P{Y_{n+1} \in \hat{C}_{\alpha}(X_{n+1}; \mathcal{Y}_{n})}
  \leq 1 - \alpha + \min\left\{\frac{1}{n+1}, \alpha - \P{N_{n+1}}  \right\}.
\end{align*}
Finally, the first steps of the second chain of inequalities also show that
\[
\P{Y_{n+1} \notin \hat{C}_{\alpha}(X_{n+1}; \mathcal{Y}_{n}), Y_{n+1} \in \mathcal{Y}_{n}} \leq \P{Y_{n+1} \notin \hat{C}_{\alpha}(X_{n+1}; \mathcal{Y}_{n+1})} \leq \alpha,
\]
by Lemma~\ref{lemma:coverage-standard-oracle}, which establishes the final claim of the theorem.
\end{proof}

\clearpage

\subsection{Main Auxiliary Lemma for Conformal Good--Turing $p$-Values}

\begin{lemma}\label{lemma:general-augmented-p-value}
Let $\{Z_i = (X_i, Y_i)\}_{i=1}^{n+1}$ be exchangeable random variables with $X_i \in \mathcal{X}$ and $Y_i \in \mathcal{Y}$, and let $\{\eta_i\}_{i=1}^{n+1}$ be an exchangeable collection of real-valued random variables, independent of $\{Z_i\}_{i=1}^{n+1}$.
Consider a score function $f^{(k)}$ that is invariant to permutations of its second argument, which is a multiset of $n+1$ elements, and that may take one of the following forms:
\begin{itemize}
\item[(i)] Label-only: $f^{(k)}: \mathcal{Y} \times \mathcal{Y}^{n+1} \times \mathbb{R} \to \mathbb{R}$, whose first two arguments involve only labels;
\item[(ii)] Feature-based: $f^{(k)}: (\mathcal{X}\times \mathcal{Y}) \times (\mathcal{X}\times \mathcal{Y})^{n+1} \times \mathbb{R} \to \mathbb{R}$.
\end{itemize}
For any candidate value $y \in \mathcal{Y}$ for $Y_{n+1}$, define the scores
\begin{align*}
F_i^{(k,y)} & := \begin{cases}
f^{(k)}(Y_i; \{Y_1, \ldots, Y_n, y\}, \eta_i), & \text{if label-only}, \\
f^{(k)}(Z_i; \{Z_1, \ldots, Z_n, (X_{n+1}, y)\}, \eta_i), & \text{if feature-based},
\end{cases} \qquad i \in [n], \\
F_{n+1}^{(k,y)} & := \begin{cases}
f^{(k)}(y; \{Y_1, \ldots, Y_n, y\}, \eta_{n+1}), & \text{if label-only}, \\
f^{(k)}((X_{n+1}, y); \{Z_1, \ldots, Z_n, (X_{n+1}, y)\}, \eta_{n+1}), & \text{if feature-based}.
\end{cases}
\end{align*}
Define also, as a function of $y \in \mathcal{Y}$, the augmented conformal $p$-value
\begin{align*}
\widetilde{\psi}_{k}(Z_1, \ldots, Z_n, X_{n+1}, y) := \frac{1}{n+1}\sum_{i=1}^{n+1}\I{F_i^{(k,y)} \geq F_{n+1}^{(k,y)}},
\end{align*}
leaving its dependence on $\{\eta_i\}_{i=1}^{n+1}$ implicit.
Then, for any $u \in (0,1)$,
\begin{align*}
\P{\widetilde{\psi}_{k}(Z_1, \ldots, Z_n, X_{n+1}, Y_{n+1}) \leq u} \leq u.
\end{align*}
Furthermore, let $\psi_k(Z_1, \ldots, Z_n, X_{n+1})$ be any statistic that does not depend on $Y_{n+1}$ and satisfies $\psi_k(Z_1, \ldots, Z_n, X_{n+1}) \geq \widetilde{\psi}_{k}(Z_1, \ldots, Z_n, X_{n+1}, Y_{n+1})$ on the event $H_k$; for example,
\begin{align*}
\psi_k(Z_1, \ldots, Z_n, X_{n+1}) := \max_{y \in \mathcal{Y} : N(y; Y_{1:n}) = k} \widetilde{\psi}_{k}(Z_1, \ldots, Z_n, X_{n+1}, y),
\end{align*}
with the convention that a maximum over an empty set is one. Then $\psi_k$ is a computable conformal $p$-value for the event $H_k$: for any $u \in (0,1)$,
\begin{align*}
\P{\psi_{k}(Z_1,\ldots,Z_n,X_{n+1}) \leq u, H_k} \leq u.
\end{align*}
\end{lemma}

\begin{proof}[of Lemma~\ref{lemma:general-augmented-p-value}]
We first show that the scores $F_1^{(k,Y_{n+1})}, \ldots, F_{n+1}^{(k,Y_{n+1})}$, evaluated at the true label, are exchangeable. Consider the feature-based case; the label-only case is identical. Each score $F_i^{(k,Y_{n+1})} = f^{(k)}(Z_i; \{Z_1, \ldots, Z_{n+1}\}, \eta_i)$, for $i \in [n+1]$, is obtained by applying the same function to the pair $(Z_i, \eta_i)$ and to the multiset $\{Z_1, \ldots, Z_{n+1}\}$, which is invariant to permutations of the data. The pairs $(Z_1, \eta_1), \ldots, (Z_{n+1}, \eta_{n+1})$ are jointly exchangeable, being formed by two independent exchangeable collections, so the scores are exchangeable.
The augmented conformal $p$-value $\widetilde{\psi}_{k}(Z_1, \ldots, Z_n, X_{n+1}, Y_{n+1})$ is the relative rank of $F_{n+1}^{(k,Y_{n+1})}$ among these exchangeable scores, and is therefore super-uniform:
\begin{align*}
\P{\widetilde{\psi}_{k}(Z_1, \ldots, Z_n, X_{n+1}, Y_{n+1}) \leq u} \leq u \quad \text{for all } u \in (0,1).
\end{align*}
For the second part, note that $\psi_k(Z_1, \ldots, Z_n, X_{n+1}) \geq \widetilde{\psi}_{k}(Z_1, \ldots, Z_n, X_{n+1}, Y_{n+1})$ on $H_k$ by assumption. Therefore, for any $u \in (0,1)$,
\begin{align*}
\P{\psi_{k}(Z_1, \ldots, Z_n, X_{n+1}) \leq u, H_k}
&\leq \P{\widetilde{\psi}_{k}(Z_1, \ldots, Z_n, X_{n+1}, Y_{n+1}) \leq u, H_k} \\
&\leq \P{\widetilde{\psi}_{k}(Z_1, \ldots, Z_n, X_{n+1}, Y_{n+1}) \leq u}
\leq u.
\end{align*}
\end{proof}

\clearpage

\subsection{Feature-Blind Conformal Good--Turing $p$-Values}

\begin{lemma} \label{lemma:least-favorable}
Fix $k \in \{0\} \cup \mathbb{N}$ and a sequence of labels $(y_1, \ldots, y_n)$ whose frequency profile $f_n = (m_1, \ldots, m_n)$ is compatible with $H_k$, i.e., $k = 0$ or $m_k \geq 1$, and let $g_k := (k+1)m_{k+1} + k + 1$. Let $y_{n+1} := y^*$, where $y^*$ is a label appearing exactly $k$ times among $y_1, \ldots, y_n$ if $k \geq 1$, and a label not appearing among them if $k = 0$. Let $(Y_1, \ldots, Y_{n+1})$ be a uniform random permutation of $(y_1, \ldots, y_{n+1})$. Then $(Y_1, \ldots, Y_{n+1})$ is exchangeable,
\begin{align*}
\P{H_k} = \frac{g_k}{n+1},
\end{align*}
and, on the event $H_k$, the frequency profile $\mathcal{F}_n$ of $(Y_1, \ldots, Y_n)$ equals $f_n$.
\end{lemma}

\begin{proof}[of Lemma~\ref{lemma:least-favorable}]
A uniform random permutation of a fixed sequence is exchangeable. The event $H_k$ occurs if and only if position $n+1$ receives one of the $(k+1)m_{k+1} + k + 1$ entries whose label has frequency $k+1$ in $(y_1, \ldots, y_{n+1})$, so
\begin{align*}
  \P{H_k} = \frac{(k+1)m_{k+1} + k + 1}{n+1} = \frac{g_k}{n+1}.
\end{align*}

Next, we show that, on $H_k$, the frequency profile $\mathcal{F}_n$ of $(Y_1, \ldots, Y_n)$ equals $f_n$. Indeed, on $H_k$, position $n+1$ contains a label $\ell$ appearing $k+1$ times in $(y_1, \ldots, y_{n+1})$. If $\ell = y^*$, then $(Y_1,\ldots,Y_n)$ is a permutation of $(y_1, \ldots, y_n)$, so its frequency profile is $f_n$. Otherwise, $\ell$ appears $k+1$ times in $(y_1, \ldots, y_{n})$, while by construction $y^*$ appears $k$ times. In $(Y_1, \ldots, Y_n)$, $\ell$ then appears $k$ times and $y^*$ appears $k+1$ times: one label moves from frequency $k+1$ to $k$ and another from $k$ to $k+1$, so the counts $m_k$ and $m_{k+1}$, and hence the frequency profile, are unchanged.
\end{proof}

\begin{lemma}\label{lemma:GT-lower-bound}
Let $\psi'_k: \mathcal{Y}^n \to \mathbb{R}$ be a function such that $\psi'_k(y_1, \ldots, y_n)$ depends only on the frequency profile of $(y_1, \ldots, y_n)$.
If $\P{\psi'_k(Y_1, \ldots, Y_n) \leq u, H_k} \leq u$ for all $u \in (0,1)$ and all exchangeable distributions over $(Y_1, \ldots, Y_n, Y_{n+1})$, then, for any sequence $(Y_1, \ldots, Y_n)$ whose frequency profile $\mathcal{F}_n = (M_1, \ldots, M_n)$ is compatible with $H_k$, i.e., $k = 0$ or $M_k \geq 1$,
\begin{align*}
\psi'_k(Y_1, \ldots, Y_n)\geq \frac{(k+1)M_{k+1} + k + 1}{n + 1}.
\end{align*}
\end{lemma}


\begin{proof}[of Lemma~\ref{lemma:GT-lower-bound}]
We argue by contradiction. Suppose that some frequency profile $f_n = (m_1, \ldots, m_n)$ compatible with $H_k$ satisfies $\psi'_k(f_n) < g_k/(n+1)$, where $g_k := (k+1)m_{k+1} + k + 1$, and fix $u \in (0,1)$ such that $\psi'_k(f_n) < u < g_k/(n+1)$, which is possible because $0 < g_k \leq n+1$. Let $(Y_1, \ldots, Y_{n+1})$ be the exchangeable sequence constructed in Lemma~\ref{lemma:least-favorable} from a sequence of labels $(y_1, \ldots, y_n)$ with frequency profile $f_n$. By that lemma, the frequency profile $\mathcal{F}_n$ of $(Y_1, \ldots, Y_n)$ equals $f_n$ on $H_k$, so $\psi'_k(Y_1, \ldots, Y_n) = \psi'_k(f_n) < u$ on $H_k$. Therefore, again by Lemma~\ref{lemma:least-favorable},
\begin{align*}
\P{\psi'_k(Y_1, \ldots, Y_n) \leq u, H_k}
\geq \P{H_k} = \frac{g_k}{n+1} > u,
\end{align*}
contradicting the validity of $\psi'_k$.
\end{proof}

\begin{proof}[of Theorem~\ref{theorem:GT-super-uniform}]
\emph{Validity.} We apply Lemma~\ref{lemma:general-augmented-p-value} in the label-only setting, ignoring the $\eta$ variables, with the score $f^{(k)}(y; A) := \phi_{k+1}(y; A)$, where, for a multiset $A$ of labels, $\phi_{k+1}(y; A) := \I{\sum_{a\in A}\I{a=y} = k+1}$ indicates that $y$ appears exactly $k+1$ times in $A$. It suffices to show that $\psi_{k}^{\text{GT}} = \widetilde{\psi}_{k}$ on $H_k$, where $\widetilde{\psi}_{k} :=  \widetilde{\psi}_{k}(Z_1, \ldots, Z_n, X_{n+1}, Y_{n+1})$ is defined as in Lemma~\ref{lemma:general-augmented-p-value}.
On $H_k$, $Y_{n+1}$ appears $k+1$ times in $\{Y_1, \ldots, Y_{n+1}\}$, so $F_{n+1}^{(k,Y_{n+1})} = 1$, and $\mathbb{I}[F_i^{(k,Y_{n+1})} \geq F_{n+1}^{(k,Y_{n+1})}] = 1$ if and only if $Y_i$ appears exactly $k+1$ times in $\{Y_1, \ldots, Y_{n+1}\}$. This holds for $i = n+1$, for the $k$ indices $i \in [n]$ with $Y_i = Y_{n+1}$, and for the $(k+1)M_{k+1}$ indices in $\mathcal{S}_{k+1} := \{ i \in [n] : N(Y_i; Y_{1:n}) = k+1 \}$. Hence, on $H_k$,
\begin{align*}
\widetilde{\psi}_{k} = \frac{1}{n+1}\sum_{i=1}^{n+1}\I{F_i^{(k,Y_{n+1})} \geq F_{n+1}^{(k,Y_{n+1})}} = \frac{(k+1)M_{k+1} + k + 1}{n+1} = \psi_{k}^{\text{GT}},
\end{align*}
and Lemma~\ref{lemma:general-augmented-p-value} gives $\P{\psi_{k}^{\text{GT}} \leq u, H_k} \leq u$ for all $u \in (0,1)$.

\emph{Optimality.} Let $\psi'_k$ be any deterministic function of $\mathcal{F}_n$ satisfying~\eqref{eq:pval-validity} under every exchangeable distribution of $Y_{1:(n+1)}$. By Lemma~\ref{lemma:GT-lower-bound}, $\psi'_k(\mathcal{F}_n) \geq \psi_k^{\text{GT}}(\mathcal{F}_n)$ for every frequency profile compatible with $H_k$, which proves the claim.
\end{proof}

\begin{lemma}
\label{lemma:RGT-grid-optimality}
Fix $k \in \{0\} \cup \mathbb{N}$, and let $\psi'_k(Y_1,\ldots,Y_n;U)$ depend on $(Y_1,\ldots,Y_n)$ only through its frequency profile $\mathcal F_n$, where $U\sim \mathrm{Uniform}(0,1)$ is independent of $(Y_1,\ldots,Y_{n+1})$. If $\P{\psi'_k(Y_1,\ldots,Y_n;U)\le u,\ H_k}\le u$ for every $u\in(0,1)$ and every exchangeable distribution, then, for every frequency profile $f_n = (m_1, \ldots, m_n)$ compatible with $H_k$ and every $j\in\{1,\ldots,g_k\}$, where $g_k := (k+1)m_{k+1}+k+1$,
\[
\P{\psi'_k(\mathcal F_n;U)\le \frac{j}{n+1} \;\Big|\; \mathcal{F}_n = f_n} \le \frac{j}{g_k}.
\]
The randomized $p$-value $\psi_k^{\mathrm{RGT}} = \lceil U g_k \rceil/(n+1)$ attains this bound for every such $f_n$ and $j$.
\end{lemma}

\begin{proof}[of Lemma~\ref{lemma:RGT-grid-optimality}]
Fix a sequence of labels $(y_1, \ldots, y_n)$ whose frequency profile $f_n = (m_1, \ldots, m_n)$ is compatible with $H_k$, let $g_k := (k+1)m_{k+1} + k + 1$, and let $(Y_1, \ldots, Y_{n+1})$ be the exchangeable sequence constructed from it in Lemma~\ref{lemma:least-favorable}, with $U$ independent of it. By that lemma, the frequency profile $\mathcal{F}_n$ of $(Y_1, \ldots, Y_n)$ equals $f_n$ on $H_k$, so $\psi'_k(Y_1, \ldots, Y_n; U) = \psi'_k(f_n; U)$ on $H_k$, which depends only on $U$ and is therefore independent of $H_k$. Hence, using $\P{H_k} = g_k/(n+1)$ from Lemma~\ref{lemma:least-favorable}, for any $u \in (0,1)$,
\begin{align*}
\P{\psi'_k(Y_1,\ldots,Y_n;U)\le u,\ H_k}
= \frac{g_k}{n+1}\, \P{\psi'_k(f_n;U)\le u},
\end{align*}
and validity gives $\P{\psi'_k(f_n;U)\le u}\le (n+1)u/g_k$. Taking $u=j/(n+1)$ for $j < g_k$, which satisfies $u < 1$ since $g_k \leq n+1$, yields
\begin{align*}
\P{\psi'_k(f_n;U)\le \frac{j}{n+1}} \leq \frac{j}{g_k};
\end{align*}
for $j = g_k$, the bound equals one and holds trivially. Since $U$ is independent of the labels, this implies the same bound on the conditional probability given $\mathcal{F}_n = f_n$, as desired. Finally, $\P{\lceil U g_k \rceil \le j}=j/g_k$, so $\psi_k^{\mathrm{RGT}}$ attains the bound.
\end{proof}

\begin{proof}[of Theorem~\ref{theorem:RGT-super-uniform}]
\emph{Validity.} Let $\eta_1,\ldots,\eta_{n+1}$ be i.i.d.\ $\mathrm{Uniform}(0,1)$ variables independent of $(Y_1,\ldots,Y_{n+1})$, and apply the first part of Lemma~\ref{lemma:general-augmented-p-value} to the randomized label-only score $f^{(k)}(y;A,\eta) := \phi_{k+1}(y;A)\,\eta$, with $\phi_{k+1}$ as in the proof of Theorem~\ref{theorem:GT-super-uniform}; the resulting augmented conformal $p$-value $\widetilde\psi_k$ satisfies $\P{\widetilde\psi_k\le u,\ H_k} \leq \P{\widetilde\psi_k\le u} \le u$ for every $u\in(0,1)$. On $H_k$, $F_{n+1}^{(k,Y_{n+1})}=\eta_{n+1}$, the $G_k$ indices identified in the proof of Theorem~\ref{theorem:GT-super-uniform} have scores $\eta_i$, and every other index has score $0<\eta_{n+1}$. Conditionally on $(Y_1,\ldots,Y_{n+1})$ and on $H_k$, the variables $\eta_i$ of these $G_k$ indices are i.i.d.\ and continuous, so $(n+1)\widetilde\psi_k$, the rank of $\eta_{n+1}$ among them counted from the top, is uniform on $\{1,\ldots,G_k\}$. This is also the conditional law of $(n+1)\psi_k^{\mathrm{RGT}} = \lceil U G_k \rceil$ given $(Y_1,\ldots,Y_{n+1})$, because $U$ is independent of the labels. Since $H_k$ is determined by the labels, taking expectations over $H_k$ gives $\P{\psi_k^{\mathrm{RGT}}\le u,\ H_k}=\P{\widetilde\psi_k\le u,\ H_k}\le u$.

\emph{Optimality.} This is Lemma~\ref{lemma:RGT-grid-optimality}.
\end{proof}

\clearpage

\subsection{Feature-Based Conformal Good--Turing $p$-Values} \label{app:proof-xgt}

\begin{proof}[of Theorem~\ref{thm:XGT}]
Since $\psi^{\text{XGT}}_k$ depends on the scores only through comparisons among them, we may assume without loss of generality that $\hat{s}_k > 0$, replacing $\hat{s}_k$ by $\exp(\hat{s}_k)$ if needed.
By assumption, on $H_k$, $\hat{s}_k$ is a function of the multiset $\{Z_1, \ldots, Z_{n+1}\}$; we write $\hat{s}_k(x; \mathcal{D})$ for this function evaluated at a multiset $\mathcal{D}$, so that $\hat{s}_k(x) = \hat{s}_k(x; \{Z_1, \ldots, Z_{n+1}\})$ on $H_k$. We apply Lemma~\ref{lemma:general-augmented-p-value} with the feature-based score
\begin{align*}
f^{(k)}((x,y); \mathcal{D}) = \phi_{k+1}(y; \{Y_j : Z_j \in \mathcal{D}\}) \cdot \hat{s}_k(x; \mathcal{D}),
\end{align*}
where $\phi_{k+1}$ is the frequency indicator $\phi_{k+1}(y; A) := \I{\sum_{a\in A}\I{a=y} = k+1}$ for a multiset $A$; this score is symmetric in $\mathcal{D}$.

On $H_k$, $F_{n+1}^{(k,Y_{n+1})} = \hat{s}_k(X_{n+1})$, while $F_i^{(k,Y_{n+1})} = \phi_{k+1}(Y_i; \{Y_1, \ldots, Y_{n+1}\})\, \hat{s}_k(X_i)$ for $i \in [n]$. Since $\hat{s}_k > 0$, $F_i^{(k,Y_{n+1})} \geq F_{n+1}^{(k,Y_{n+1})}$ requires $\phi_{k+1}(Y_i; \{Y_1, \ldots, Y_{n+1}\}) = 1$, which on $H_k$ holds exactly for $i = n+1$, for the indices in $\mathcal{S}_{k+1}$, and for the $k$ indices $i \in [n]$ with $Y_i = Y_{n+1}$, and then $\hat{s}_k(X_i) \geq \hat{s}_k(X_{n+1})$. Therefore, on $H_k$,
\begin{align*}
\widetilde{\psi}_{k}
  & = \widetilde{\psi}_{k}(Z_1, \ldots, Z_n, X_{n+1}, Y_{n+1})
   = \frac{1}{n+1}\sum_{i=1}^{n+1}\I{F_i^{(k,Y_{n+1})} \geq F_{n+1}^{(k,Y_{n+1})}} \\
  & = \frac{1}{n+1}\sum_{i=1}^{n+1}\I{F_i^{(k,Y_{n+1})} \geq F_{n+1}^{(k,Y_{n+1})}} \I{\phi_{k+1}(Y_i; \{Y_1, \ldots, Y_{n+1}\}) = 1} \\
  & = \frac{1}{n+1}\sum_{i=1}^{n+1}\I{\hat{s}_k(X_i) \geq \hat{s}_k(X_{n+1})} \I{\phi_{k+1}(Y_i; \{Y_1, \ldots, Y_{n+1}\}) = 1} \\
  & = \frac{1}{n+1}\left(1 + \sum_{i \in \mathcal{S}_{k+1}} \I{\hat{s}_k(X_i) \geq \hat{s}_k(X_{n+1})} + \sum_{j \in [n]: Y_j = Y_{n+1}} \I{\hat{s}_k(X_j) \geq \hat{s}_k(X_{n+1})}\right) \\
  & \leq \frac{1}{n+1}\left(1 + \sum_{i \in \mathcal{S}_{k+1}} \I{\hat{s}_k(X_i) \geq \hat{s}_k(X_{n+1})} + \max_{i \in \mathcal{S}_k} \sum_{j \in [n]: Y_j = Y_i} \I{\hat{s}_k(X_j) \geq \hat{s}_k(X_{n+1})}\right) \\
  & = \psi^{\text{XGT}}_{k}.
\end{align*}
Thus, $\psi^{\text{XGT}}_{k}$, which does not depend on $Y_{n+1}$, satisfies $\psi^{\text{XGT}}_{k} \geq \widetilde{\psi}_{k}$ on $H_k$, so the result follows from the second part of Lemma~\ref{lemma:general-augmented-p-value}.
\end{proof}

\begin{proposition} \label{prop:unseen-error-exact}
Assume $\{(X_i,Y_i)\}_{i=1}^{n+1}$ are exchangeable and fix $u \in (0,1)$. Let $S := M_1(Y_{1:(n+1)})$ be the number of singletons in the augmented sample, and $v := \lfloor u(n+1) \rfloor$. Then
\begin{align*}
\P{\psi^{\text{GT}}_0 \leq u,\, H_0} = \frac{\E{S\, \I{S \leq v}}}{n+1},
\qquad
\P{\psi^{\text{RGT}}_0 \leq u,\, H_0} = \frac{\E{\min\{S, v\}}}{n+1}.
\end{align*}
If, in addition, $\psi^{\text{XGT}}_0$ is computed with a score $\hat{s}_0$ as in Theorem~\ref{thm:XGT} such that $\hat{s}_0(X_1), \ldots, \hat{s}_0(X_{n+1})$ are almost surely distinct, then
\begin{align*}
\P{\psi^{\text{XGT}}_0 \leq u,\, H_0} = \P{\psi^{\text{RGT}}_0 \leq u,\, H_0} = \frac{\E{\min\{S, v\}}}{n+1}.
\end{align*}
Moreover, for $\psi_0 \in \{\psi^{\text{RGT}}_0, \psi^{\text{XGT}}_0\}$, under the same conditions,
\begin{equation} \label{eq:unseen-two-sided}
\Bigl(u - \frac{1}{n+1}\Bigr) \bigl\{1 - \P{M_1(Y_{1:n}) \leq (n+1)u}\bigr\}
\leq
\P{\psi_0 \leq u,\, H_0}
\leq
\min\bigl\{u,\, \mu_n\bigr\},
\end{equation}
where $\mu_n := \P{N_{n+1}} = \P{H_0}$ is the probability that the test label is new.
\end{proposition}

\begin{remark}
Proposition~\ref{prop:unseen-error-exact} explains why $\psi^{\text{GT}}_0$ is more conservative than $\psi^{\text{RGT}}_0$ and $\psi^{\text{XGT}}_0$. If the number $S$ of singletons in the augmented sample exceeds $v$ almost surely, the type-I error of $\psi^{\text{GT}}_0$ is zero, whereas that of the other two equals $v/(n+1) \geq u - 1/(n+1)$: the deterministic test spends none of its error budget in the regime where the others spend nearly all of it.
\end{remark}

\begin{proof}[of Proposition~\ref{prop:unseen-error-exact}]
Let $\mathcal{G}$ be the sigma-algebra generated by the unordered sample $\{(X_i,Y_i)\}_{i=1}^{n+1}$, with respect to which $S$ is measurable. By exchangeability, conditionally on $\mathcal{G}$, the data are a uniform random ordering of the unordered sample. The event $H_0$ occurs if and only if the test observation is one of the $S$ observations whose label is a singleton in $Y_{1:(n+1)}$, so $\P{H_0 \mid \mathcal{G}} = S/(n+1)$. Moreover, conditionally on $\mathcal{G}$ and $H_0$, the test observation is uniformly distributed over these $S$ observations, the other $S-1$ of which are the singletons of $Y_{1:n}$.

\emph{GT.} On $H_0$, $\psi^{\text{GT}}_0 = (M_1 + 1)/(n+1) = S/(n+1)$, which is at most $u$ if and only if $S \leq v$.
Hence,
\begin{align*}
\P{\psi^{\text{GT}}_0 \leq u,\, H_0}
& = \E{\P{\psi^{\text{GT}}_0 \leq u,\, H_0 \mid \mathcal{G}}} \\
& = \E{\P{S \leq v,\, H_0 \mid \mathcal{G}}} \\
& = \E{\I{S \leq v}\, \P{H_0 \mid \mathcal{G}}} \\
& = \frac{\E{S\, \I{S \leq v}}}{n+1}.
\end{align*}

\emph{RGT and XGT.} On $H_0$, $(n+1)\,\psi^{\text{RGT}}_0 = \lceil U S \rceil$ is uniform on $\{1, \ldots, S\}$, by~\eqref{eq:p-value-RGT}. For XGT, the score $\hat{s}_0$ is, on $H_0$, a function of the multiset $\{Z_1, \ldots, Z_{n+1}\}$, hence determined by $\mathcal{G}$, and $(n+1)\,\psi^{\text{XGT}}_0 = 1 + \sum_{i \in \mathcal{S}_1} \I{\hat{s}_0(X_i) \geq \hat{s}_0(X_{n+1})}$ is the rank, counted from the top, of $\hat{s}_0(X_{n+1})$ among the scores of the $S$ singleton observations of $Y_{1:(n+1)}$. When the scores are distinct, this rank is uniform on $\{1, \ldots, S\}$ conditionally on $\mathcal{G}$ and $H_0$.
In both cases, $(n+1)\,\psi_0$ is an integer uniformly distributed on $\{1, \ldots, S\}$ conditionally on $\mathcal{G}$ and $H_0$, so $\psi_0 \leq u$ if and only if $(n+1)\,\psi_0 \leq v$, and
\begin{align*}
\P{\psi_0 \leq u,\, H_0}
& = \E{\P{\psi_0 \leq u \mid \mathcal{G},\, H_0}\, \P{H_0 \mid \mathcal{G}}} \\
& = \E{\P{(n+1)\,\psi_0 \leq v \mid \mathcal{G},\, H_0}\, \P{H_0 \mid \mathcal{G}}} \\
& = \E{\frac{\min\{S, v\}}{S} \cdot \frac{S}{n+1}} \\
& = \frac{\E{\min\{S, v\}}}{n+1},
\end{align*}
for $\psi_0 \in \{\psi^{\text{RGT}}_0, \psi^{\text{XGT}}_0\}$.

\emph{Bounds.} The upper bound in~\eqref{eq:unseen-two-sided} follows from the validity of $\psi_0$ and from $\P{\psi_0 \leq u,\, H_0} \leq \P{H_0} = \mu_n$. For the lower bound, note that adding one observation decreases the number of singletons by at most one, so $M_1(Y_{1:n}) \leq S + 1$, and hence $M_1(Y_{1:n}) > v$ implies $S \geq v$.
Therefore,
\begin{align*}
\P{\psi_0 \leq u,\, H_0}
& = \frac{\E{\min\{S, v\}}}{n+1} \\
& \geq \frac{v}{n+1}\, \P{S \geq v}
  && \text{(since $\min\{S, v\} \geq v\, \I{S \geq v}$)} \\
& \geq \frac{v}{n+1}\, \P{M_1(Y_{1:n}) > v}
  && \text{(since $M_1(Y_{1:n}) > v$ implies $S \geq v$)} \\
& \geq \frac{v}{n+1}\, \P{M_1(Y_{1:n}) > (n+1)u}
  && \text{(since $v \leq (n+1)u$)} \\
& = \frac{v}{n+1}\, \left( 1 - \P{M_1(Y_{1:n}) \leq (n+1)u} \right) \\
& \geq \Bigl(u - \frac{1}{n+1}\Bigr) \bigl\{1 - \P{M_1(Y_{1:n}) \leq (n+1)u}\bigr\}
  && \text{(since $v \geq (n+1)u - 1$)}.
\end{align*}

\end{proof}

\subsection{Testing the Composite Hypothesis that $Y_{n+1}$ is a Previously Seen Label}

\begin{proof}[of Theorem~\ref{thm:psi-seen-super-uniformity}]
Recall from~\eqref{eq:seen-hypothesis} that $H_{\text{seen}}=\bigcup_{k\in K_n} H_k$, with $K_n=\{k\in[n]: M_k>0\}$. Then,
\begin{align*}
\P{\psi_{\text{seen}}\le u,\; H_{\text{seen}}}
& \le \sum_{k=1}^n \P{\psi_{\text{seen}}\le u,\; H_k}  \\
& =  \sum_{k=1}^n \P{\max_{j \in K_n} \frac{\psi_j}{c_j} \leq u,\; H_k} && \text{(definition of $\psi_{\text{seen}}$ in~\eqref{eq:combination-pval})} \\
&\leq  \sum_{k=1}^n \P{\psi_k \leq c_k u,\; H_k} && \text{(on $H_k$, $M_k \geq 1$, so $k \in K_n$)} \\
& \leq  \sum_{k=1}^n c_k u  \leq u, && \text{(validity of $\psi_k$, and $\textstyle\sum_{k} c_k \leq 1$)}
\end{align*}
where the first inequality uses that the events $H_1, \ldots, H_n$ are disjoint.
\end{proof}

\clearpage

\subsection{Conformal Good--Turing Classification} \label{app:proof-cgtc}

\begin{proof}[of Theorem~\ref{theorem:coverage-cgtc}]
Recall the unseen-label event $N_{n+1} = \{Y_{n+1} \notin \mathcal{Y}_n\}$ of~\eqref{eq:event-new-label}, on which $H_{\text{unseen}}$ is true, and its complement $N_{n+1}^c$, on which $H_{\text{seen}}$ is true.
The miscoverage probability decomposes as
\begin{align*}
  \P{Y_{n+1} \notin \hat{C}_{\alpha}^*(X_{n+1})}
  = \P{Y_{n+1} \notin \hat{C}_{\alpha}^*(X_{n+1}),\, N_{n+1}^c}
  + \P{Y_{n+1} \notin \hat{C}_{\alpha}^*(X_{n+1}),\, N_{n+1}},
\end{align*}
and we bound the two terms separately.

Consider first the event $N_{n+1}^c$, on which $Y_{n+1} \in \mathcal{Y}_n$. By~\eqref{eq:open-set-predictor}, on the event $\{\psi_{\text{seen}} > \alpha_{\text{seen}}\}$ the second case cannot occur, so $\hat{C}_{\alpha_{\text{class}}}(X_{n+1}; \mathcal{Y}_n) \subseteq \hat{C}_{\alpha}^*(X_{n+1})$. Therefore,
\begin{align*}
  \P{Y_{n+1} \notin \hat{C}_{\alpha}^*(X_{n+1}),\, N_{n+1}^c}
  & = \P{Y_{n+1} \notin \hat{C}_{\alpha}^*(X_{n+1}),\, \psi_{\text{seen}} > \alpha_{\text{seen}},\, N_{n+1}^c} \\
  & \qquad + \P{Y_{n+1} \notin \hat{C}_{\alpha}^*(X_{n+1}),\, \psi_{\text{seen}} \leq \alpha_{\text{seen}},\, N_{n+1}^c} \\
  & \leq \P{Y_{n+1} \notin \hat{C}_{\alpha_{\text{class}}}(X_{n+1}; \mathcal{Y}_n),\, Y_{n+1} \in \mathcal{Y}_n}
  + \P{\psi_{\text{seen}} \leq \alpha_{\text{seen}},\, H_{\text{seen}}} \\
  & \leq \alpha_{\text{class}} + \alpha_{\text{seen}},
\end{align*}
by the validity assumptions on the closed-set component and on $\psi_{\text{seen}}$.

Consider next the event $N_{n+1}$, on which $Y_{n+1}$ is covered if and only if $\joker \in \hat{C}_{\alpha}^*(X_{n+1})$. By~\eqref{eq:open-set-predictor}, the joker is excluded only in the first case, i.e., $\{\joker \notin \hat{C}_{\alpha}^*(X_{n+1})\} = \{\psi_{\text{unseen}} \leq \alpha_{\text{unseen}},\, \psi_{\text{seen}} > \alpha_{\text{seen}}\}$. Therefore,
\begin{align*}
  \P{Y_{n+1} \notin \hat{C}_{\alpha}^*(X_{n+1}),\, N_{n+1}}
  \leq \P{\psi_{\text{unseen}} \leq \alpha_{\text{unseen}},\, H_{\text{unseen}}}
  \leq \alpha_{\text{unseen}},
\end{align*}
by the validity assumption on $\psi_{\text{unseen}}$. Combining the two bounds yields $\P{Y_{n+1} \notin \hat{C}_{\alpha}^*(X_{n+1})} \leq \alpha_{\text{class}} + \alpha_{\text{seen}} + \alpha_{\text{unseen}} = \alpha$.
\end{proof}

\begin{proof}[of Proposition~\ref{prop:seen-coverage}]
Let $E := \{Y_{n+1} \in \mathcal{Y}_n\} = H_{\text{seen}}$ and $\hat{C}_{\mathrm{class}} := \hat{C}_{\alpha_{\mathrm{class}}}(X_{n+1}; \mathcal{Y}_n)$.

For the lower bound, observe from~\eqref{eq:open-set-predictor} that, on $E$, the label $Y_{n+1}$ can be missed only if $Y_{n+1} \notin \hat{C}_{\mathrm{class}}$ or if $\hat{C}^*_{\alpha}(X_{n+1}) = \{\joker\}$, which requires $\psi_{\mathrm{seen}} \leq \alpha_{\mathrm{seen}}$. Therefore,
\begin{align*}
\P{Y_{n+1} \notin \hat{C}^*_{\alpha}(X_{n+1}), E} \leq \P{Y_{n+1} \notin \hat{C}_{\mathrm{class}}, E} + \P{\psi_{\mathrm{seen}} \leq \alpha_{\mathrm{seen}}, H_{\mathrm{seen}}} \leq \alpha_{\mathrm{class}} + \alpha_{\mathrm{seen}},
\end{align*}
and the lower bound follows since $\P{E} = 1 - \P{N_{n+1}}$.

For the upper bound, note from~\eqref{eq:open-set-predictor} that $\hat{C}^*_{\alpha}(X_{n+1}) \cap \mathcal{Y}_n \subseteq \hat{C}_{\mathrm{class}}$ in all three cases. Moreover, on $E$, $\mathcal{Y}_{n+1} = \mathcal{Y}_n$, so $\hat{C}_{\mathrm{class}} = \hat{C}_{\alpha_{\mathrm{class}}}(X_{n+1}; \mathcal{Y}_{n+1})$. Therefore,
\begin{align*}
\P{Y_{n+1} \in \hat{C}^*_{\alpha}(X_{n+1}), E}
& \leq \P{Y_{n+1} \in \hat{C}_{\alpha_{\mathrm{class}}}(X_{n+1}; \mathcal{Y}_{n+1}), E} \\
& \leq \min\left\{1 - \alpha_{\mathrm{class}} + \frac{1}{n+1}, 1 - \P{N_{n+1}}\right\},
\end{align*}
by Lemma~\ref{lemma:coverage-standard-oracle} at level $\alpha_{\mathrm{class}}$. Rewriting the minimum as in the proof of Theorem~\ref{theorem:coverage-standard} completes the proof.
\end{proof}

\begin{proof}[of Proposition~\ref{prop:recovery}]
Since $|\mathcal{Y}| = K$, the number $M_1$ of labels appearing exactly once in $Y_1, \ldots, Y_n$ is at most $K$, so $\psi_{0}^{\text{GT}} = (M_{1} + 1)/(n + 1) \leq (K + 1)/(n + 1)$, and $\psi_{0}^{\text{RGT}}, \psi_{0}^{\text{XGT}} \leq \psi_{0}^{\text{GT}}$ by construction. Thus, $\psi_{\text{unseen}} \leq \alpha_{\text{unseen}}$ almost surely for any of the three $p$-values, while $\alpha_{\text{seen}} = 0 < \psi_{\text{seen}}$. The prediction rule~\eqref{eq:open-set-predictor} therefore always selects its first case, so $\hat{C}^{*}_{\alpha}(X_{n+1}) = \hat{C}_{\alpha_{\mathrm{class}}}(X_{n+1};\mathcal{Y}_n)$ almost surely.
\end{proof}

\clearpage

\subsection{Tightening Coverage via Missing-Mass Estimation}

\begin{assumption}[Asymptotic seen-label calibration] \label{assumption:seen-calibration}
For every fixed $\alpha_{\text{class}} \in (0,1)$,
\begin{equation} \label{eq:seen-calibration-asymptotic}
\limsup_{n\to\infty} \Bigl\{ \P{Y_{n+1}\notin \hat{C}_{\alpha_{\text{class}}}(X_{n+1};\mathcal{Y}_n),\, Y_{n+1}\in\mathcal{Y}_n} - \alpha_{\text{class}}\, (1 - \mu_n) \Bigr\} \leq 0.
\end{equation}
\end{assumption}

\begin{lemma} \label{lemma:seen-calibration}
Assume $\{(X_i,Y_i)\}_{i=1}^{n+1}$ are exchangeable, and let the closed-set component be either the conformal prediction set~\eqref{eq:standard-prediction-set} or the weighted prediction set~\eqref{eq:weighted-split}, computed with a fixed $\pi$. Then Assumption~\ref{assumption:seen-calibration} holds in each of the following cases.
\begin{enumerate}
\item[(i)] Suppose assumption~\ref{assumption:novel-miscoverage} holds. In this case, for every fixed $\alpha_{\text{class}} \in (0,1)$ and every $n$,
\begin{equation} \label{eq:seen-calibration}
\P{Y_{n+1}\notin \hat{C}_{\alpha_{\text{class}}}(X_{n+1};\mathcal{Y}_n),\, Y_{n+1}\in\mathcal{Y}_n} \le \alpha_{\text{class}}\, (1 - \mu_n).
\end{equation}
\item[(ii)] Suppose new labels become rare, i.e., $\mu_n \to 0$. This case requires no further assumption, for either closed-set component.
\item[(iii)] Suppose the closed-set component is~\eqref{eq:weighted-split} with $\pi(1) = 0$. In this case,~\eqref{eq:seen-calibration} also holds for every fixed $\alpha_{\text{class}} \in (0,1)$ and every $n$, without further assumptions.
\end{enumerate}
\end{lemma}

Parts (i) and (ii) apply to both closed-set components, including uniform random splitting, which is the special case of~\eqref{eq:weighted-split} with a constant $\pi$. Part (iii) covers the selective splitting of Section~\ref{sec:selective-splitting} with $\pi(1) = 0$, for which Assumption~\ref{assumption:novel-miscoverage} can fail: for a candidate label not in $\mathcal{Y}_n$, every swap in~\eqref{eq:conformalization-weights} would give this label frequency one, and hence selection probability $\pi(1) = 0$, so the weights of all calibration points vanish, and the auxiliary set always contains new labels.

\begin{proof}[of Lemma~\ref{lemma:seen-calibration}]
For (i), let $\mathcal{Y}_{n+1} := \mathcal{Y}_n \cup \{Y_{n+1}\}$, and consider the auxiliary set $\hat{C}_{\alpha_{\text{class}}}(X_{n+1};\mathcal{Y}_{n+1})$, computed like the closed-set component but with $\mathcal{Y}_{n+1}$ in place of $\mathcal{Y}_n$. This set is not computable in practice, but it has two useful properties. First, it coincides with $\hat{C}_{\alpha_{\text{class}}}(X_{n+1};\mathcal{Y}_n)$ on $\{Y_{n+1} \in \mathcal{Y}_n\}$, because $\mathcal{Y}_{n+1} = \mathcal{Y}_n$ there. Second, its candidate labels always include $Y_{n+1}$, so, by exchangeability, it has marginal coverage at least $1-\alpha_{\text{class}}$, by Lemma~\ref{lemma:coverage-standard-oracle} for~\eqref{eq:standard-prediction-set} and by the proof of Theorem~\ref{thm:weighted-split-validity} for~\eqref{eq:weighted-split}.

If $\mu_n = 1$, the left-hand side of~\eqref{eq:seen-calibration} is zero. If $\mu_n = 0$, then $Y_{n+1} \in \mathcal{Y}_n$ almost surely, so the left-hand side of~\eqref{eq:seen-calibration} is the miscoverage of the auxiliary set, which is at most $\alpha_{\text{class}}$. Otherwise, for $0 < \mu_n < 1$, define
\begin{align*}
\epsilon_{\text{new}} & := \P{Y_{n+1}\notin \hat{C}_{\alpha_{\text{class}}}(X_{n+1};\mathcal{Y}_{n+1}) \mid Y_{n+1} \notin \mathcal{Y}_n}, \\
\epsilon_{\text{seen}} & := \P{Y_{n+1}\notin \hat{C}_{\alpha_{\text{class}}}(X_{n+1};\mathcal{Y}_{n+1}) \mid Y_{n+1} \in \mathcal{Y}_n},
\end{align*}
so that Assumption~\ref{assumption:novel-miscoverage} states $\epsilon_{\text{new}} \geq \epsilon_{\text{seen}}$. Then
\begin{align*}
\alpha_{\text{class}}
\geq \P{Y_{n+1}\notin \hat{C}_{\alpha_{\text{class}}}(X_{n+1};\mathcal{Y}_{n+1})}
= \mu_n\, \epsilon_{\text{new}} + (1-\mu_n)\, \epsilon_{\text{seen}}
\geq \epsilon_{\text{seen}},
\end{align*}
and, by the first property of the auxiliary set, the left-hand side of~\eqref{eq:seen-calibration} equals $(1-\mu_n)\, \epsilon_{\text{seen}} \leq \alpha_{\text{class}}(1-\mu_n)$. In all cases,~\eqref{eq:seen-calibration} holds for every $n$, so the difference in~\eqref{eq:seen-calibration-asymptotic} is non-positive for every $n$, and Assumption~\ref{assumption:seen-calibration} holds.

For (ii), the left-hand side of~\eqref{eq:seen-calibration} is at most $\alpha_{\text{class}}$, by the last claim of Theorem~\ref{theorem:coverage-standard} for~\eqref{eq:standard-prediction-set} and by Theorem~\ref{thm:weighted-split-validity} for~\eqref{eq:weighted-split}. Hence, the difference in~\eqref{eq:seen-calibration-asymptotic} is at most $\alpha_{\text{class}} - \alpha_{\text{class}}(1-\mu_n) = \alpha_{\text{class}}\,\mu_n$, which tends to zero because $\mu_n \to 0$, and Assumption~\ref{assumption:seen-calibration} holds.

Part (iii) is proved in Appendix~\ref{app:selective-proof}, after the proof of Theorem~\ref{thm:weighted-split-validity}, whose notation it uses.
\end{proof}

\begin{proof}[of Proposition~\ref{prop:sharpened-coverage}]
By~\eqref{eq:open-set-predictor}, the prediction set $\hat{C}_{\alpha}^*(X_{n+1})$ can miss $Y_{n+1}$ in only three ways.
If $Y_{n+1}$ is new, it is missed only if the joker is excluded, which requires $\psi_{\text{unseen}} \leq \alpha_{\text{unseen}}$.
If $Y_{n+1}$ has been seen, it is missed only if $\psi_{\text{seen}} \leq \alpha_{\text{seen}}$, in which case the prediction set may consist of the joker alone, or if $Y_{n+1} \notin \hat{C}_{\alpha_{\text{class}}}(X_{n+1};\mathcal{Y}_n)$.
Since $H_{\text{unseen}} = N_{n+1}$ and $H_{\text{seen}} = \{Y_{n+1} \in \mathcal{Y}_n\}$, a union bound gives
\begin{align*}
\P{Y_{n+1} \notin \hat{C}_{\alpha}^*(X_{n+1})}
&\leq \P{\psi_{\text{unseen}} \leq \alpha_{\text{unseen}},\, H_{\text{unseen}}}
+ \P{\psi_{\text{seen}} \leq \alpha_{\text{seen}},\, H_{\text{seen}}} \\
&\qquad + \P{Y_{n+1}\notin \hat{C}_{\alpha_{\text{class}}}(X_{n+1};\mathcal{Y}_n),\, Y_{n+1}\in\mathcal{Y}_n} \\
&\leq \alpha_{\text{unseen}} + \alpha_{\text{seen}} + \alpha_{\text{class}}(1-\mu_n),
\end{align*}
where the last step uses the validity of $\psi_{\text{unseen}}$ and $\psi_{\text{seen}}$, and Lemma~\ref{lemma:seen-calibration}(i).
\end{proof}

\begin{proof}[of Theorem~\ref{prop:gt-adaptive-allocation}]
By Lemma~\ref{lemma:seen-calibration}(i), Assumption~\ref{assumption:novel-miscoverage} implies Assumption~\ref{assumption:seen-calibration}; we prove the result under the latter condition.

Let $\nu := \min\{\tilde\alpha_{\text{unseen}}, q\}$ and $g := \min\{1, (\alpha - \nu - \tilde\alpha_{\text{seen}})/(1-q)\}$, and note that $g > 0$ because $\alpha - \nu - \tilde\alpha_{\text{seen}} \geq \tilde\alpha_{\text{class}} > 0$.
Since $\mu_n \to q$ and $\widehat\mu_n - \mu_n \to 0$ in probability, $\widehat\mu_n \to q < 1$ in probability, so~\eqref{eq:plugin-adjustment} gives $\alpha^{\text{mm}}_{\text{unseen}} \to \nu$ and $\alpha^{\text{mm}}_{\text{class}} \to g$ in probability.

As in the proof of Proposition~\ref{prop:sharpened-coverage}, every miscoverage event is contained in the union of three events: the closed-set component misses a previously seen label, the test of $H_{\text{seen}}$ rejects when $Y_{n+1} \in \mathcal{Y}_n$, or the test of $H_{\text{unseen}}$ rejects when $Y_{n+1}$ is new. We bound each in turn.

\emph{Classification.} Suppose first $g<1$, and fix $\varepsilon \in (0, 1-g)$. Closed-set conformal prediction sets, including~\eqref{eq:standard-prediction-set} and~\eqref{eq:weighted-split}, are nested, decreasing in the level $\alpha$, so on the event $\{\alpha^{\text{mm}}_{\text{class}} \leq g + \varepsilon\}$ any label missed at level $\alpha^{\text{mm}}_{\text{class}}$ is also missed at the fixed level $g + \varepsilon$. Hence,
\begin{align*}
& \P{Y_{n+1}\notin \hat{C}_{\alpha^{\text{mm}}_{\text{class}}}(X_{n+1};\mathcal{Y}_n),\, Y_{n+1}\in\mathcal{Y}_n} \\
& \qquad \leq \P{Y_{n+1}\notin \hat{C}_{g+\varepsilon}(X_{n+1};\mathcal{Y}_n),\, Y_{n+1}\in\mathcal{Y}_n} + \P{\alpha^{\text{mm}}_{\text{class}} > g + \varepsilon}.
\end{align*}
By Assumption~\ref{assumption:seen-calibration} at the fixed level $g+\varepsilon$, the limit superior of the first term is at most $\lim_{n\to\infty} (g+\varepsilon)(1-\mu_n) = (g+\varepsilon)(1-q)$, and the second term tends to zero, because $\alpha^{\text{mm}}_{\text{class}} \to g$ in probability. Therefore,
\begin{align*}
\limsup_{n\to\infty} \P{Y_{n+1}\notin \hat{C}_{\alpha^{\text{mm}}_{\text{class}}}(X_{n+1};\mathcal{Y}_n),\, Y_{n+1}\in\mathcal{Y}_n} \leq (g+\varepsilon)(1-q),
\end{align*}
and letting $\varepsilon \downarrow 0$ gives the bound $g(1-q)$. If $g=1$, the same bound follows directly from
\begin{align*}
\P{Y_{n+1}\notin \hat{C}_{\alpha^{\text{mm}}_{\text{class}}}(X_{n+1};\mathcal{Y}_n),\, Y_{n+1}\in\mathcal{Y}_n} \leq \P{Y_{n+1}\in\mathcal{Y}_n} = 1 - \mu_n \to 1-q.
\end{align*}

\emph{Seen-label test.} The seen-label test uses the fixed level $\alpha^{\text{mm}}_{\text{seen}} = \tilde\alpha_{\text{seen}}$, so, by the validity of $\psi_{\text{seen}}$,
\begin{align*}
\P{\psi_{\text{seen}} \leq \alpha^{\text{mm}}_{\text{seen}},\, H_{\text{seen}}} \leq \tilde\alpha_{\text{seen}}.
\end{align*}

\emph{Unseen-label test.} Since $\alpha^{\text{mm}}_{\text{unseen}} \leq \tilde\alpha_{\text{unseen}}$, the validity of $\psi_{\text{unseen}}$ gives
\begin{align*}
\P{\psi_{\text{unseen}} \leq \alpha^{\text{mm}}_{\text{unseen}},\, H_{\text{unseen}}}
& \leq \min\bigl\{\P{\psi_{\text{unseen}} \leq \tilde\alpha_{\text{unseen}},\, H_{\text{unseen}}},\, \P{H_{\text{unseen}}}\bigr\} \\
& \leq \min\{\tilde\alpha_{\text{unseen}},\, \mu_n\},
\end{align*}
whose limit superior is at most $\min\{\tilde\alpha_{\text{unseen}},\, q\} = \nu$.

Combining the three bounds with the union bound, and using that the limit superior of a sum is at most the sum of the limit superiors,
\begin{align*}
\limsup_{n\to\infty}\P{Y_{n+1}\notin\hat C^*_{\alpha,n}(X_{n+1})}
& \leq g(1-q) + \tilde\alpha_{\text{seen}} + \nu \\
& = \min\{1-q,\ \alpha - \nu - \tilde\alpha_{\text{seen}}\} + \tilde\alpha_{\text{seen}} + \nu \\
& = \min\{1-q + \nu + \tilde\alpha_{\text{seen}},\ \alpha\} \\
& \leq \alpha.
\end{align*}
\end{proof}

\subsection{Asymptotic Analysis} \label{app:asymptotic-proofs}

\begin{proof}[of Theorem~\ref{thm:asymptotic-vanishing}]
Let $S := M_1(Y_{1:(n+1)})$ be the number of singletons in the augmented sample.
The test label is new if and only if it appears exactly once in $Y_{1:(n+1)}$. By exchangeability,
\begin{align*}
\mu_n
= \P{N_{n+1}}
= \frac{\sum_{i=1}^{n+1} \P{N(Y_i; Y_{1:(n+1)}) = 1}}{n+1}
= \frac{\E{\sum_{i=1}^{n+1} \I{N(Y_i; Y_{1:(n+1)}) = 1}}}{n+1}
= \frac{\E{S}}{n+1},
\end{align*}
where $N(y; Y_{1:(n+1)})$ is the number of times $y$ appears in $Y_{1:(n+1)}$.
Adding one observation decreases the number of singletons by at most one, so $M_1 \leq S + 1$, where $M_1$ denotes the number of singletons in $Y_{1:n}$, and hence
\begin{align*}
\E{\widehat\mu_n} = \frac{\E{M_1} + 1}{n+1} \leq \frac{\E{S} + 2}{n+1} = \mu_n + \frac{2}{n+1} \to 0.
\end{align*}
Since $\widehat\mu_n \geq 0$, it follows that $\widehat\mu_n \to 0$ in probability. On the event $\{\widehat\mu_n \leq a\}$, whose probability tends to one, $\alpha^{\text{mm}}_{\text{unseen}} = \widehat\mu_n \geq \psi_{\text{unseen}}$, so the joker is excluded, and~\eqref{eq:plugin-adjustment} gives $\alpha^{\text{mm}}_{\text{class}} = \min\{(\alpha - \widehat\mu_n)/(1-\widehat\mu_n), 1\}$, which tends to $\alpha$ in probability.
This proves the first part of the result.

For the coverage claim, consider first the lower bound. By Lemma~\ref{lemma:seen-calibration}(ii), Assumption~\ref{assumption:seen-calibration} holds, because $\mu_n \to 0$: when new labels become rare, conditioning on the test label having been seen makes no difference asymptotically. Since $\widehat\mu_n$ is also consistent, Theorem~\ref{prop:gt-adaptive-allocation} applies with $q = 0$ and gives $\liminf_{n\to\infty} \P{Y_{n+1}\in\hat C^*_{\alpha,n}(X_{n+1})} \geq 1-\alpha$.

For the upper bound, fix $\varepsilon \in (0, \alpha)$. The test label is covered only if it belongs to the closed-set component or the joker is included. By nestedness, on the event $\{\alpha^{\text{mm}}_{\text{class}} \geq \alpha - \varepsilon\}$, the closed-set component is contained in its version at the fixed level $\alpha - \varepsilon$. Therefore,
\begin{align*}
\P{Y_{n+1}\in\hat C^*_{\alpha,n}(X_{n+1})}
& \leq \P{Y_{n+1} \in \hat{C}_{\alpha-\varepsilon}(X_{n+1};\mathcal{Y}_n)} + \P{\alpha^{\text{mm}}_{\text{class}} < \alpha - \varepsilon} + \P{\joker\in\hat C^*_{\alpha,n}(X_{n+1})} \\
& \leq 1 - (\alpha - \varepsilon) + \frac{1}{n+1} + \P{\alpha^{\text{mm}}_{\text{class}} < \alpha - \varepsilon} + \P{\joker\in\hat C^*_{\alpha,n}(X_{n+1})},
\end{align*}
where the second step uses the upper bound of Theorem~\ref{theorem:coverage-standard}. The last two terms tend to zero, by the first part of the proof, so the limit superior of the coverage is at most $1 - \alpha + \varepsilon$. Letting $\varepsilon \downarrow 0$ completes the proof.
\end{proof}

\begin{proof}[of Proposition~\ref{prop:oracle-optimality}]
\emph{Coverage.} Fix $x$, and let $S_r := \sum_{j \leq r} p_{(j)}(x)$, with $S_0 := 0$. Given $X_{n+1} = x$, if the test label is the recurring label $y_{(r)}$, which has probability $p_{(r)}(x)$, its oracle score $S_{r-1}+(1-U)\,p_{(r)}(x)$ is uniform on the interval $[S_{r-1}, S_r)$, of length $p_{(r)}(x)$. These intervals partition $[0, m(x))$, so the test label is recurring and has oracle score below $t$ with conditional probability $\min\{t, m(x)\}$. Since the closed-set oracle includes a recurring label if and only if its oracle score is below $\tau_\nu$, the test label is recurring and belongs to $C^{\mathrm{or}}_{\tau_\nu}(x)$ with conditional probability $\min\{\tau_\nu, m(x)\}$, which equals $\tau_\nu$ for almost every $x$, by Condition (C1).
Since a recurring label is covered by the open-set oracle if and only if it belongs to its closed-set component, the oracle misses a recurring label with probability $(1-q)-\tau_\nu = \alpha-\nu$.

A non-recurring label is covered if and only if the joker is included, so
\begin{align*}
\P{Y_{n+1} \notin \mathcal{A},\ \joker \notin C^{\mathrm{or},\star}_{\alpha}(X_{n+1})}
& = \E{z(X)\,\I{\Phi(z(X)) \leq \nu}} \\
& = \sup_{u \,:\, \Phi(u) \leq \nu} \E{z(X)\,\I{z(X) \leq u}} \\
& = \sup_{u \,:\, \Phi(u) \leq \nu} \Phi(u),
\end{align*}
where the second equality holds by monotone convergence, because, since $\Phi$ is non-decreasing, the event $\{\Phi(z(X)) \leq \nu\}$ is the increasing union of the events $\{z(X) \leq u\}$ over all $u$ with $\Phi(u) \leq \nu$. By Condition (C2), $\Phi$ is continuous, since its jump at any $t \in (0,1]$ is $t\,\P{z(X)=t} = 0$. Since $\Phi(0) = 0 \leq \nu \leq q = \Phi(1)$, there is therefore $u$ with $\Phi(u) = \nu$, and the supremum equals $\nu$. 

The conditional coverage~\eqref{eq:c-nu} is the probability $\tau_\nu$ of covering a recurring label, plus the probability $z(x)$ of a non-recurring label when the joker is included.

\emph{Optimality of the closed-set component.} Fix $x$ with $m(x) \geq \tau_\nu$, which holds for almost every $x$, by Condition (C1). Let $b>0$ be the probability of the boundary label of $C^{\mathrm{or}}_{\tau_\nu}(x)$, and let $r'_y$ and $r_y^{\text{or}}$ denote the inclusion probabilities of a label $y \in \mathcal{A}$ by a competing randomized set $C'(x) \subseteq \mathcal{A}$ and by $C^{\mathrm{or}}_{\tau_\nu}(x)$, respectively. The oracle includes every label with $p(y \mid x)>b$ with probability one, and every label with $p(y \mid x)<b$ with probability zero, so $(r'_y-r_y^{\text{or}})\{p(y \mid x)-b\}\leq0$ for every $y$.
Summing over $y \in \mathcal{A}$ gives
\begin{align*}
  b\sum_{y \in \mathcal{A}}(r'_y-r_y^{\text{or}})
  & \geq \sum_{y \in \mathcal{A}} (r'_y-r_y^{\text{or}})\,p(y \mid x) \\
  & = \P{Y_{n+1} \in C'(X_{n+1}) \mid X_{n+1} = x} - \P{Y_{n+1} \in C^{\mathrm{or}}_{\tau_\nu}(X_{n+1}) \mid X_{n+1} = x} \\
  & \geq \tau_{\nu} - \tau_{\nu} = 0,
\end{align*}
because $C'(x)$ has conditional coverage at least $\tau_\nu$, which the oracle attains exactly. Hence, $\sum_{y \in \mathcal{A}} r'_y \geq \sum_{y \in \mathcal{A}} r^{\text{or}}_y$, i.e., the competing set has expected cardinality at least that of $C^{\mathrm{or}}_{\tau_\nu}(x)$.

\emph{Optimality of the joker use.} Let $r'(x)$ and $r^{\text{or}}(x):=\I{\Phi(z(x))>\nu}$ denote the probabilities that the joker is included at $x$ by a competing randomized rule $C'$ and by the open-set oracle, respectively. We need to show that $\E{r'(X)} \geq \E{r^{\text{or}}(X)}$.

First, the oracle includes the joker exactly at the inputs where $z(x)$ is large. Indeed, let $s := \sup\{u \in [0,1] : \Phi(u) \leq \nu\}$; since $\Phi$ is non-decreasing, $r^{\text{or}}(x) = 1$ if $z(x) > s$, and $r^{\text{or}}(x) = 0$ if $z(x) < s$. Hence, wherever $z(x) > s$, the competing rule cannot include the joker more often than the oracle, i.e., $r'(x) \leq r^{\text{or}}(x)$, and wherever $z(x) < s$, it cannot include it less often, i.e., $r'(x) \geq r^{\text{or}}(x)$. In both cases, $(r'(x)-r^{\text{or}}(x))(z(x)-s) \leq 0$, and this also holds trivially if $z(x) = s$.

Second, taking expectations gives
\begin{align*}
  s\,\E{r'(X)-r^{\text{or}}(X)}
  & \geq \E{(r'(X)-r^{\text{or}}(X))\,z(X)} \\
  & = \P{Y_{n+1} \notin \mathcal{A},\ \joker \notin C^{\mathrm{or},\star}_{\alpha}(X_{n+1})} - \P{Y_{n+1} \notin \mathcal{A},\ \joker \notin C'(X_{n+1})} \\
  & \geq \nu - \nu = 0,
\end{align*}
because the competing rule leaves a non-recurring label uncovered with probability at most $\nu$, while the oracle does so with probability exactly $\nu$, by (i).

Third, $s>0$: if $q > 0$, then $\nu > 0 = \Phi(0)$, and $\Phi$ is right-continuous; otherwise, $\Phi \equiv 0$, so $s=1$. Dividing by $s$ gives $\E{r'(X)} \geq \E{r^{\text{or}}(X)}$, i.e., the competing rule includes the joker at least as often as the oracle.

\end{proof}

\begin{proof}[of Theorem~\ref{thm:oracle-convergence}]
The proof has three parts. We first show that the seen labels in $\hat C^*_{\alpha,n}(X_{n+1})$ agree with the closed-set oracle $C^{\mathrm{or}}_{\tau_\nu}(X_{n+1})$ with probability tending to one, and then that the joker decision agrees with that of the open-set oracle $C^{\mathrm{or},\star}_{\alpha}$. Finally, we derive the coverage claims.

\emph{Preliminaries.} Let $\gamma := (\alpha-\nu)/(1-q)$, so that $\tau_\nu = (1-q)(1-\gamma)$. Realize selective splitting with independent $\mathrm{Bernoulli}(\varrho)$ variables $\xi_1, \ldots, \xi_n$, independent of the data. Let $\Omega_n$ be the event that every recurring label appears at least three times in $Y_{1:n}$; since $\mathcal{A}$ is finite, $\P{\Omega_n}\to1$.
On $\Omega_n$, the singletons of $Y_{1:n}$ are exactly the observations with non-recurring labels, so $M_1$ counts these observations, and the law of large numbers gives $\widehat\mu_n = (M_1+1)/(n+1) \to q$ almost surely.
Similarly, $|\mathcal{D}^{\text{cal}}|/n\to\varrho(1-q)$ almost surely, and, by~\eqref{eq:plugin-adjustment} with $\tilde\alpha_{\text{seen}} = 0$,
\begin{align*}
\alpha^{\text{mm}}_{\text{class}} = \min\Bigl\{\frac{\alpha - \min\{a, \widehat\mu_n\}}{1-\widehat\mu_n},\, 1\Bigr\}\to\gamma \quad \text{almost surely}.
\end{align*}

\emph{Oracle scores.} For an input $x$ and a recurring label $y$ with rank $r$ under $p(\cdot \mid x)$, define the oracle score $E^\circ(x,y;U) := \sum_{j<r} p_{(j)}(x) + (1-U)\, p_{(r)}(x)$. Given $X = x$, the oracle score of each recurring label is uniform on an interval of length $p(y \mid x)$, and these intervals partition $[0,m(x)]$. Hence, for a recurring observation $(X_i, Y_i)$ with auxiliary variable $U_i$, the oracle score $E_i^\circ := E^\circ(X_i,Y_i;U_i)$ has distribution function
\begin{align*}
F(t) := \P{E_i^\circ\leq t\mid Y_i\in\mathcal{A}} = \frac{\E{\min\{t,m(X)\}}}{1-q}, \qquad 0\leq t\leq1.
\end{align*}
The function $F$ is continuous, and, by Condition (C3), $F(t) = t/(1-q)$ for $t \in [0,\tau_\nu+\eta]$; in particular, $F(\tau_\nu) = 1-\gamma$, and $F$ is strictly increasing near $\tau_\nu$. Note also that, by definition, the closed-set oracle includes a recurring label $y$ if and only if $E^\circ(X_{n+1},y;U_{n+1}) < \tau_\nu$.

\emph{Fitted scores.} Let $\delta_n$ denote the supremum in~\eqref{eq:consistent-classifier} from Condition (C4), so that $\delta_n \to 0$ in probability, and let $B_n := \Omega_n \cap \{\delta_n < \eta/(2K)\}$, whose probability tends to one. On $B_n$, two consequences of Condition (C3) hold for every $x \in \mathcal{X}_0$. First, the recurring probabilities differ by at least $\eta$, while their estimates have errors below $\eta/2$, so the estimated probabilities rank the recurring labels in the same order as the true ones. Second, each recurring probability is at least $\eta$, while each estimated singleton probability is below $\eta/2$, so every recurring label is ranked above every singleton label. Since the scores~\eqref{eq:giq-set} are sums of at most $K$ estimated probabilities of recurring labels, it follows that, on $B_n$, every fitted score of a recurring label is within $K\delta_n$ of its oracle score, while every singleton candidate has score at least $m(x)-K\delta_n \geq \tau_\nu + \eta/2$.
On $\Omega_n$, the calibration points are an i.i.d.\ thinned sample of recurring observations, so the empirical distribution of their oracle scores converges uniformly to $F$, by the Glivenko--Cantelli theorem. Their fitted scores are within $K\delta_n$ of their oracle scores, so, by the continuity of $F$,
\begin{align*}
\sup_t\Bigl|\frac{1}{|\mathcal{D}^{\text{cal}}|}\sum_{j\in\mathcal{D}^{\text{cal}}}\I{S_j\geq t}-\bigl\{1-F(t)\bigr\}\Bigr| \to 0 \text{ in probability}.
\end{align*}

\emph{Weighted conformalization.} On $\Omega_n$, no calibration point is paired, and Proposition~\ref{prop:ess-weights} gives $\tilde p^{(j)}(y) = \kappa_y$ for every $j \in \mathcal{D}^{\text{cal}}$ with $Y_j \neq y$, where $\kappa_y = 1$ if $y \in \mathcal{A}$ and $\kappa_y = 1-\varrho$ otherwise; this also holds when $Y_j = y$, since then $\tilde p^{(j)}(y) = 1 = \kappa_y$. Hence, writing $T(y) := \sum_{j\in\mathcal{D}^{\text{cal}}}\I{S_j\geq S_{n+1}^{(y)}}$, the inclusion rule in~\eqref{eq:weighted-split} at level $\alpha^{\text{mm}}_{\text{class}}$ includes $y$ if and only if
\begin{align*}
p^\pi(y) := \frac{1+\kappa_y\, T(y)}{1+\kappa_y |\mathcal{D}^{\text{cal}}|} > \alpha^{\text{mm}}_{\text{class}}.
\end{align*}
A direct calculation gives
\begin{align*}
0 \leq p^\pi(y) - \frac{T(y)}{|\mathcal{D}^{\text{cal}}|}
= \frac{|\mathcal{D}^{\text{cal}}| - T(y)}{|\mathcal{D}^{\text{cal}}|\,(1+\kappa_y |\mathcal{D}^{\text{cal}}|)}
\leq \frac{1}{(1-\varrho)\,|\mathcal{D}^{\text{cal}}|},
\end{align*}
so, by the previous display, $p^\pi(y)=1-F(S_{n+1}^{(y)})+o_p(1)$, uniformly in $y \in \mathcal{Y}_n$.

\emph{The seen labels.} Let $\lambda := 1/(1-q)$, so that $F(t) = \lambda t$ for $t \in [0, \tau_\nu + \eta]$ and $F(\tau_\nu) = 1-\gamma$. Fix $\varepsilon \in (0, \eta)$, and let $G_n$ be the event on which
\begin{align*}
|\alpha^{\text{mm}}_{\text{class}} - \gamma| < \frac{\lambda\varepsilon}{4},
\qquad
\sup_{y \in \mathcal{Y}_n} \bigl|p^\pi(y) - \{1 - F(S_{n+1}^{(y)})\}\bigr| < \frac{\lambda\varepsilon}{4},
\qquad
K\delta_n < \frac{\varepsilon}{2},
\end{align*}
together with $B_n$. By the previous steps, $\P{G_n} \to 1$. Recall that $y$ is included if and only if $p^\pi(y) > \alpha^{\text{mm}}_{\text{class}}$. On $G_n$, we distinguish three cases.

\begin{enumerate}
\item[(a)] \emph{Recurring labels with oracle score below $\tau_\nu - \varepsilon$.} The fitted score is below $\tau_\nu - \varepsilon/2$, so $F(S_{n+1}^{(y)}) \leq 1 - \gamma - \lambda\varepsilon/2$, and hence
$p^\pi(y) > \gamma + \lambda\varepsilon/2 - \lambda\varepsilon/4 > \alpha^{\text{mm}}_{\text{class}}$: the label is included, as by the oracle.
\item[(b)] \emph{Recurring labels with oracle score above $\tau_\nu + \varepsilon$.} The fitted score is above $\tau_\nu + \varepsilon/2$, so $F(S_{n+1}^{(y)}) \geq 1 - \gamma + \lambda\varepsilon/2$, and hence
$p^\pi(y) < \gamma - \lambda\varepsilon/2 + \lambda\varepsilon/4 < \alpha^{\text{mm}}_{\text{class}}$: the label is excluded, as by the oracle.
\item[(c)] \emph{Singleton candidates.} On $B_n$, their scores are at least $\tau_\nu + \eta/2 \geq \tau_\nu + \varepsilon/2$, so, as in (b), they are excluded, as by the oracle.
\end{enumerate}
Therefore, on $G_n$, the seen labels of $\hat C^*_{\alpha,n}(X_{n+1})$ can differ from $C^{\mathrm{or}}_{\tau_\nu}(X_{n+1})$ only if some recurring label has oracle score within $\varepsilon$ of $\tau_\nu$. Given $X_{n+1}$, each oracle score is uniform on an interval of length at least $\eta$, by Condition (C3), so each of the $K$ recurring labels has this property with probability at most $2\varepsilon/\eta$. Hence,
\begin{align*}
\limsup_{n\to\infty} \P{\text{the seen labels differ}} \leq \limsup_{n\to\infty} \P{G_n^c} + \frac{2K\varepsilon}{\eta} = \frac{2K\varepsilon}{\eta},
\end{align*}
and letting $\varepsilon\downarrow0$ shows that the seen labels agree with probability tending to one.

\emph{The joker.} Recall that $\Phi(t) = \E{z(X)\,\I{z(X) \leq t}}$ is the probability that the label is non-recurring and $z(X) \leq t$, and that the open-set oracle includes the joker if and only if $\Phi(z(X_{n+1})) > \nu$.
Our method includes the joker if and only if $\psi_{\text{unseen}} > \alpha^{\text{mm}}_{\text{unseen}} = \min\{a,\widehat\mu_n\}$, and, since $\psi_{\text{unseen}} \leq \widehat\mu_n$, this event equals $\{\psi^{\text{XGT}}_0 > a\}$. If $q = 0$, then $M_1 = 0$ on $\Omega_n$, so $\psi^{\text{XGT}}_0 = 1/(n+1) < a$ and the joker is excluded with probability tending to one, as by the open-set oracle since $\Phi \equiv 0 = \nu$. Suppose now $q>0$, and let $W := \Phi(z(X_{n+1}))$.

First, $\psi^{\text{XGT}}_0 = W + o_p(1)$. On $\Omega_n$, the singletons are the observations with non-recurring labels, so $\mathcal{S}_1$ is an i.i.d.\ sample from the distribution of $X$ given $Y \notin \mathcal{A}$, of size $M_1 = qn + o_p(n)$. By Condition (C5), the Glivenko--Cantelli theorem, and the continuity of the distribution of $z(X)$ given $Y \notin \mathcal{A}$, which follows from Condition (C2),
\begin{align*}
\psi^{\text{XGT}}_0
& = \frac{1}{n+1}\Bigl(1 + \sum_{i \in \mathcal{S}_1} \I{\hat{s}_0(X_i) \geq \hat{s}_0(X_{n+1})}\Bigr) \\
& = q\,\P{z(X) \leq z(x) \mid Y \notin \mathcal{A}}\big|_{x = X_{n+1}} + o_p(1) \\
& = W + o_p(1),
\end{align*}
where the middle step uses that $\hat{s}_0$ ranks inputs like $h(m(x)) = h(1-z(x))$, i.e., in decreasing order of $z(x)$, and the last step uses $q\,\P{z(X) \leq t \mid Y \notin \mathcal{A}} = \E{z(X)\,\I{z(X) \leq t}} = \Phi(t)$.

Second, $\{W > a\}$ coincides with the event $\{W > \nu\}$ that the oracle includes the joker, and $\P{W = a} = 0$. Indeed, $\Phi$ is continuous and non-decreasing, and it is constant on an interval only if $z(X)$ puts no mass there, apart from the point $0$, because $\Phi$ increases by $\E{z(X)\,\I{z(X) \in (s,t]}}$ on $(s,t]$.
Since $a > 0 = \Phi(0)$, the level set $\{t : \Phi(t) = a\}$ is an interval $[t_1, t_2]$ with $t_1 > 0$, on which $z(X)$ puts no mass, by the previous observation and Condition (C2); hence $\P{W = a} = 0$.
If $a < q$, then $\nu = a$, so the two events are identical. If $a \geq q$, then $\nu = q$ and $W \leq \Phi(1) = q \leq a$, so both $\{W > a\}$ and $\{W > \nu\}$ are empty.

Together, these two facts imply $\P{\I{\psi^{\text{XGT}}_0 > a} \neq \I{\Phi(z(X_{n+1})) > \nu}} \to 0$, and, combined with the result for the seen labels, $\P{\hat C^*_{\alpha,n}(X_{n+1}) \neq C^{\mathrm{or},\star}_{\alpha}(X_{n+1})} \to 0$.

\emph{Coverage.} By Proposition~\ref{prop:oracle-optimality}(i), the open-set oracle has marginal coverage $1-\alpha$ and conditional coverage $c_\nu$.
Finally, with probability tending to one, $\mathcal{A}\subseteq\mathcal{Y}_n$ and no non-recurring test label has been observed, so the coverage events of $\hat C^*_{\alpha,n}$ and $C^{\mathrm{or},\star}_{\alpha}$ agree whenever the two sets agree. The marginal coverage limit and the $L^1$ convergence of conditional coverages follow, since both differences are bounded by the probability that the two sets disagree.
\end{proof}

\subsection{Frequency-Based Selective Sample Splitting} \label{app:selective-proof}

The proof of Theorem~\ref{thm:weighted-split-validity} relies on the following general result, which states explicitly the argument behind weighted conformal prediction under covariate shift \citep{tibshirani2019conformal}, extending it to arbitrary joint laws.

\begin{lemma}[\citealp{liang2026structured}, Lemma~1] \label{lem:general_quant}
Let $Z_1, \ldots, Z_{m+1}$ be random variables with joint law $f$.
For any fixed function $s$ and $i \in [m+1]$, define $V_i=s( Z_i, Z_{-i} )$, where $Z_{-i}=\{Z_1, \ldots, Z_{m+1}\} \setminus\{Z_i\}$.
Assume that $V_1,\ldots,V_{m+1}$ are distinct almost surely.
Define also
\begin{equation} \label{eq:general_prob}
p^f_i(Z_1,\ldots,Z_{m+1}) :=
\frac{\sum_{\sigma \in \mathcal{S} : \sigma(m+1)=i} f(Z_{\sigma(1)},\ldots, Z_{\sigma(m+1)})}
{\sum_{\sigma \in \mathcal{S}} f(Z_{\sigma(1)},\ldots, Z_{\sigma(m+1)})}, \qquad i \in [m+1],
\end{equation}
where $\mathcal{S}$ is the set of all permutations of $[m+1]$.
Then, for any $\beta\in(0,1)$,
\begin{align*}
\P{ V_{m+1} \leq \mathrm{Quantile}\biggl(\beta; \, \sum_{i=1}^m p^f_i(Z_1,\ldots,Z_{m+1})\, \delta_{V_i} + p^f_{m+1}(Z_1,\ldots,Z_{m+1})\, \delta_\infty\biggr) } \geq \beta.
\end{align*}
\end{lemma}

\begin{proof}[of Theorem~\ref{thm:weighted-split-validity}]
Let $\mathcal{Y}_{n+1} := \text{unique}(Y_1,\ldots,Y_{n+1})$, and let $\hat{C}^{\pi}_{\alpha}(X_{n+1}; \mathcal{Y}_{n+1})$ be the prediction set~\eqref{eq:weighted-split} with $\mathcal{Y}_{n+1}$ in place of $\mathcal{Y}_n$. If $Y_{n+1} \in \mathcal{Y}_n$, then $\mathcal{Y}_{n+1} = \mathcal{Y}_n$, and the two sets coincide. Therefore,
\begin{align*}
\P{Y_{n+1} \notin \hat{C}^{\pi}_{\alpha}(X_{n+1}; \mathcal{Y}_n),\ Y_{n+1} \in \mathcal{Y}_n}
& = \P{Y_{n+1} \notin \hat{C}^{\pi}_{\alpha}(X_{n+1}; \mathcal{Y}_{n+1}),\ Y_{n+1} \in \mathcal{Y}_n} \\
& \leq \P{Y_{n+1} \notin \hat{C}^{\pi}_{\alpha}(X_{n+1}; \mathcal{Y}_{n+1})},
\end{align*}
so it suffices to show that the right-hand side is at most $\alpha$.

Let $h$ denote the joint law of $Z_{1:(n+1)} := (Z_1, \ldots, Z_{n+1})$, which is symmetric because the data are exchangeable. The inclusion indicators are independent given the labels, so, for any $z_{1:(n+1)}$ with labels $y_{1:(n+1)}$ and any $D \subseteq [n]$, writing $p^{(n+1)}(y)$ for the quantity in~\eqref{eq:conformalization-weights-p} computed with calibration set $[n] \setminus D$,
\begin{align*}
& \P{Z_{1:(n+1)} = z_{1:(n+1)},\ \mathcal{D}^{\text{train}} = D} \\
& \qquad = h(z_1, \ldots, z_{n+1}) \prod_{i \in [n] \setminus D} \pi\bigl(N(y_i; y_{1:n})\bigr) \prod_{i \in D} \bigl\{1 - \pi(N(y_i; y_{1:n}))\bigr\} \\
& \qquad = h(z_1, \ldots, z_{n+1})\; p^{(n+1)}(y_{n+1}).
\end{align*}

We now condition on $\mathcal{D}^{\text{train}} = D$ and on the training points $\{Z_i\}_{i \in D}$, which fixes the score function. Let $m := |\mathcal{D}^{\text{cal}}| = n - |D|$, let $Z_1', \ldots, Z_m'$ denote the calibration points, in any order, and let $Z_{m+1}' := Z_{n+1}$ be the test point. For a permutation $\sigma$ of $[m+1]$, placing $Z'_{\sigma(m+1)}$ in the test position and the other points in the calibration positions swaps the test point with the calibration point $Z'_{\sigma(m+1)}$, if $\sigma(m+1) \neq m+1$. Since $p^{(j)}(y)$ depends on the calibration labels only through their multiset, the display above gives, for the conditional joint law $f$ of $Z_1', \ldots, Z_{m+1}'$,
\begin{align*}
f(Z'_{\sigma(1)}, \ldots, Z'_{\sigma(m+1)}) \propto h(Z_1, \ldots, Z_{n+1})\; p^{(j)}(Y_{n+1}),
\end{align*}
where $Z'_{\sigma(m+1)} = Z_j$ for $j \in \mathcal{D}^{\text{cal}} \cup \{n+1\}$.
Each of these terms depends on $\sigma$ only through $\sigma(m+1)$, and $h$ does not depend on $\sigma$, so~\eqref{eq:general_prob} gives
\begin{align*}
p^f_i(Z_1', \ldots, Z_{m+1}')
= \frac{m!\; p^{(j)}(Y_{n+1})}{m! \bigl\{p^{(n+1)}(Y_{n+1}) + \sum_{k \in \mathcal{D}^{\text{cal}}} p^{(k)}(Y_{n+1})\bigr\}}
= w_j(Y_{n+1}),
\qquad \text{where } Z_i' = Z_j,
\end{align*}
by~\eqref{eq:conformalization-weights}.

If the scores are almost surely distinct, Lemma~\ref{lem:general_quant} with $\beta = 1-\alpha$ therefore gives
\begin{align*}
\P{S_{n+1}^{(Y_{n+1})} \leq \mathrm{Quantile}\Bigl(1-\alpha;\; \sum_{j \in \mathcal{D}^{\text{cal}}} w_j(Y_{n+1})\, \delta_{S_j} + w_{n+1}(Y_{n+1})\, \delta_{\infty}\Bigr) \;\Big|\; \mathcal{D}^{\text{train}}, \{Z_i\}_{i \in \mathcal{D}^{\text{train}}}} \geq 1-\alpha.
\end{align*}
The weighted distribution in this display assigns mass $w_{n+1}(Y_{n+1}) + \sum_{j\in\mathcal{D}^{\text{cal}}} w_j(Y_{n+1})\,\I{S_j \geq s}$ to $[s, \infty]$, so the event in the display is the inclusion rule in~\eqref{eq:weighted-split}. Marginalizing over the conditioning variables gives $\P{Y_{n+1} \notin \hat{C}^{\pi}_{\alpha}(X_{n+1}; \mathcal{Y}_{n+1})} \leq \alpha$.

If the scores are not almost surely distinct, break ties at random using independent noise. This does not change the weights, and it can only remove labels from the prediction set, so the coverage bound for the tie-broken set carries over to the original one.
\end{proof}

\begin{proof}[of Lemma~\ref{lemma:seen-calibration}(iii)]
We use the notation of the proof of Theorem~\ref{thm:weighted-split-validity}, with $\alpha_{\text{class}}$ in place of $\alpha$, and assume that the scores are almost surely distinct; ties are handled as in that proof. Let $E := \{Y_{n+1} \in \mathcal{Y}_n\}$, so that $\P{E} = 1 - \mu_n$. If $\mu_n = 1$, the left-hand side of~\eqref{eq:seen-calibration} is zero. Otherwise, since $\mathcal{Y}_{n+1} = \mathcal{Y}_n$ on $E$,
\begin{align} \label{eq:seen-calibration-decomposition}
\P{Y_{n+1} \notin \hat{C}^{\pi}_{\alpha_{\text{class}}}(X_{n+1}; \mathcal{Y}_n),\, Y_{n+1} \in \mathcal{Y}_n}
= (1 - \mu_n)\, \P{Y_{n+1} \notin \hat{C}^{\pi}_{\alpha_{\text{class}}}(X_{n+1}; \mathcal{Y}_{n+1}) \mid E},
\end{align}
and it suffices to show that the conditional probability on the right-hand side is at most $\alpha_{\text{class}}$.

Condition on $\mathcal{D}^{\text{train}} = D$ and on the training points $\{Z_i\}_{i \in D}$. Write $Z' := (Z_1', \ldots, Z_{m+1}')$ and, for a permutation $\sigma$ of $[m+1]$, $Z'_{\sigma} := (Z'_{\sigma(1)}, \ldots, Z'_{\sigma(m+1)})$. For any points $z' = (z_1', \ldots, z_{m+1}')$ with labels $y_1', \ldots, y_{m+1}'$, let $N'(y)$ be the number of times that $y$ appears among the $n+1$ labels formed by $\{Y_i\}_{i \in D}$ and $y_1', \ldots, y_{m+1}'$, and define
\begin{align*}
e(z') := \I{N'(y_{m+1}') \geq 2},
\qquad
c(z') := \prod_{i=1}^{m+1} \I{N'(y_i') \geq 2}.
\end{align*}
The function $e$ indicates that the label of the last point is repeated, so $E = \{e(Z') = 1\}$. The function $c$ indicates that all labels are repeated, and it is invariant to permutations of its arguments.

We first show that, for every permutation $\sigma$ of $[m+1]$,
\begin{align} \label{eq:seen-indicator-symmetric}
f(Z'_{\sigma})\, e(Z'_{\sigma}) = f(Z'_{\sigma})\, c(Z').
\end{align}
If $c(Z') = 1$, all labels are repeated, so $e(Z'_{\sigma}) = 1$. If $c(Z') = 0$, some point $Z_a'$ has a label that is not repeated. When $\sigma(m+1) = a$, this point is in the test position, so $e(Z'_{\sigma}) = 0$. When $\sigma(m+1) \neq a$, this point is in a calibration position, and its label appears only once among the labels of the $n$ points that are not in the test position. The corresponding factor in~\eqref{eq:conformalization-weights-p} is then $\pi(1) = 0$, so $f(Z'_{\sigma}) = 0$. In each case, the two sides of~\eqref{eq:seen-indicator-symmetric} coincide.

Next, let $f_E$ denote the joint law of $Z'$ conditional also on $E$, and let $\kappa := \P{E \mid \mathcal{D}^{\text{train}}, \{Z_i\}_{i \in \mathcal{D}^{\text{train}}}}$. By Bayes' rule and~\eqref{eq:seen-indicator-symmetric}, for every permutation $\sigma$,
\begin{align*}
f_E(Z'_{\sigma}) = \frac{f(Z'_{\sigma})\, e(Z'_{\sigma})}{\kappa} = \frac{c(Z')}{\kappa}\, f(Z'_{\sigma}).
\end{align*}
The factor $c(Z')/\kappa$ does not depend on $\sigma$, and it equals $1/\kappa$ almost surely on $E$, because~\eqref{eq:seen-indicator-symmetric} with $\sigma$ equal to the identity gives $c(Z') = e(Z') = 1$ whenever $f(Z') > 0$. It therefore cancels in~\eqref{eq:general_prob}: almost surely on $E$, for every $i \in [m+1]$,
\begin{align*}
p^{f_E}_i(Z')
= \frac{\sum_{\sigma : \sigma(m+1) = i} f_E(Z'_{\sigma})}{\sum_{\sigma} f_E(Z'_{\sigma})}
= \frac{\sum_{\sigma : \sigma(m+1) = i} f(Z'_{\sigma})}{\sum_{\sigma} f(Z'_{\sigma})}
= p^{f}_i(Z')
= w_j(Y_{n+1}),
\end{align*}
where $j$ is such that $Z_i' = Z_j$. Conditioning on $E$ thus leaves the weights unchanged. Let $\hat{Q}$ denote the $(1-\alpha_{\text{class}})$-quantile of the distribution $\sum_{j \in \mathcal{D}^{\text{cal}}} w_j(Y_{n+1})\, \delta_{S_j} + w_{n+1}(Y_{n+1})\, \delta_{\infty}$. Lemma~\ref{lem:general_quant}, applied to $f_E$ with $\beta = 1-\alpha_{\text{class}}$, gives
\begin{align*}
\P{S_{n+1}^{(Y_{n+1})} \leq \hat{Q} \;\Big|\; \mathcal{D}^{\text{train}}, \{Z_i\}_{i \in \mathcal{D}^{\text{train}}}, E} \geq 1-\alpha_{\text{class}}.
\end{align*}
As in the proof of Theorem~\ref{thm:weighted-split-validity}, the event $\{S_{n+1}^{(Y_{n+1})} \leq \hat{Q}\}$ is $\{Y_{n+1} \in \hat{C}^{\pi}_{\alpha_{\text{class}}}(X_{n+1}; \mathcal{Y}_{n+1})\}$. Averaging over the training set and the training points, conditionally on $E$, therefore gives
\begin{align*}
\P{Y_{n+1} \notin \hat{C}^{\pi}_{\alpha_{\text{class}}}(X_{n+1}; \mathcal{Y}_{n+1}) \mid E} \leq \alpha_{\text{class}}.
\end{align*}
By~\eqref{eq:seen-calibration-decomposition}, this proves~\eqref{eq:seen-calibration}, which implies Assumption~\ref{assumption:seen-calibration}.
\end{proof}

\begin{proof}[of Proposition~\ref{prop:ess-weights}]
Fix $y \in \mathcal{Y}_n$ and $j \in \mathcal{D}^{\text{cal}}$. If $Y_j = y$, the swap changes nothing. Otherwise, by Appendix~\ref{app:computation-shortcut}, $\tilde p^{(j)}(y) = a(Y_j)\, b(y)$, where $a(Y_j)$ and $b(y)$ are the multiplicative changes in the factors of the observations with labels $Y_j$ and $y$, respectively. Since $\pi(1) = 0$ and $\pi = \varrho$ on $\{2, 3, \ldots\}$, the factor $\varrho$ of observation $j$ moves from the factors of label $Y_j$ to those of label $y$, and only two further changes can occur: if $N_{Y_j} = 2$, the other copy of $Y_j$ becomes a singleton, so its factor changes from $\varrho$ to $0$ if $j$ is paired, and from $1-\varrho$ to $1$ if $j$ is split; if $N_y = 1$, the factor of the copy of $y$, a singleton in $\mathcal{D}^{\text{train}}$, changes from $1$ to $1-\varrho$. Hence, $a(Y_j)$ equals $1/\varrho$ if $N_{Y_j} \geq 3$, $0$ if $j$ is paired, and $1/\{\varrho(1-\varrho)\}$ if $j$ is split, while $b(y)$ equals $\varrho$ if $N_y \geq 2$ and $\varrho(1-\varrho)$ otherwise; multiplying gives the four cases.

Thus, a weight vanishes only if $j$ is paired, and $\E{R \mid Y_{1:n}} = 2 M_2 \varrho^2$, because both copies of each of the $M_2$ labels appearing twice are selected with probability $\varrho^2$. Moreover, by the four cases, the nonzero ratios lie in $\{1, 1/(1-\varrho)\}$ if $N_y \geq 2$ and in $\{1-\varrho, 1\}$ otherwise, so the larger value is $r = 1/(1-\varrho)$ times the smaller. As~\eqref{eq:kish-ess} is invariant to a common rescaling of the weights, if a fraction $\varphi$ of the nonzero weights takes the larger value, then $n_{\mathrm{eff}}(y)$ equals their number, which is at least $|\mathcal{D}^{\text{cal}}| - R$, times
\[
\frac{(1 - \varphi + \varphi r)^2}{1 - \varphi + \varphi r^2} = 1 - \frac{\varphi(1-\varphi)(r-1)^2}{1 - \varphi + \varphi r^2} \geq 1 - \frac{(r-1)^2}{4} = 1 - \frac{\varrho^2}{4(1-\varrho)^2},
\]
where the inequality uses $\varphi(1-\varphi) \leq 1/4$ and $r \geq 1$; this gives~\eqref{eq:ess-bound}, trivially when its right-hand side is negative.
Finally,~\eqref{eq:ess-bound-expected} follows by taking conditional expectations in~\eqref{eq:ess-bound}, since $\E{|\mathcal{D}^{\text{cal}}| \mid Y_{1:n}} = \varrho\,(n - M_1)$ and $\E{R \mid Y_{1:n}} = 2 M_2 \varrho^2$, so that $\E{|\mathcal{D}^{\text{cal}}| - R \mid Y_{1:n}} = \varrho\,(n - M_1 - 2\varrho M_2)$.
\end{proof}

\clearpage

\section{Additional Numerical Results}\label{app:experiments}

This appendix gives implementation details for Section~\ref{sec:empirical} and further numerical results.

\subsection{Implementation Details on Classification Model}
\label{app:classification-model-adjustment}

Across all experiments, the base classifier is a $k$-nearest neighbors model, with $k=5$ for the synthetic data and $k=10$ for the CelebA data, chosen because more complex models overfit easily with many labels and very low per-label frequencies. Neighbors are weighted by the inverse of their distance to the query point, under the Minkowski metric for the synthetic data and the cosine metric on the image embeddings for CelebA, where it better reflects visual similarity.

Labels that appear in the calibration set but not in the training set are assigned the following default probability estimate:
\[
\hat{p}_{\text{unseen}} = \frac{1 + M_1^{\text{train}}}{(1 + n_{\text{train}})\, |\mathcal{C}_{\text{unseen}}|},
\]
plus a small independent random perturbation that breaks ties, where $M_1^{\text{train}}$ is the number of labels appearing once in the training set, $n_{\text{train}}$ its size, and $|\mathcal{C}_{\text{unseen}}|$ the number of calibration-only labels. This spreads the Good--Turing estimate of the probability of an unseen label evenly over the calibration-only labels. The same rule is used for the standard split conformal benchmark and for the closed-set component of our method.

\subsection{Implementation Details on Method Hyperparameters} \label{app:hyperparameters}

Our method tests $H_{\text{unseen}}$ with the feature-based $p$-value $\psi_0^{\text{XGT}}$, whose score function is given by a local outlier factor model fitted with \texttt{scikit-learn} on the features of observations whose labels appear at least twice. The complementary hypothesis  $H_{\text{seen}}$ is tested by combining the randomized $p$-values $\psi_k^{\text{RGT}}$ as in~\eqref{eq:combination-pval}, with the weights~\eqref{eq:combination-pval-power-law}. The exponent $\beta$ is selected by cross-validation over the grid $\{1.0, 1.2, \ldots, 3.0\}$, with ties broken toward the default $\beta = 1.6$, as described in Appendix~\ref{app:plugin-adjustment}.
The budget $\alpha$, equal to $0.1$ for the synthetic data and $0.2$ for CelebA, is allocated by Algorithm~\ref{alg:alpha-tuning-loss-optimization}, with $\mathcal{G}_{\text{seen}} = \{0, 0.01, 0.02, 0.05, 0.1\} \cap [0, \alpha)$. The grid $\mathcal{G}(a)$ is $\{0.01, 0.015, 0.02, \ldots\} \cap [0, a]$ for $\alpha_{\text{class}}$ in CGTC, and $\{0, a/20, 2a/20, \ldots, a\}$ for $\tilde\alpha_{\text{unseen}}$ in CGTC+. The loss~\eqref{eq:alpha_allocation_loss} uses $\lambda = 0.5$ and is estimated by $10$-fold cross-validation, with the training and calibration sets randomly partitioned within each fold for fast computation. Appendix~\ref{app:lambda-sensitivity} examines the sensitivity of the results to $\lambda$.

\FloatBarrier

\subsection{Sensitivity to the Exponent $\beta$}
\label{app:beta-sensitivity}

This subsection studies numerically how the combination $p$-value $\psi_{\text{seen}}$ depends on the exponent $\beta$ of the weights~\eqref{eq:combination-pval-power-law}, supporting the choices described in Appendix~\ref{app:plugin-adjustment}. For each $\theta \in \{750, 1000, 1250, 1500,\allowbreak 1750, 2000, 2500, 3000\}$ and $n \in \{500, 1000, 1500, 2000\}$, we draw $100$ label samples $Y_{1:n}$ from the Dirichlet process model of Appendix~\ref{app:dp-example}, and evaluate $\psi_{\text{seen}}(\beta)$ on the grid $\beta \in \{0, 0.1, \ldots, 3\}$, both with the deterministic $p$-values $\psi^{\text{GT}}_k$ (Figure~\ref{fig:app-beta-sensitivity}) and with the randomized $p$-values $\psi^{\text{RGT}}_k$ used in our experiments, averaged over $20$ randomization draws per label sample (Figure~\ref{fig:app-beta-sensitivity-rgt}). Both figures report the average of $\min\{\psi_{\text{seen}}(\beta), 1\}$, since any value above one means that the test cannot reject at any level. The two constructions give similar results, with minimizers within one grid step of each other, as expected since $\E{\psi^{\text{RGT}}_k \mid \mathcal{F}_n} = \psi^{\text{GT}}_k/2 + 1/\{2(n+1)\}$; randomization roughly halves each $p$-value, so the test remains informative on somewhat denser samples.

Three conclusions follow. First, consistent with~\eqref{eq:psi-seen-oracle-bound}, $\psi_{\text{seen}}$ is uninformative for every $\beta$ whenever $n$ is large relative to $\theta$, so the choice of $\beta$ matters only in sparse, heavily open-set regimes. Second, where the test is informative, the curve is U-shaped in $\beta$, with the default $\beta = 1.6$ within a few percent of the optimum, which shifts to $1.9$ to $2.0$ in the sparsest configurations ($\theta \geq 2000$ with $n = 500$). Third, the curve is asymmetric: the combination $p$-value degrades sharply below the optimum and slowly above it, so slightly larger values cost little.

\begin{figure}[!htb]
    \centering
    \includegraphics[width=\textwidth]{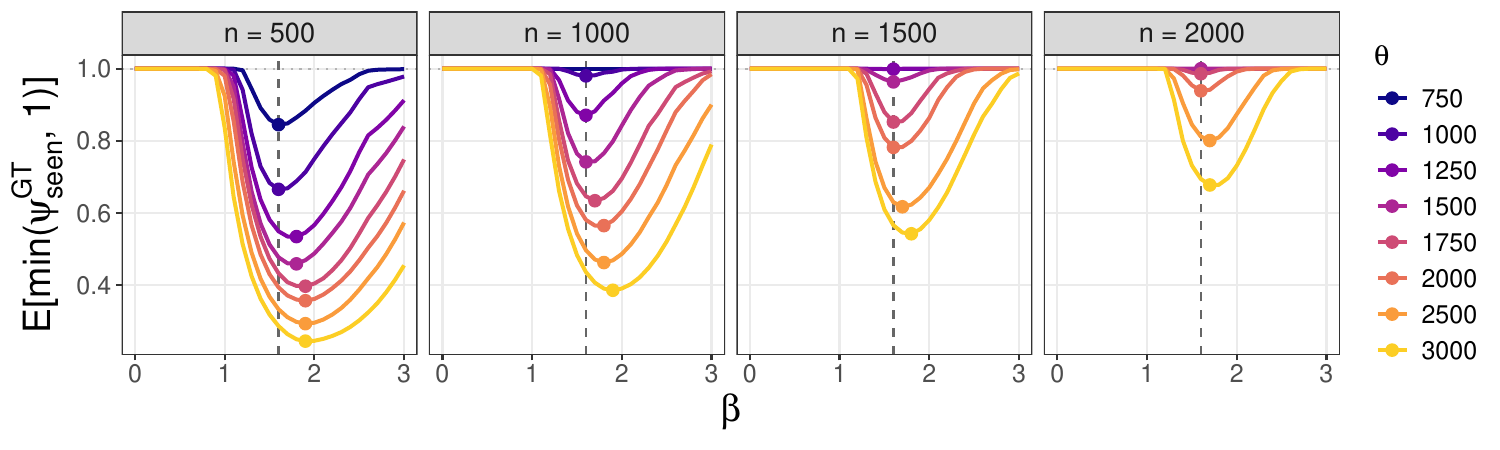}
    \caption{Average truncated combination $p$-value $\E{\min\{\psi_{\text{seen}}(\beta),1\}}$ as a function of $\beta$, for label samples of size $n$ from a Dirichlet process with concentration parameter $\theta$. Dots mark the minimizer of each curve, and the dashed vertical line marks the default $\beta = 1.6$.}
    \label{fig:app-beta-sensitivity}
\end{figure}

\begin{figure}[!htb]
    \centering
    \includegraphics[width=\textwidth]{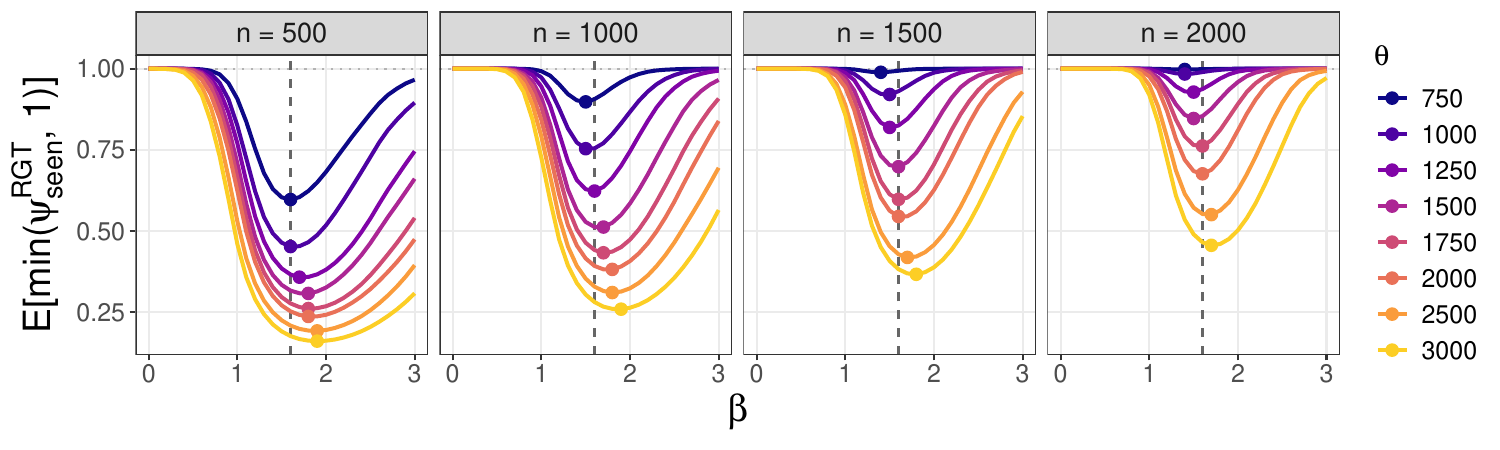}
    \caption{Same as Figure~\ref{fig:app-beta-sensitivity}, but with the randomized $p$-values $\psi^{\text{RGT}}_k$.}
    \label{fig:app-beta-sensitivity-rgt}
\end{figure}

\FloatBarrier

\subsection{Results for CGTC}
\label{app:original-cgtc-experiments}

Figures~\ref{fig:app-dp-orig-four-panel}--\ref{fig:app-dp-orig-pvalues} repeat the synthetic experiments of Section~\ref{sec:synthetic-exp} using the CGTC version of our method, i.e., with the cross-validated allocation of Section~\ref{sec:alpha-allocation} but without the missing-mass adjustment of Section~\ref{sec:plugin-adjustment}. CGTC also achieves valid coverage, overall and across label frequencies, but it is more conservative than CGTC+: with selective splitting, its coverage lies further above the nominal level, and its prediction sets are larger for moderate to large values of $\theta$. This is the conservativeness that the missing-mass adjustment removes.

\begin{figure}[!htb]
    \centering
    \includegraphics[width=\textwidth]{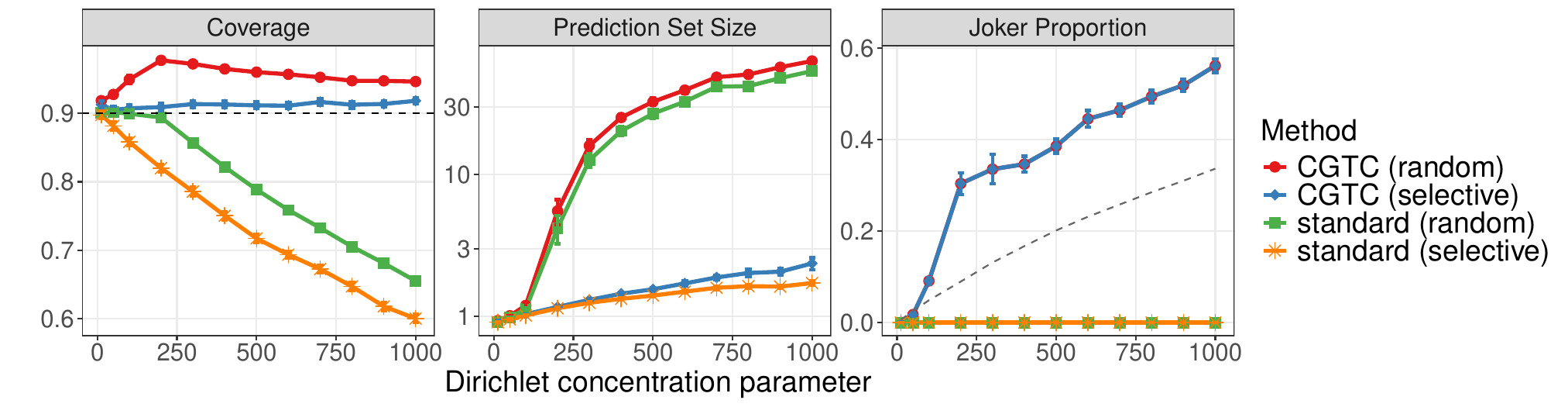}
    \caption{Performance of CGTC on synthetic data. Other details are as in Figure~\ref{fig:dp-main-four-panel}, which shows the same results for CGTC+.}
    \label{fig:app-dp-orig-four-panel}
\end{figure}

\begin{figure}[!htb]
    \centering
    \includegraphics[width=\textwidth]{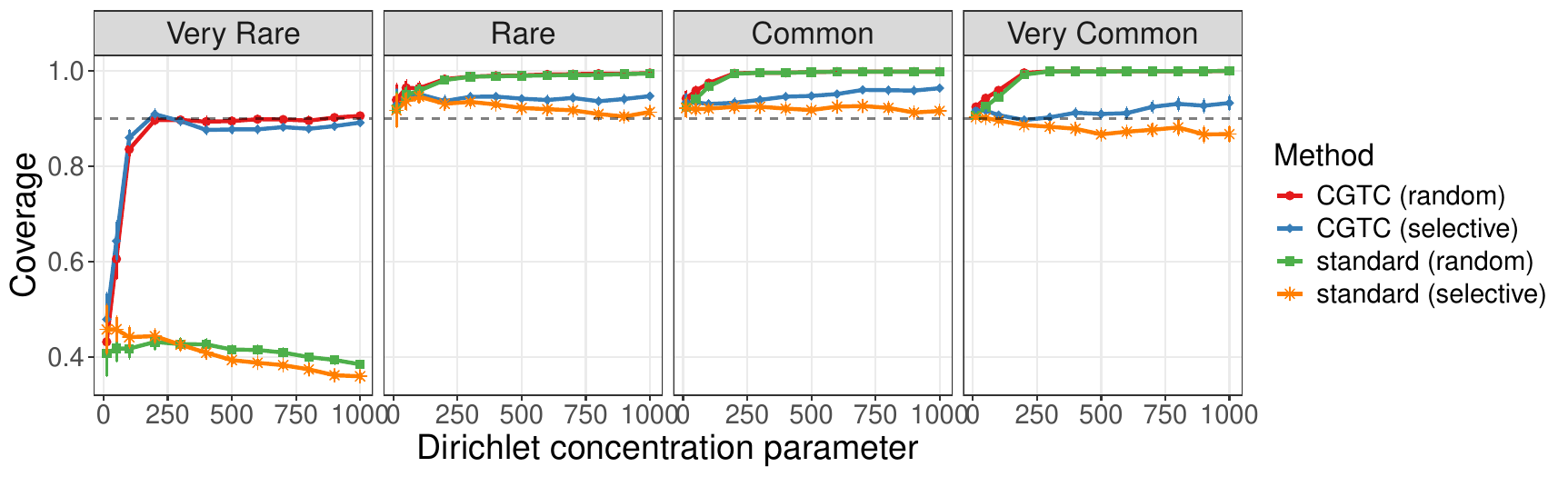}
    \caption{Coverage of CGTC on synthetic data, stratified by test-label frequency. Other details are as in Figure~\ref{fig:dp-main-conditional-coverage}.}
    \label{fig:app-dp-orig-cond-cov}
\end{figure}

\begin{figure}[!htb]
    \centering
    \includegraphics[width=0.6\textwidth]{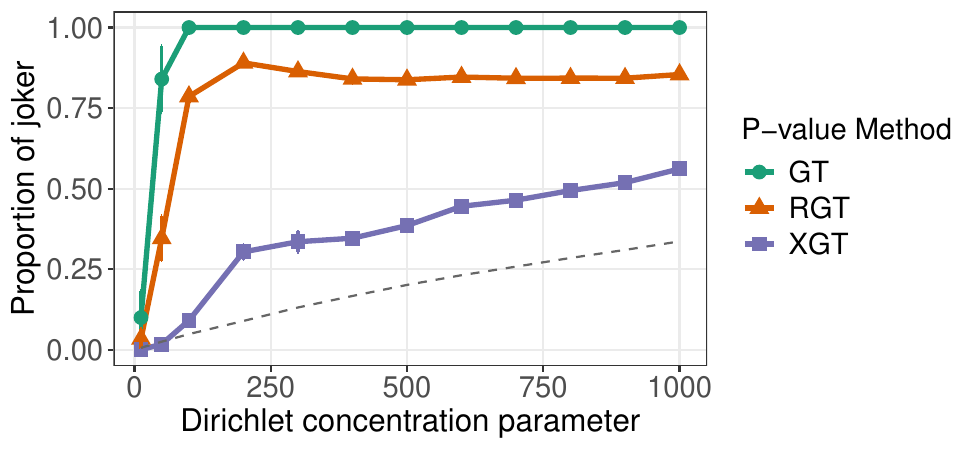}
    \caption{Proportion of CGTC prediction sets containing the joker on synthetic data, with each of the three $p$-values for testing $H_{\text{unseen}}$. Other details are as in Figure~\ref{fig:app-dp-orig-four-panel}.}
    \label{fig:app-dp-orig-pvalues}
\end{figure}

\FloatBarrier

\subsection{Data-Driven Allocation of the Significance Budget}
\label{app:alpha-allocation-results}

Figure~\ref{fig:app-dp-alpha-allocation} shows the significance levels deployed by CGTC+ in the synthetic experiments, after the missing-mass adjustment of the cross-validated base allocation. As $\theta$ grows and new labels become more frequent, more of the budget is allocated to the test of $H_{\text{unseen}}$.

\begin{figure}[!htb]
    \centering
    \includegraphics[width=0.5\textwidth]{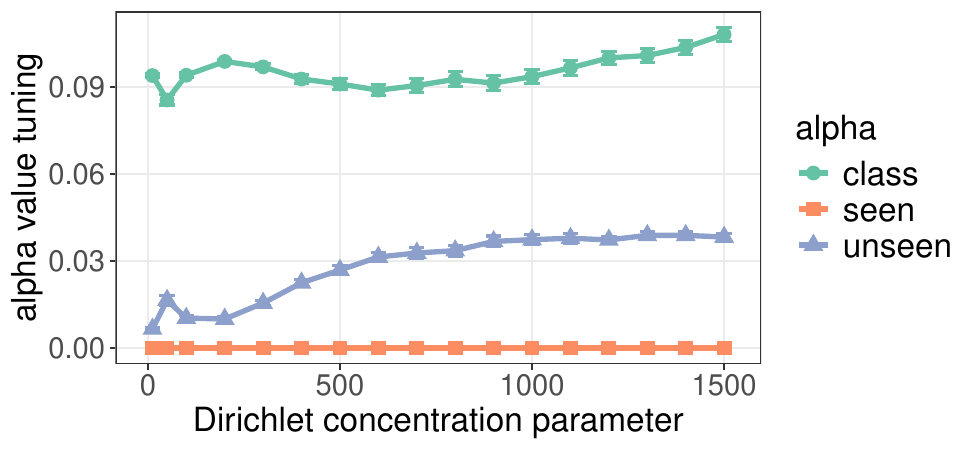}
    \caption{Significance levels $\alpha_{\text{class}}$, $\alpha_{\text{unseen}}$, and $\alpha_{\text{seen}}$ deployed by CGTC+ on synthetic data, as a function of $\theta$. Other details are as in Figure~\ref{fig:dp-main-four-panel}.}
    \label{fig:app-dp-alpha-allocation}
\end{figure}

\FloatBarrier

\subsection{Composition of the Very Rare Frequency Bin} \label{app:very-rare-composition}

This subsection explains why the coverage of CGTC+ drops on the very rare bin of Figure~\ref{fig:dp-main-conditional-coverage} at small values of $\theta$. This bin pools test labels that appear once in the reference sample, which are covered by the closed-set component, with test labels never observed, which are covered only if the prediction set includes the joker.

At small $\theta$, the missing mass is small, so $\widehat{\mu}_n = (M_1+1)/(n+1)$ falls below the base budget $\tilde\alpha_{\text{unseen}}$, and the adjusted level~\eqref{eq:plugin-adjustment} equals $\widehat{\mu}_n$. Since every unseen-label $p$-value satisfies $\psi_{\text{unseen}} \leq \widehat{\mu}_n$, the joker is then excluded from every prediction set, and unseen test labels are never covered. This does not contradict the validity of $\psi_{\text{unseen}}$, which bounds the probability that the test label is new and the joker is excluded, not this probability conditional on a new label. It is the intended behavior in the asymptotically closed-set regime of Section~\ref{sec:asymptotics}, where our method reduces to the standard one; indeed, at the smallest $\theta$, the two methods nearly coincide in the very rare panel.

Two facts explain the size of the drop and why it matters little. First, unseen labels outnumber singletons in the very rare bin at small $\theta$, by a factor of about $1.6$ at the smallest $\theta$ (Figure~\ref{fig:app-dp-freq0freq1}), so the coverage of this bin is approximately the singleton share times the nominal level. Second, the very rare bin itself holds only about $2\theta/(\theta+n)$ of the test points at small $\theta$, i.e., $1.2\%$, $4.9\%$, and $9.4\%$ at $\theta = 12$, $50$, and $100$ (Figure~\ref{fig:app-dp-bin-mass}). Even a complete loss of coverage on this bin at $\theta = 12$ would therefore reduce marginal coverage by only about one percentage point, consistent with Figure~\ref{fig:dp-main-four-panel}. At large $\theta$, instead, $\widehat{\mu}_n$ exceeds the base budget, the joker is included more often, and unseen labels are covered.

\begin{figure}[!htb]
    \centering
    \includegraphics[width=0.5\textwidth]{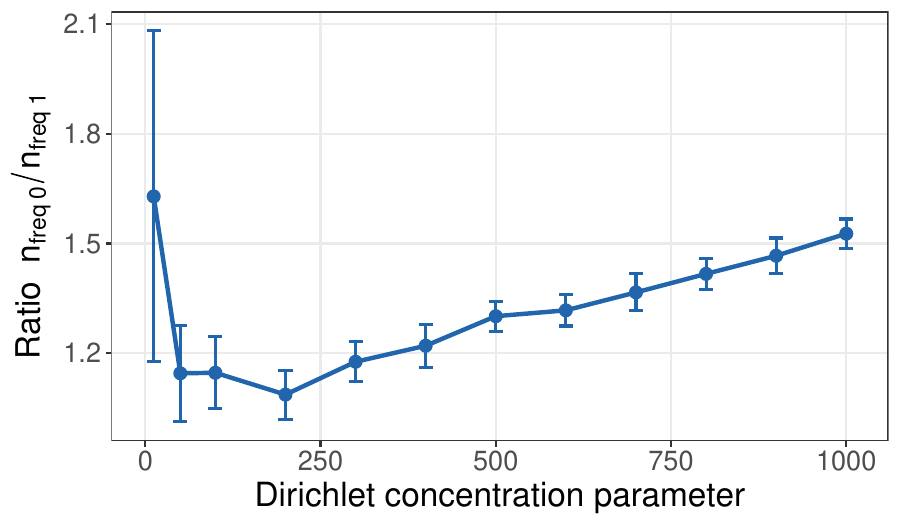}
    \caption{Ratio of unseen to singleton test labels in the very rare frequency bin of Figure~\ref{fig:dp-main-conditional-coverage}.}
    \label{fig:app-dp-freq0freq1}
\end{figure}

\begin{figure}[!htb]
    \centering
    \includegraphics[width=0.6\textwidth]{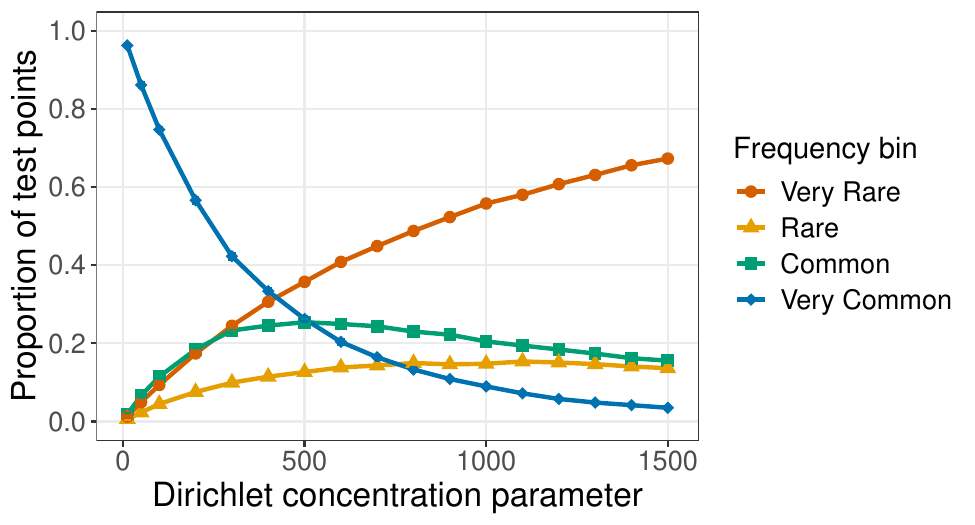}
    \caption{Proportion of test points in each of the four frequency bins of Figure~\ref{fig:dp-main-conditional-coverage}.}
    \label{fig:app-dp-bin-mass}
\end{figure}

\FloatBarrier

\subsection{Comparison of $p$-Values for Testing $H_{\text{unseen}}$}
\label{app:pvalue-comparison}

Figure~\ref{fig:app-dp-pvalue-full} repeats the synthetic experiments of Section~\ref{sec:synthetic-exp} separately for each of the three $p$-values for testing $H_{\text{unseen}}$. All three achieve valid coverage, but the feature-based $p$-value includes the joker least often, leading to the most informative prediction sets.

\begin{figure}[!htb]
    \centering
    \includegraphics[width=\textwidth]{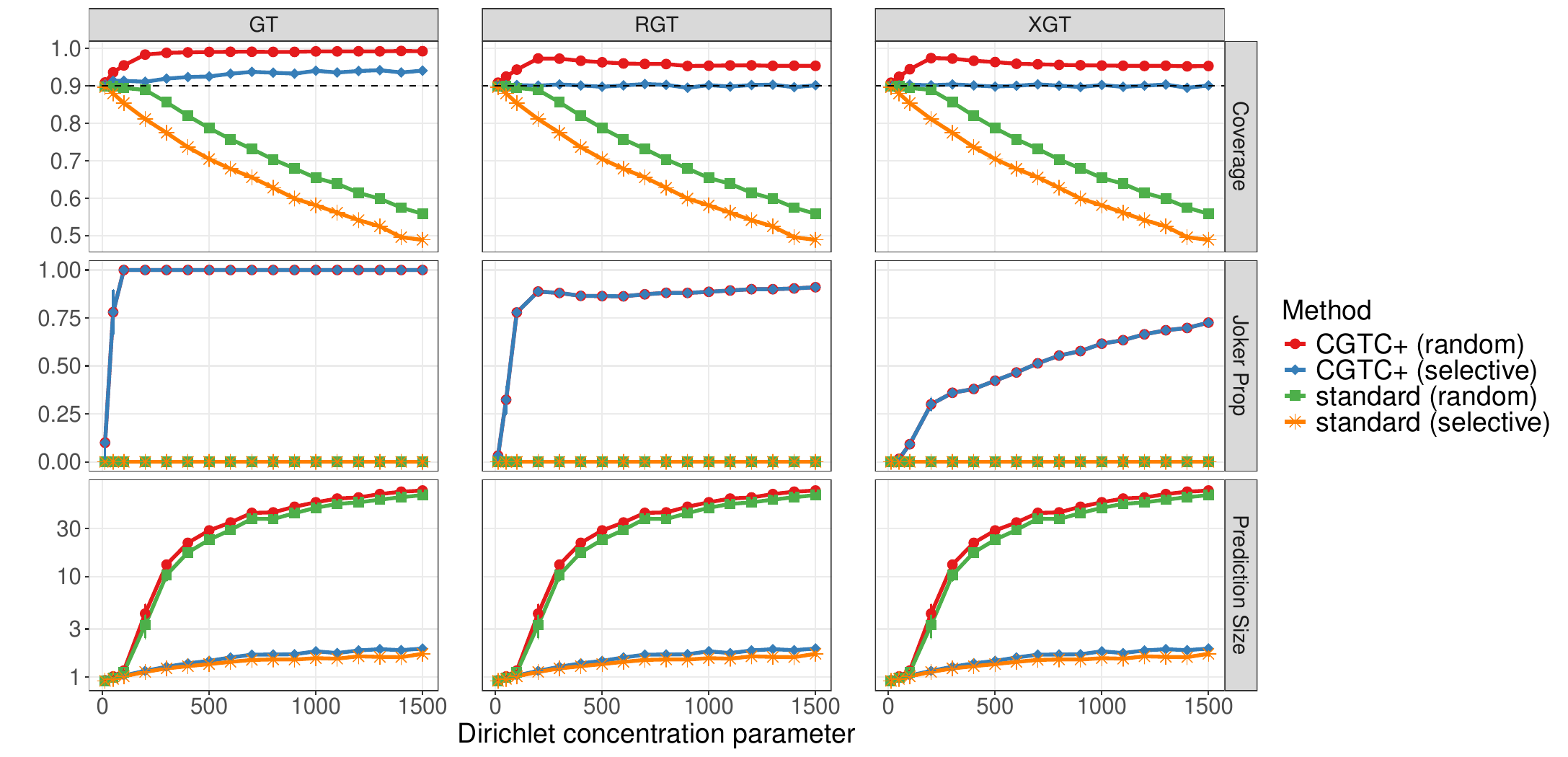}
    \caption{Same as Figure~\ref{fig:dp-main-four-panel}, but with each column corresponding to a different $p$-value for testing $H_{\text{unseen}}$: deterministic feature-blind (GT), randomized feature-blind (RGT), and feature-based (XGT).}
    \label{fig:app-dp-pvalue-full}
\end{figure}

\FloatBarrier

\subsection{Sensitivity to the Choice of One-Class Classifier}
\label{app:occ-sensitivity}

By Theorem~\ref{thm:XGT}, the validity of $\psi^{\text{XGT}}_k$ does not depend on the one-class classifier used to compute the score $\hat{s}_k$, so this choice affects only power. We repeat the synthetic experiments with CGTC, replacing the local outlier factor model with an isolation forest with $100$ trees (Figures~\ref{fig:app-dp-iforest-four-panel}--\ref{fig:app-dp-iforest-pjoker}), or with a one-class support vector machine with a radial basis function kernel, whose bandwidth is fitted from the training data (Figures~\ref{fig:app-dp-ocsvm-four-panel}--\ref{fig:app-dp-ocsvm-pjoker}). Both alternatives reproduce the results of Figures~\ref{fig:app-dp-orig-four-panel} and~\ref{fig:app-dp-orig-cond-cov}: the prediction sets keep valid coverage, including on the very rare labels where the standard method fails, with similar set sizes and joker frequencies. The conclusions of Section~\ref{sec:synthetic-exp} are therefore not specific to the local outlier factor model.

\begin{figure}[!htb]
    \centering
    \includegraphics[width=\textwidth]{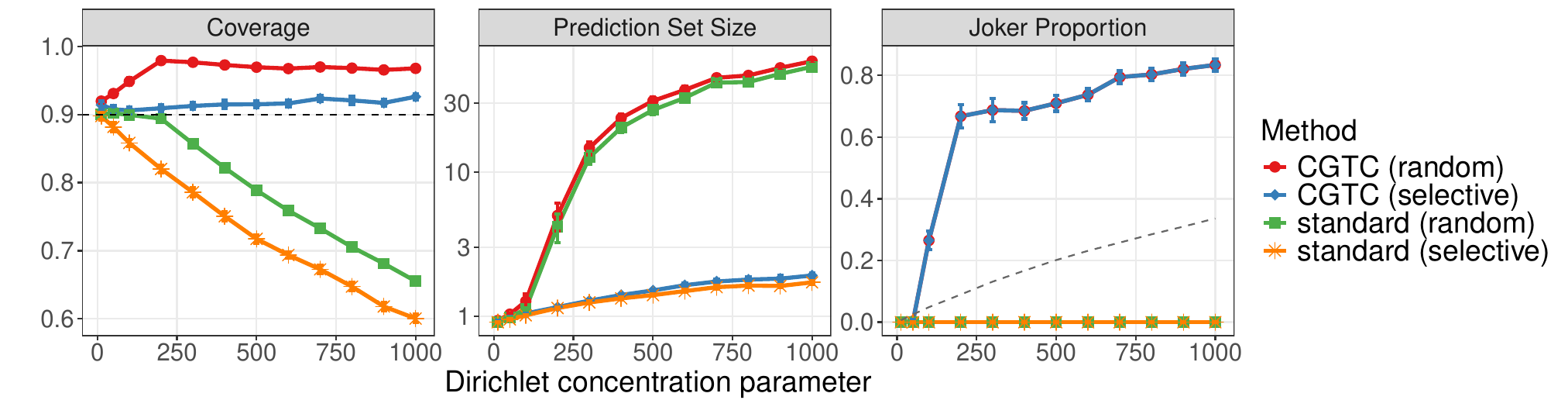}
    \caption{Same as Figure~\ref{fig:app-dp-orig-four-panel}, but with an isolation forest as the one-class classifier.}
    \label{fig:app-dp-iforest-four-panel}
\end{figure}

\begin{figure}[!htb]
    \centering
    \includegraphics[width=\textwidth]{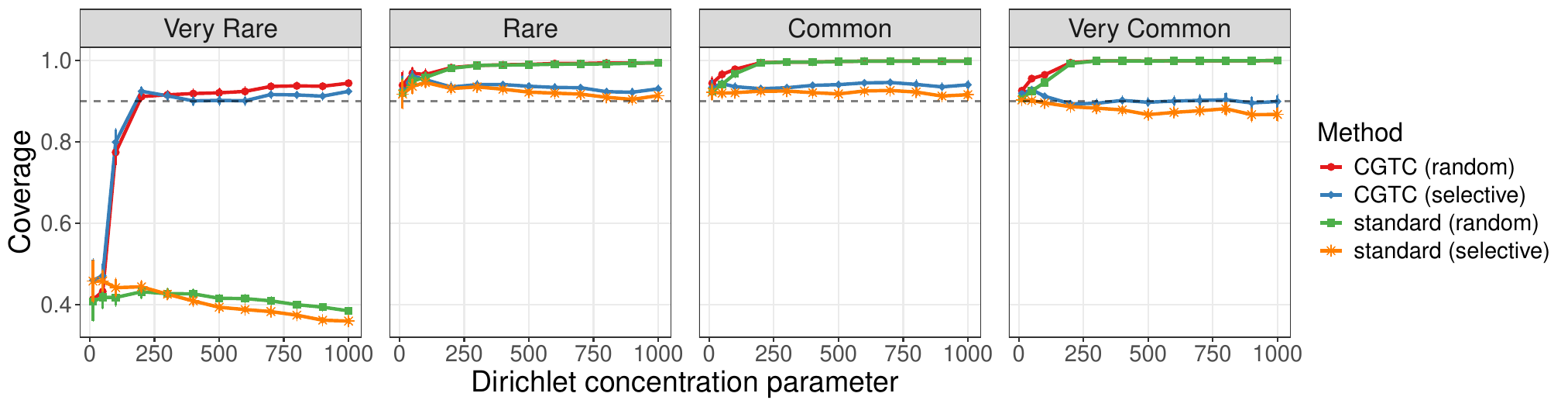}
    \caption{Same as Figure~\ref{fig:app-dp-orig-cond-cov}, but with an isolation forest as the one-class classifier.}
    \label{fig:app-dp-iforest-condcov}
\end{figure}

    \begin{figure}[!htb]
        \centering
        \includegraphics[width=0.5\textwidth]{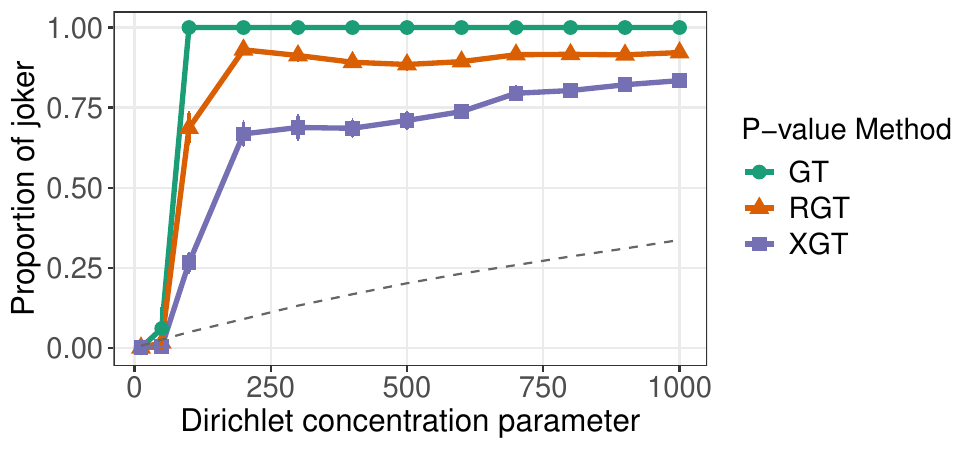}
        \caption{Same as Figure~\ref{fig:app-dp-orig-pvalues}, but with an isolation forest as the one-class classifier.}
        \label{fig:app-dp-iforest-pjoker}
    \end{figure}

\begin{figure}[!htb]
    \centering
    \includegraphics[width=\textwidth]{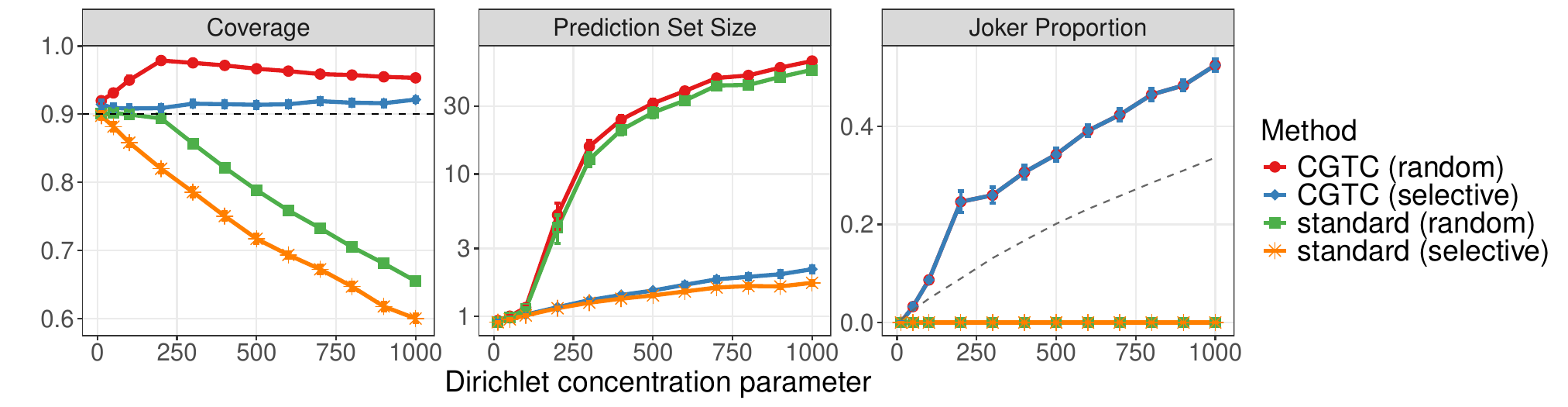}
    \caption{Same as Figure~\ref{fig:app-dp-orig-four-panel}, but with a one-class support vector machine as the one-class classifier.}
    \label{fig:app-dp-ocsvm-four-panel}
\end{figure}

\begin{figure}[!htb]
    \centering
    \includegraphics[width=\textwidth]{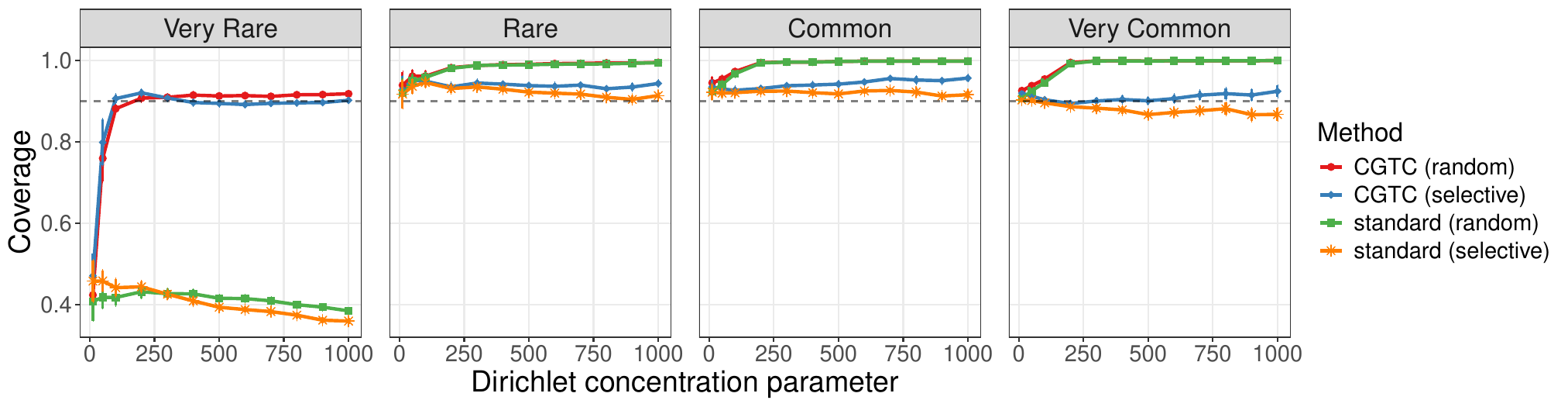}
    \caption{Same as Figure~\ref{fig:app-dp-orig-cond-cov}, but with a one-class support vector machine as the one-class classifier.}
    \label{fig:app-dp-ocsvm-condcov}
\end{figure}

    \begin{figure}[!htb]
        \centering
        \includegraphics[width=0.5\textwidth]{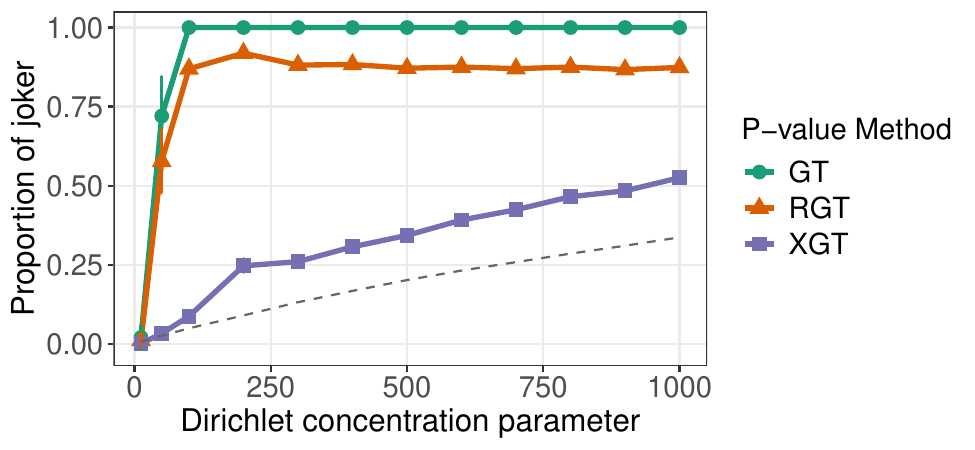}
        \caption{Same as Figure~\ref{fig:app-dp-orig-pvalues}, but with a one-class support vector machine as the one-class classifier.}
        \label{fig:app-dp-ocsvm-pjoker}
    \end{figure}

\FloatBarrier

\subsection{Sensitivity to the Reference Sample Size}
\label{app:reference-size}

Figures~\ref{fig:app-dp-four-panel-nref100}--\ref{fig:app-dp-propjoker-nref100} vary the reference sample size at $\theta = 100$. The conclusions of Section~\ref{sec:synthetic-exp} hold over the whole range: CGTC+ maintains valid coverage, including on rare labels (Figure~\ref{fig:app-dp-cond-cov-nref100}), and the joker becomes rarer as the reference sample grows and new labels become less frequent. As in Figure~\ref{fig:app-dp-pvalue-full}, the feature-based $p$-value includes the joker least often (Figure~\ref{fig:app-dp-propjoker-nref100}).

\begin{figure}[!htb]
    \centering
    \includegraphics[width=\textwidth]{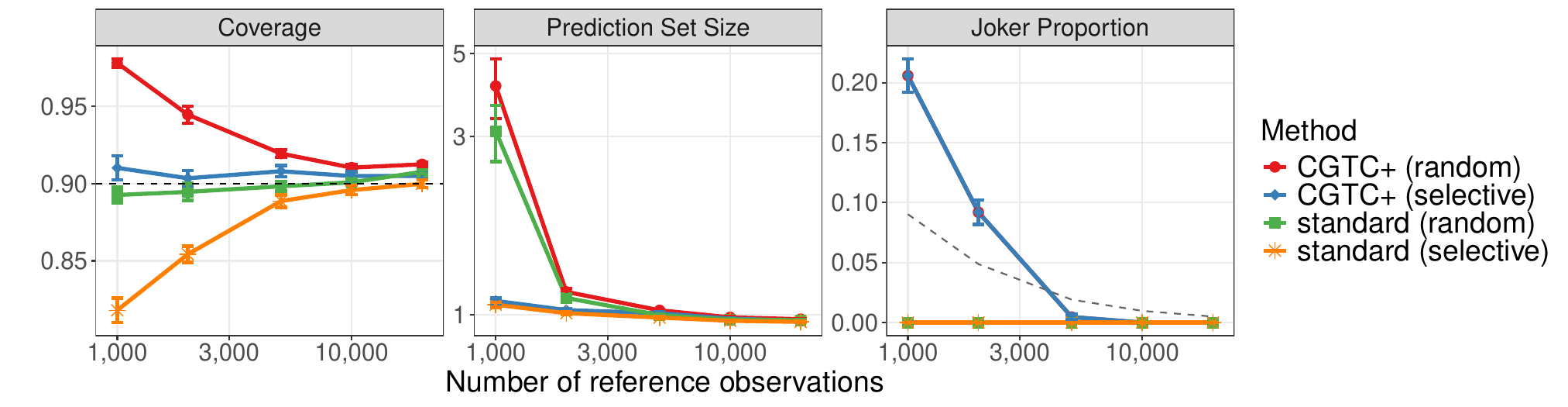}
    \caption{Same as Figure~\ref{fig:dp-main-four-panel}, but with $\theta = 100$ and varying reference sample size (log scale).}
    \label{fig:app-dp-four-panel-nref100}
\end{figure}

\begin{figure}[!htb]
    \centering
    \includegraphics[width=\textwidth]{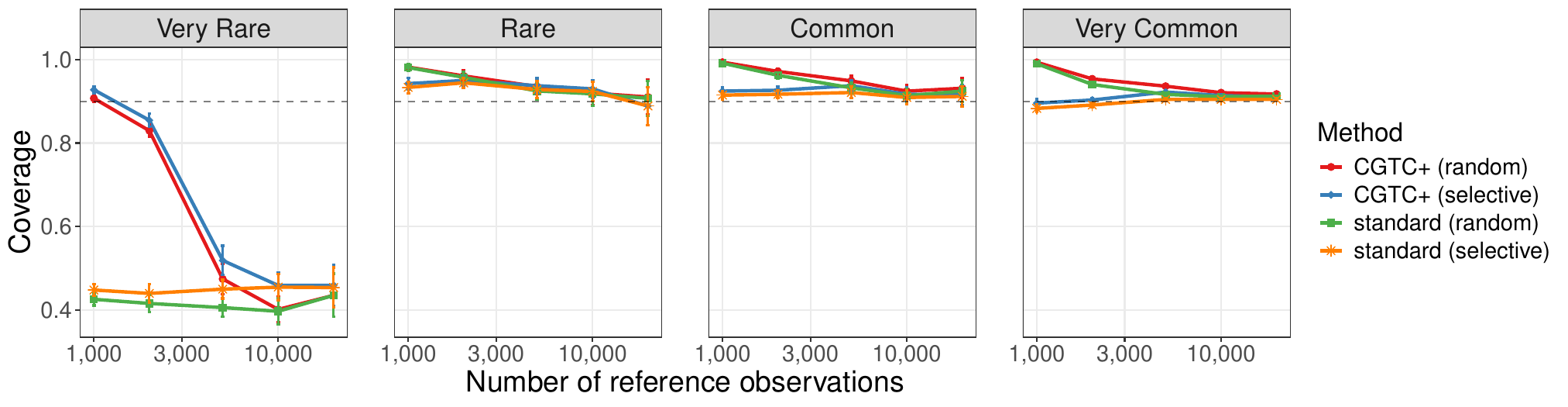}
    \caption{Coverage of the methods of Figure~\ref{fig:app-dp-four-panel-nref100}, stratified by test-label frequency.}
    \label{fig:app-dp-cond-cov-nref100}
\end{figure}

\begin{figure}[!htb]
    \centering
    \includegraphics[width=0.6\textwidth]{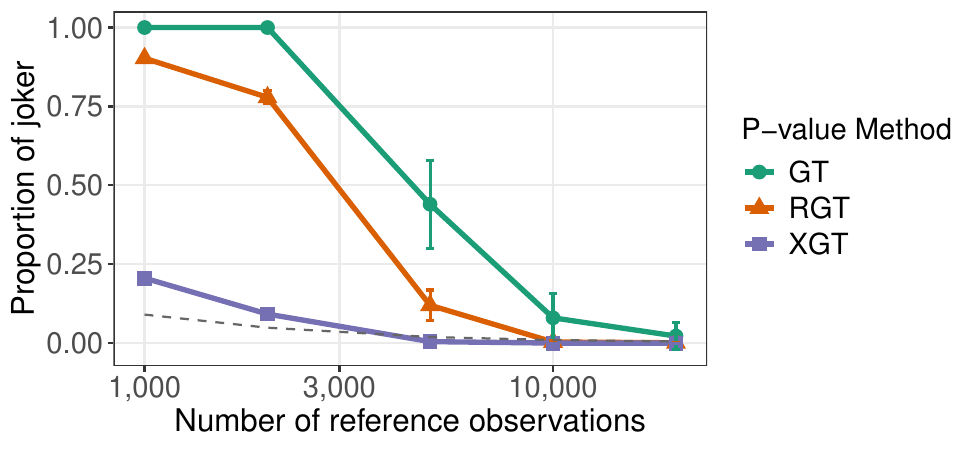}
    \caption{Proportion of CGTC+ prediction sets containing the joker, with each of the three $p$-values for testing $H_{\text{unseen}}$, in the setting of Figure~\ref{fig:app-dp-four-panel-nref100}.}
    \label{fig:app-dp-propjoker-nref100}
\end{figure}

\FloatBarrier

\subsection{Sensitivity to the Calibration Sample Size}
\label{app:calibration-size}

The synthetic experiments of Section~\ref{sec:synthetic-exp} use $10\%$ of the reference data for calibration. To examine the robustness of the results to this choice, we repeat the experiments with $\theta = 1000$ and $n = 2000$ for CGTC+, varying the calibration size from $100$ to $1000$, i.e., the calibration proportion from $5\%$ to $50\%$. Since $n$ is fixed, the training size decreases from $1900$ to $1000$ as the calibration size increases. Under selective splitting, the realized calibration sizes are very close to those of random splitting (Figure~\ref{fig:app-dp-vary-cal-mm}), as intended by the choice of $\pi$ in Section~\ref{sec:selective-splitting}.

Figures~\ref{fig:app-dp-vary-cal-mm} and~\ref{fig:app-dp-vary-cal-condcov-mm} show that the conclusions of Section~\ref{sec:synthetic-exp} hold at every calibration size. CGTC+ maintains valid marginal coverage, and covers the very rare labels at a rate between $0.88$ and $0.92$, against $0.33$ to $0.40$ for the standard method. With selective splitting, CGTC+ sets contain fewer than two seen labels on average at calibration sizes up to $400$, against $28$ to $113$ with random splitting; since set sizes count seen labels only, this gain reflects sharper inference on the seen labels, not a more frequent joker. A larger calibration share does not help: since $n$ is fixed, each additional calibration point is one fewer training point, so the classifier becomes less accurate and the sets of all methods grow with the calibration size. The selective variants also degrade at the largest proportions, where the weighted calibration loses half of its effective sample size (Appendix~\ref{app:ess-selective}).

A small calibration proportion suffices because calibration is a one-dimensional task: given the trained model, the coverage of a split conformal prediction set at miscoverage level $\alpha$, conditional on the calibration set, follows a Beta distribution with standard deviation about $\sqrt{\alpha(1-\alpha)/n^{\text{cal}}}$ \citep{vovk2012conditional}. Figure~\ref{fig:app-dp-cov-boxplot} is consistent with this: with selective splitting, the coverage of CGTC+ has mean between $0.899$ and $0.907$ at every calibration size, and its standard deviation falls only from $2.0$ to $1.2$ percentage points as the calibration size grows from $100$ to $1000$.
The same observations are far more valuable for training a classifier over a large and imbalanced label space. Results for CGTC are similar and are omitted.

\begin{figure}[!htb]
    \centering
    \includegraphics[width=\textwidth]{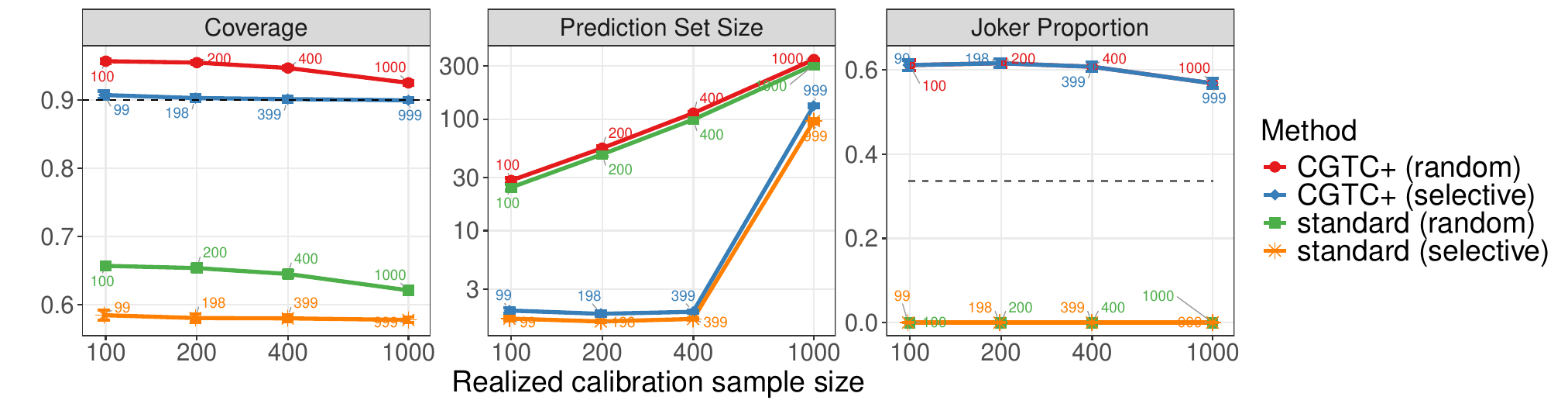}
    \caption{Performance of CGTC+ and the standard method on synthetic data, as a function of the calibration sample size, with $n = 2000$ fixed, so that the training sample size decreases as the calibration sample size increases; this explains why the prediction sets grow with the calibration size. The realized calibration sizes are annotated next to each curve, and are similar under selective and random splitting. Other details are as in Figure~\ref{fig:dp-main-four-panel}, with $\theta = 1000$.}
    \label{fig:app-dp-vary-cal-mm}
\end{figure}

\begin{figure}[!htb]
    \centering
    \includegraphics[width=\textwidth]{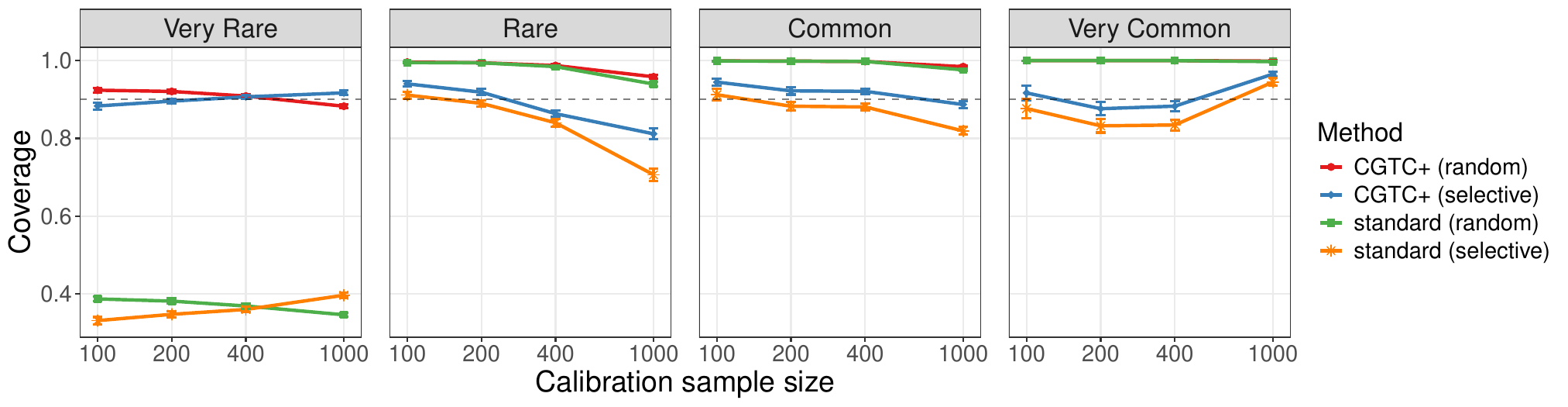}
    \caption{Coverage of the methods of Figure~\ref{fig:app-dp-vary-cal-mm}, stratified by test-label frequency.}
    \label{fig:app-dp-vary-cal-condcov-mm}
\end{figure}

\begin{figure}[!htb]
    \centering
    \includegraphics[width=0.7\textwidth]{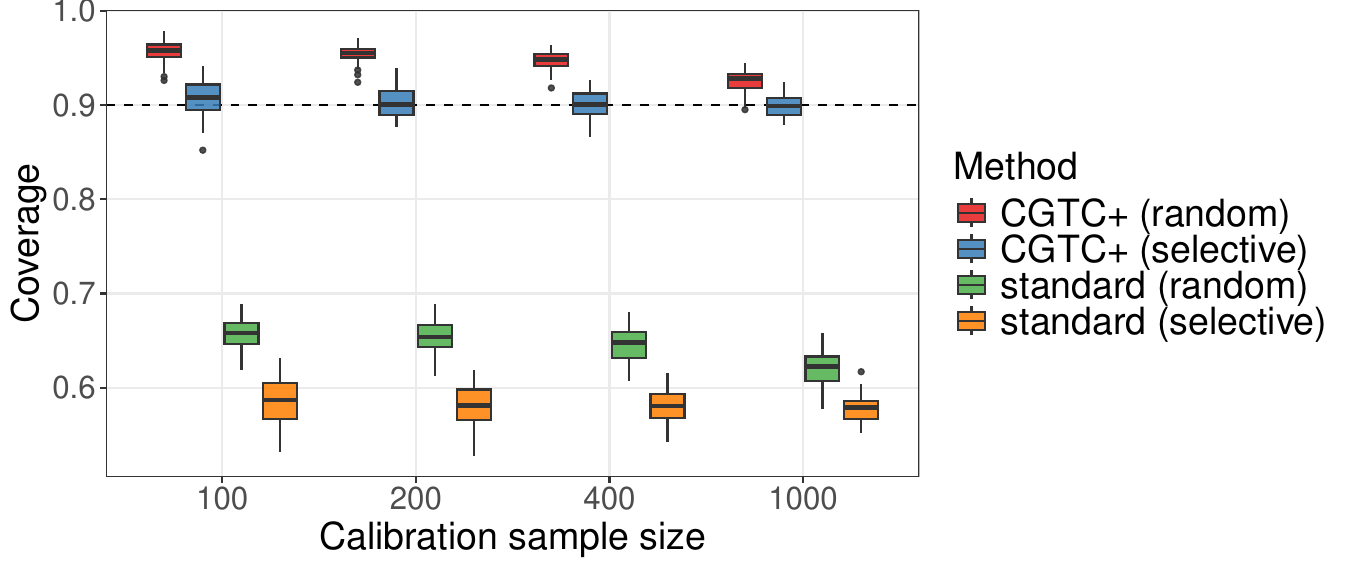}
    \caption{Distribution across $50$ repetitions of the marginal coverage of the methods of Figure~\ref{fig:app-dp-vary-cal-mm}, evaluated on $1{,}000$ test points, as a function of the calibration sample size.}
    \label{fig:app-dp-cov-boxplot}
\end{figure}

\FloatBarrier

\subsection{Additional Results on CelebA Data}
\label{app:real-supplements}

Figures~\ref{fig:real-app-conditional-coverage}--\ref{fig:real-prop-joker-mixed-labels-80} complement the CelebA experiments of Section~\ref{sec:real-exp}, and confirm the findings on synthetic data: CGTC+ remains valid across label frequencies, the selected allocation assigns more of the budget to the test of $H_{\text{unseen}}$ when new labels are more frequent, and the feature-based $p$-value includes the joker least often.

\begin{figure}[!htb]
    \centering
    \includegraphics[width=\textwidth]{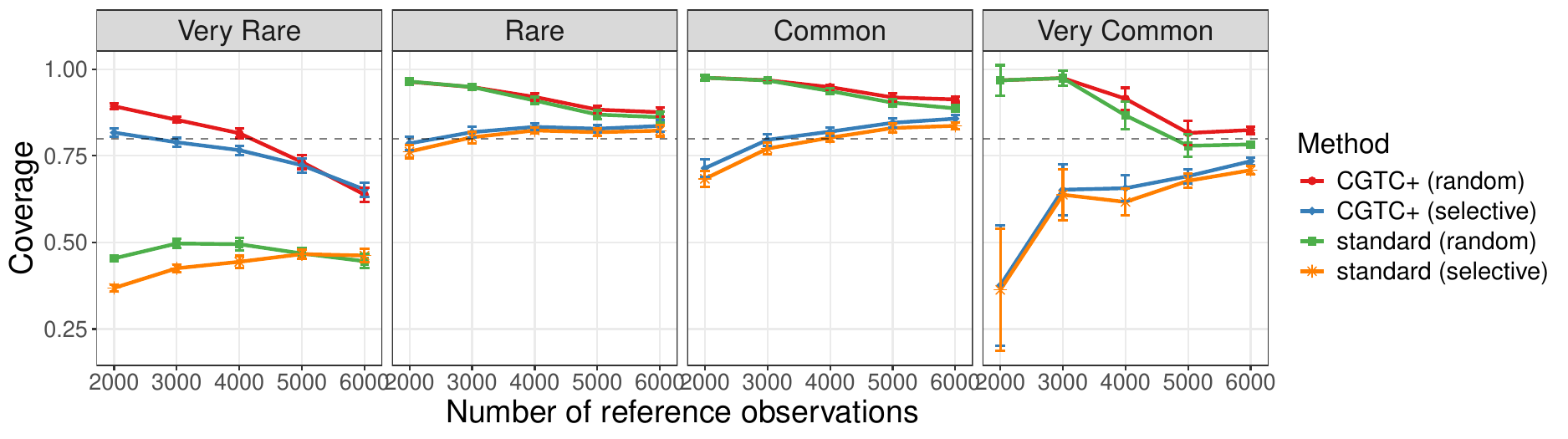}
    \caption{Coverage of the methods of Figure~\ref{fig:real-main-four-panel}, stratified by test-label frequency.}
    \label{fig:real-app-conditional-coverage}
\end{figure}

\begin{figure}[!htb]
    \centering
    \includegraphics[width=0.6\textwidth]{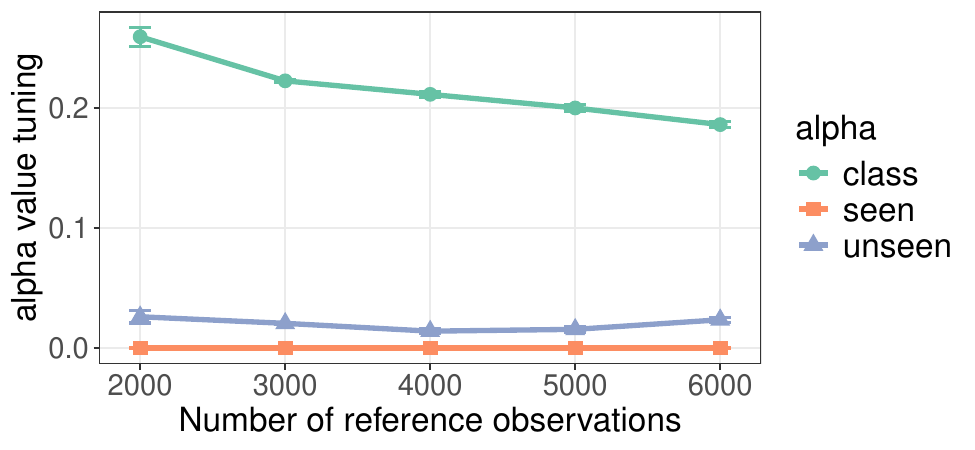}
    \caption{Significance levels $\alpha_{\text{class}}$, $\alpha_{\text{unseen}}$, and $\alpha_{\text{seen}}$ deployed by CGTC+ on CelebA data, as a function of the reference sample size. Other details are as in Figure~\ref{fig:real-main-four-panel}.}
    \label{fig:real-app-alpha-allocation}
\end{figure}

\begin{figure}[!htb]
    \centering
    \includegraphics[width=\textwidth]{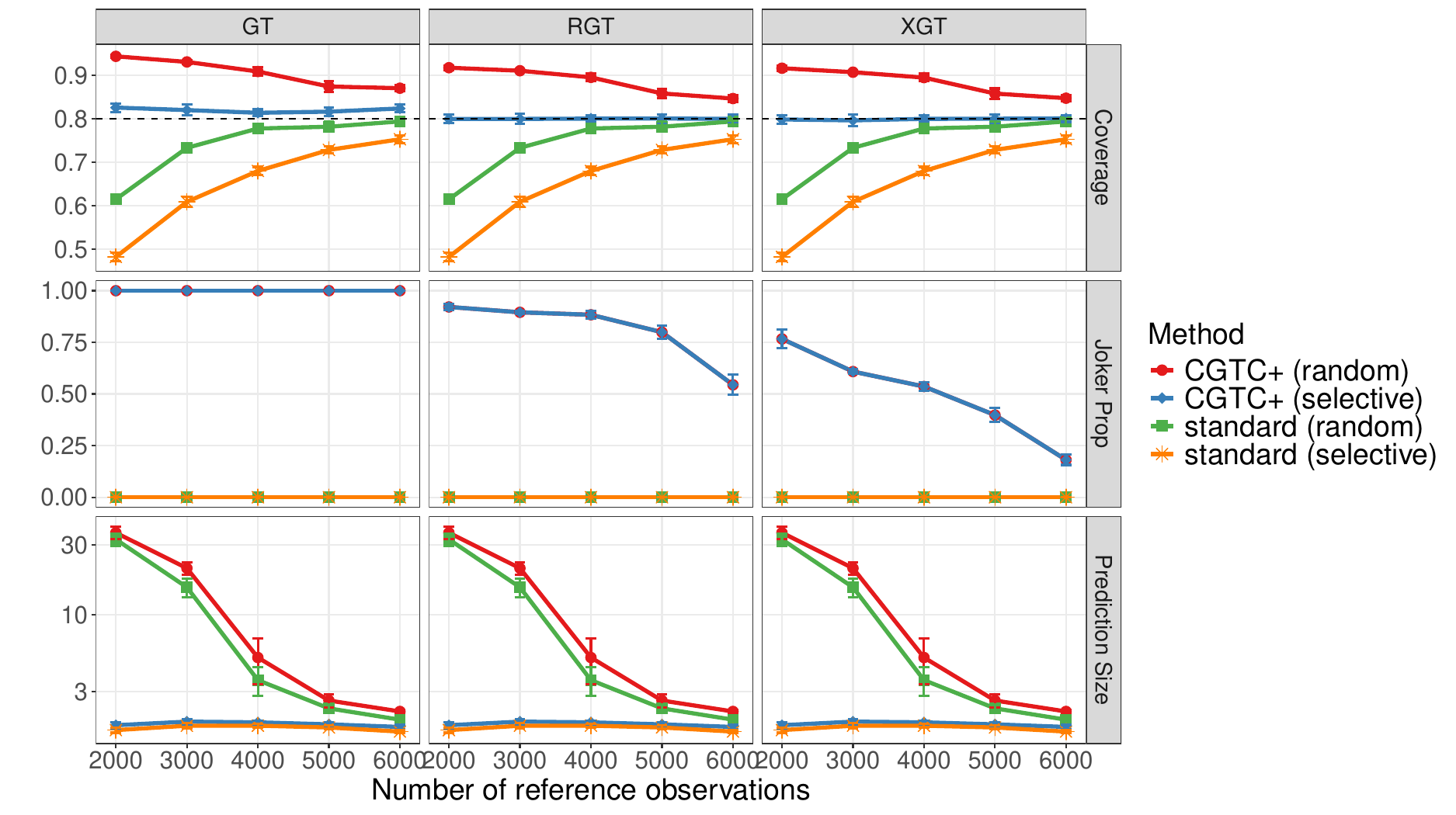}
    \caption{Same as Figure~\ref{fig:real-main-four-panel}, but with each column corresponding to a different $p$-value for testing $H_{\text{unseen}}$: deterministic feature-blind (GT), randomized feature-blind (RGT), and feature-based (XGT).}
    \label{fig:real-prop-joker-mixed-labels-80}
\end{figure}

\FloatBarrier

\subsection{Results for CGTC on CelebA Data}
\label{app:real-cgtc}

Figures~\ref{fig:app-real-orig-four-panel} and~\ref{fig:app-real-orig-cond-cov} repeat the CelebA experiments of Section~\ref{sec:real-exp} for CGTC. As on synthetic data, CGTC remains valid, overall and across label frequencies, and benefits substantially from selective splitting, but it is more conservative than CGTC+.

\begin{figure}[!htb]
    \centering
    \includegraphics[width=\textwidth]{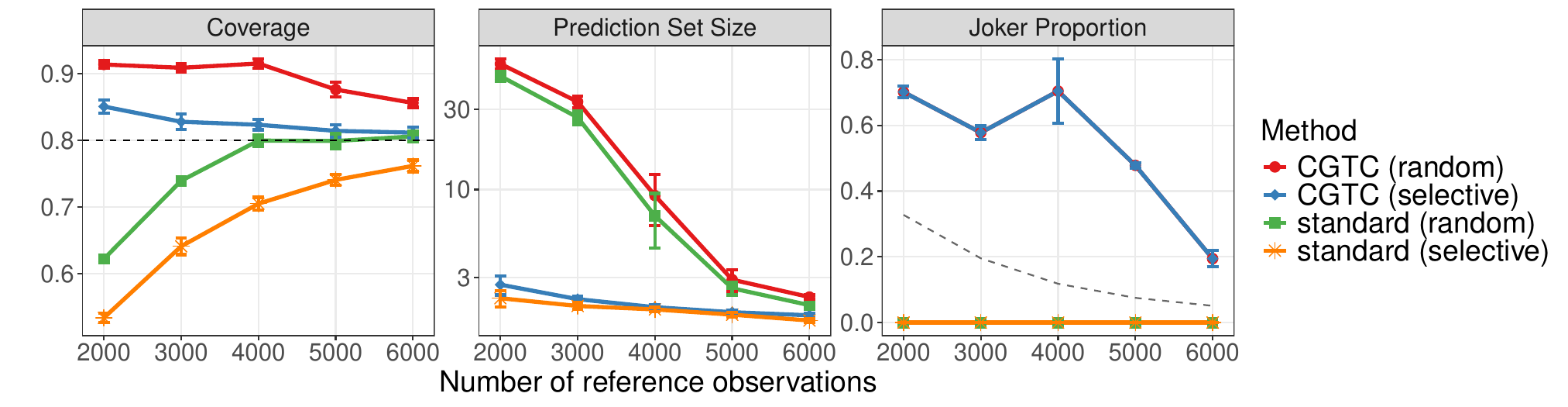}
    \caption{Performance of CGTC on CelebA data. Other details are as in Figure~\ref{fig:real-main-four-panel}, which shows the same results for CGTC+.}
    \label{fig:app-real-orig-four-panel}
\end{figure}

\begin{figure}[!htb]
    \centering
    \includegraphics[width=\textwidth]{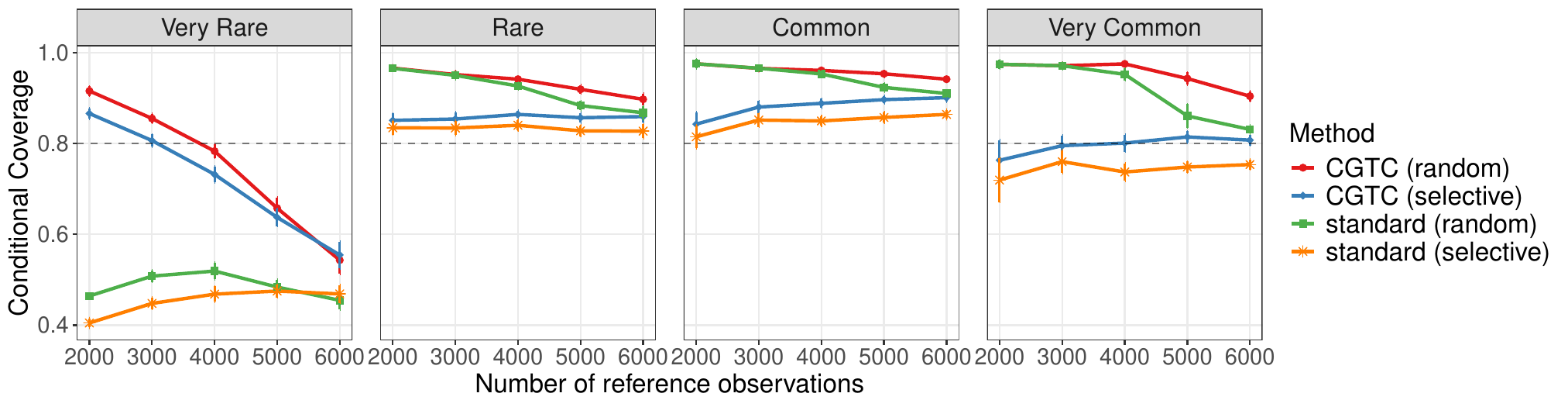}
    \caption{Coverage of CGTC on CelebA data, stratified by test-label frequency. Other details are as in Figure~\ref{fig:real-app-conditional-coverage}.}
    \label{fig:app-real-orig-cond-cov}
\end{figure}

\FloatBarrier

\subsection{Sensitivity to the Number of Identities in CelebA Data}
\label{app:real-nlab}

Figure~\ref{fig:app-real-four-panel-nlab} repeats the CelebA experiment for CGTC with the reference sample size fixed at $2{,}000$, varying the total number of identities from which the data are sampled. More identities mean more frequent new labels, and the conclusions of Section~\ref{sec:real-exp} hold throughout: CGTC remains valid, unlike the standard method, and benefits substantially from selective splitting.

\begin{figure}[!htb]
    \centering
    \includegraphics[width=\textwidth]{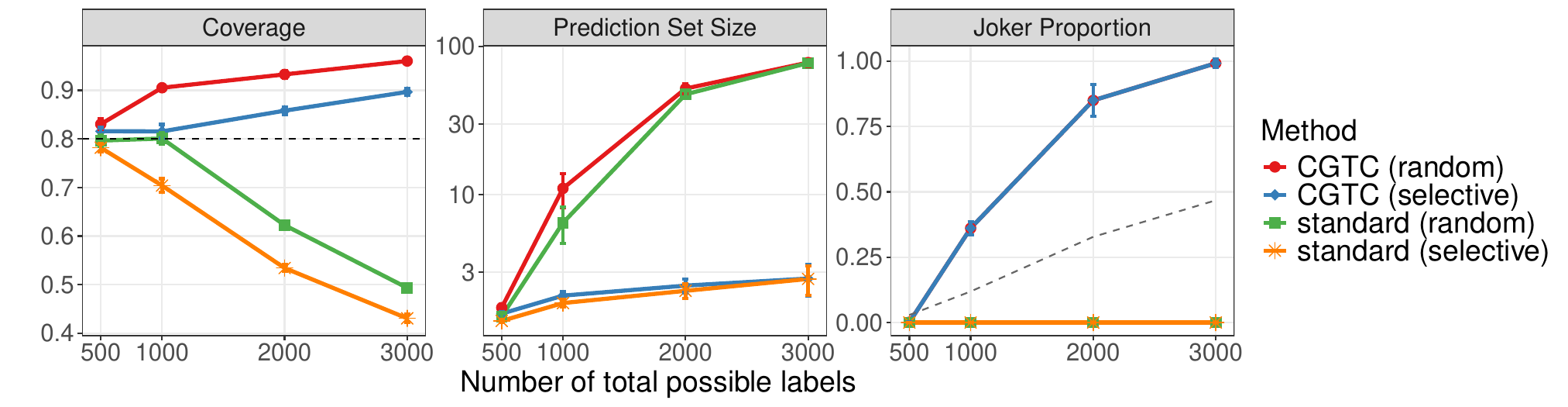}
    \caption{Performance of CGTC and the standard method on CelebA data, as a function of the total number of identities, with the reference sample size fixed at $2{,}000$. Other details are as in Figure~\ref{fig:real-main-four-panel}.}
    \label{fig:app-real-four-panel-nlab}
\end{figure}

\FloatBarrier

\subsection{Sensitivity to the Preference Parameter $\lambda$ on CelebA Data}
\label{app:lambda-sensitivity}

The loss~\eqref{eq:alpha_allocation_loss} weighs, through $\lambda \in [0,1]$, the number of seen labels in the prediction set against the cost of including the joker when the test label has been seen; all experiments use $\lambda = 0.5$. Figure~\ref{fig:app-lambda-tradeoff} repeats the CelebA experiment with CGTC+, selective splitting, and XGT $p$-values, for $\lambda \in \{0.3, 0.5, 0.7\}$. Marginal coverage stays at the nominal level for every $\lambda$, while larger values of $\lambda$ give fewer seen labels but more frequent jokers, and smaller values the reverse.

\begin{figure}[!htb]
    \centering
    \includegraphics[width=\textwidth]{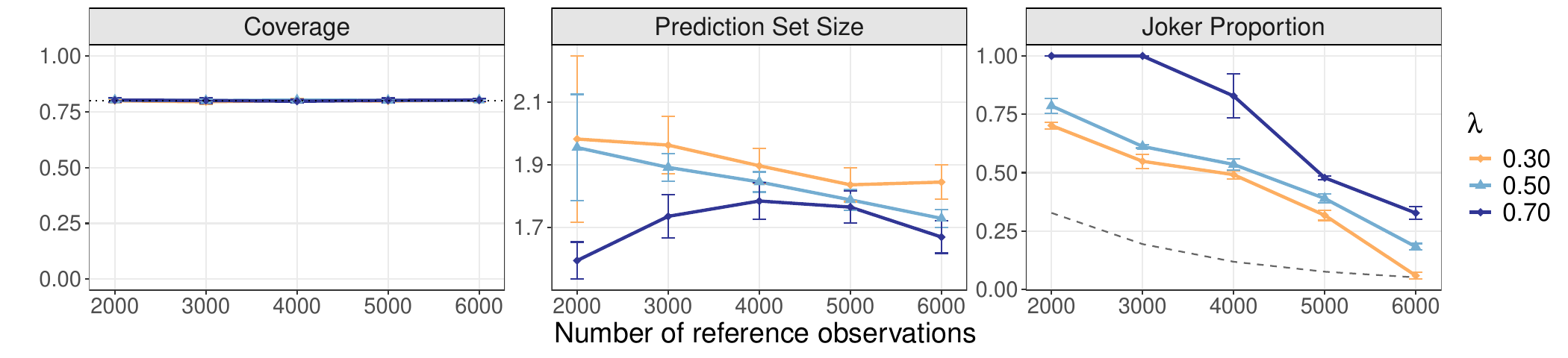}
    \caption{Coverage, number of seen labels, and proportion of prediction sets containing the joker on CelebA, for CGTC+ with different values of $\lambda$. Other details are as in Figure~\ref{fig:real-main-four-panel}.}
    \label{fig:app-lambda-tradeoff}
\end{figure}

\FloatBarrier

\section{Numerical Comparison with Open-Set Classifiers} \label{app:openset-findings}

We compare CGTC+, with selective splitting and XGT $p$-values, with the direct conformalization of open-set classifiers described in Appendix~\ref{sec:why-not-open-set-classifier}, on the CelebA experiment of Section~\ref{sec:real-exp} ($2{,}000$ identities, $\alpha = 0.2$, reference sample sizes from $2{,}000$ to $6{,}000$). For each method, we report four quantities: marginal coverage; the nominal prediction set size excluding the joker; the decoded set size $|\mathrm{dec}(\hat{C})|$ defined in~\eqref{eq:decode-map}, which for the direct conformalization approaches replaces the unknown class with the calibration-only identities it stands for, and equals the nominal size for CGTC+; and the proportion of sets containing the joker or unknown class.

The benchmarks combine a novelty score with the $k$-nearest neighbors classifier of Appendix~\ref{app:classification-model-adjustment}, except for OpenMax and PROSER, which use their own classifiers. Curve labels name the score and its treatment: KNN-dist ($k$) is the distance to the $k$-th nearest training image, KNN-MSP is one minus the maximum estimated probability, OCC denotes a one-class classifier, and Naive is the constant unknown probability given by the Good--Turing estimate. Scores are used raw or recalibrated as in Appendix~\ref{app:gt-recalibration}, with the prefix Recal for the multiplicative and Iso for the isotonic recalibration.

For simplified OpenMax \citep{Bendale2016}, we train a multilayer perceptron on the FaceNet embeddings and fit Weibull distributions to the upper tails of within-identity distances from the mean hidden activation. At prediction time, the fitted cumulative distribution functions determine the fraction of each identity's probability transferred to the unknown class. The variant labeled faithful OpenMax combines Euclidean and cosine distances and applies the reference revision rule to the ten highest output logits, with rank-dependent weights, before softmax normalization; distances are still computed in hidden-activation space. OpenMax-KNN applies the simplified probability revision to the $k$-nearest neighbors classifier, using identity centroids in the FaceNet embedding space. All three variants use SciPy Weibull fits, with a pooled fit when class-specific fitting is unavailable.

Our implementation of PROSER \citep{zhou2021learning} uses a multilayer perceptron on the FaceNet embeddings with five placeholder logits, whose maximum represents the unknown class before softmax normalization. The training losses encourage this maximum to rank second on training examples and first on mixtures of hidden representations from different identities. The centered local outlier factor uses the same one-neighbor LOF configuration as CGTC+, fitted to the benchmark training split. Writing $L(x)$ for its LOF value, the unknown-class score is $\min\{1,\max\{0,(L(x)-1)/2\}\}$. The one-class support vector machine (OCSVM) uses a Gaussian kernel with $\gamma=20$ and $\nu=0.05$. Its novelty score is the negative logarithm of the normalized kernel expansion (divided by the sum of dual coefficients), rescaled and clipped to $[0,1]$ using the first and ninety-ninth percentiles of scores from an internal holdout split.

\emph{Main comparison.} Figure~\ref{fig:app-openset-real-main} compares CGTC+ with the naive baseline and the two strongest direct conformalization approaches: Recal-OCC (LOF), which uses the same LOF configuration as the XGT $p$-values of CGTC+, and Iso-KNN-dist ($k{=}1$), which recalibrates the score with the best ranking. Table~\ref{tab:app-openset-benchmarks} summarizes all methods at the smallest and largest reference sample sizes. The naive baseline includes the unknown class in nearly every set when novelty is common, with decoded sets of up to $69$ labels. The two benchmarks do better, with decoded sets falling from about $30$ labels at $2{,}000$ reference images to under $7$ at $6{,}000$, but these remain $4$ to $18$ times larger than the $1.8$ labels of CGTC+, because every set containing the unknown class implicitly lists all calibration-only identities. CGTC+ includes the joker more often, keeping it whenever novelty cannot be rejected at the guaranteed level, although at the largest reference size the best benchmark covers more novel test points ($0.83$ against $0.55$); neither method guarantees coverage conditional on a new label.

\begin{figure}[!htb]
    \centering
    \includegraphics[width=\textwidth]{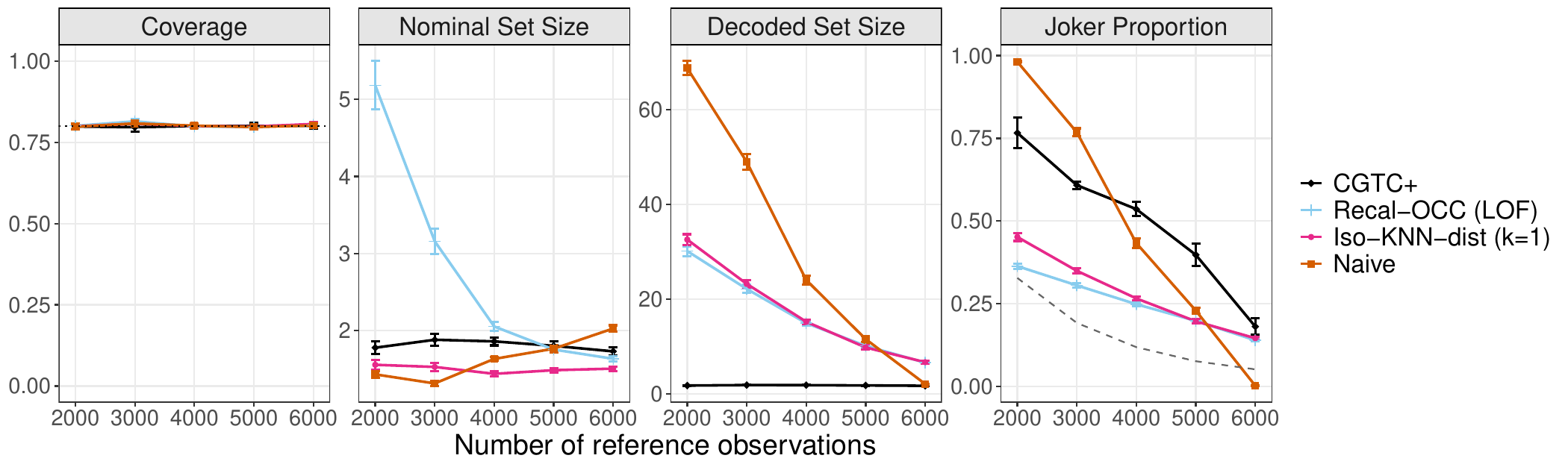}
    \caption{Comparison on CelebA data of CGTC+ with the naive baseline and the two strongest direct conformalizations, as a function of the reference sample size. From left to right: marginal coverage, with the dotted line at the nominal level; nominal set size; decoded set size; and proportion of sets containing the joker or unknown class, with the dashed line at the true fraction of test identities absent from the reference sample. Error bars represent $\pm 1.96$ standard errors.}
    \label{fig:app-openset-real-main}
\end{figure}

\begin{table}[!htb]
\centering
\caption{Summary of all methods on CelebA at the smallest and largest reference sample sizes. \emph{Cov.}, \emph{Size}, and \emph{Joker}: marginal coverage, decoded set size, and proportion of sets containing the joker or unknown class, as in Figure~\ref{fig:app-openset-real-main}. \emph{Uns.\ Cov.}: coverage conditional on an unseen label, i.e., a label outside the training set for the benchmarks and outside the reference sample for CGTC+. The bottom rows report the corresponding true novelty rates. Monte Carlo standard errors are at most $0.03$ for proportions and $67$ for set sizes.}
\label{tab:app-openset-benchmarks}
\small
\setlength{\tabcolsep}{4.5pt}
\begin{tabular}{lcccccccc}
\toprule
& \multicolumn{4}{c}{$n_{\mathrm{ref}} = 2{,}000$} & \multicolumn{4}{c}{$n_{\mathrm{ref}} = 6{,}000$} \\
\cmidrule(lr){2-5} \cmidrule(lr){6-9}
Method & Cov. & Size & Joker & Uns.\ Cov. & Cov. & Size & Joker & Uns.\ Cov. \\
\midrule
CGTC+ & 0.80 & 1.8 & 0.77 & 0.92 & 0.80 & 1.7 & 0.18 & 0.55 \\
\midrule
OpenMax (simplified) & 0.80 & 70.0 & 0.99 & 1.00 & 0.80 & 66.2 & 1.00 & 1.00 \\
OpenMax (faithful) & 1.00 & 1199.0 & 1.00 & 1.00 & 0.80 & 19.6 & 0.00 & 0.00 \\
OpenMax-KNN & 0.79 & 69.1 & 1.00 & 1.00 & 0.81 & 176.6 & 1.00 & 1.00 \\
KNN-dist ($k{=}10$, raw) & 0.80 & 65.3 & 0.93 & 1.00 & 0.80 & 22.3 & 0.59 & 0.95 \\
KNN-dist ($k{=}1$, raw) & 0.80 & 44.1 & 0.63 & 0.95 & 0.81 & 12.5 & 0.31 & 0.90 \\
KNN-MSP (raw) & 0.80 & 69.6 & 0.99 & 0.99 & 0.81 & 36.1 & 0.97 & 0.99 \\
PROSER (raw) & 0.80 & 70.3 & 0.98 & 1.00 & 0.80 & 37.0 & 0.97 & 1.00 \\
OCC (LOF, raw) & 0.80 & 29.1 & 0.30 & 0.55 & 0.80 & 6.4 & 0.13 & 0.60 \\
OCC (OCSVM $\gamma{=}20$, raw) & 0.80 & 62.8 & 0.89 & 0.86 & 0.80 & 32.6 & 0.87 & 0.87 \\
Naive (GT constant) & 0.80 & 68.8 & 0.98 & 1.00 & 0.80 & 2.1 & 0.00 & 0.00 \\
\midrule
Recal-KNN-dist ($k{=}10$) & 0.80 & 68.3 & 0.98 & 1.00 & 0.80 & 2.3 & 0.01 & 0.00 \\
Recal-KNN-MSP & 0.80 & 65.4 & 0.94 & 1.00 & 0.80 & 3.8 & 0.05 & 0.12 \\
Recal-PROSER & 0.79 & 59.5 & 0.82 & 0.96 & 0.80 & 18.7 & 0.46 & 0.86 \\
Recal-OCC (LOF) & 0.80 & 30.2 & 0.36 & 0.64 & 0.80 & 6.6 & 0.14 & 0.60 \\
Recal-OpenMax-KNN & 0.79 & 43.0 & 0.61 & 0.74 & 0.79 & 38.5 & 1.00 & 1.00 \\
\midrule
Iso-KNN-dist ($k{=}10$) & 0.80 & 66.1 & 0.94 & 0.98 & 0.80 & 4.3 & 0.07 & 0.21 \\
Iso-KNN-dist ($k{=}1$) & 0.80 & 32.6 & 0.45 & 0.80 & 0.81 & 6.7 & 0.15 & 0.83 \\
Iso-KNN-MSP & 0.80 & 54.6 & 0.77 & 0.91 & 0.80 & 7.5 & 0.16 & 0.53 \\
Iso-PROSER & 0.79 & 56.3 & 0.77 & 0.94 & 0.80 & 16.3 & 0.39 & 0.83 \\
Iso-OCC (LOF) & 0.79 & 64.5 & 0.92 & 0.92 & 0.80 & 6.3 & 0.13 & 0.59 \\
Iso-OCC (OCSVM $\gamma{=}20$) & 0.80 & 34.1 & 0.47 & 0.81 & 0.81 & 7.2 & 0.16 & 0.85 \\
Iso-OpenMax-KNN & 0.79 & 33.3 & 0.47 & 0.86 & 0.79 & 59.9 & 0.62 & 0.97 \\
\midrule
True novelty rate (training level) & & & 0.37 & & & & 0.07 & \\
True novelty rate (reference level) & & & 0.33 & & & & 0.05 & \\
\bottomrule
\end{tabular}

\end{table}

\FloatBarrier

\emph{Raw scores.} Figure~\ref{fig:app-openset-real-direct} shows that direct conformalization of off-the-shelf classifiers performs poorly, for two reasons. Most of these classifiers rank inputs by novelty reasonably well (Table~\ref{tab:app-openset-diagnostics}), but their estimated probabilities of new labels are miscalibrated, so some of them include the unknown class much more often than CGTC+ includes the joker, although a few sometimes include it less often. Moreover, regardless of how often it is used, the unknown class is less informative than our joker, because each time it is included it implicitly includes all calibration-only identities. As a result, the decoded prediction sets of all these classifiers are much larger than those of CGTC+.

\begin{figure}[!htb]
    \centering
    \includegraphics[width=\textwidth]{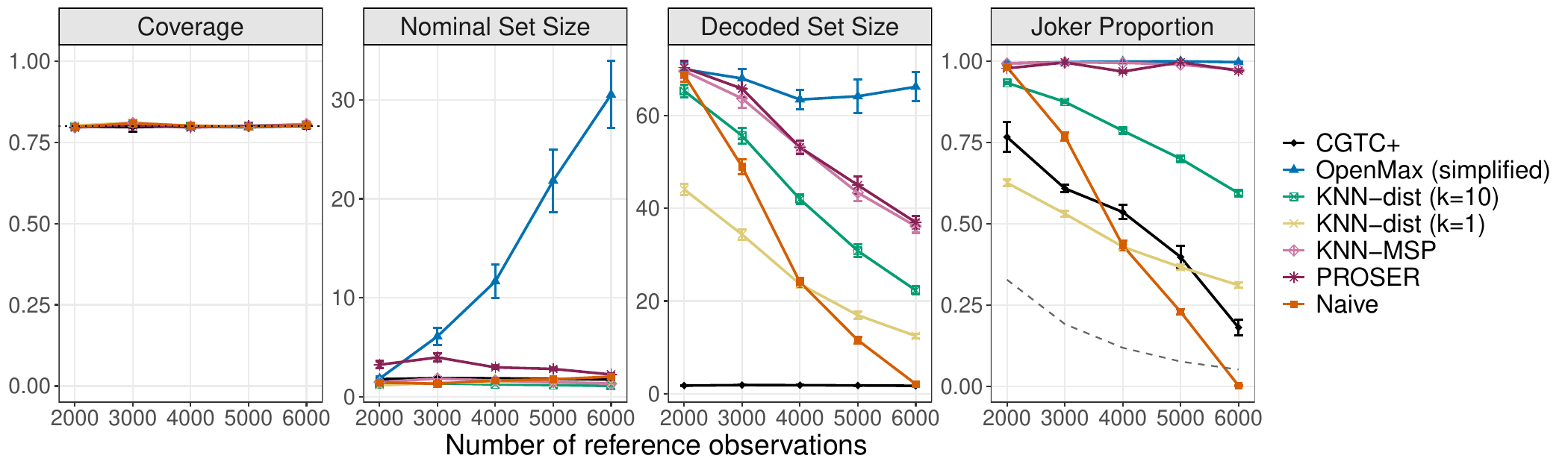}
    \caption{Same as Figure~\ref{fig:app-openset-real-main}, but for the direct conformalization of off-the-shelf classifiers using their raw unknown probabilities.}
    \label{fig:app-openset-real-direct}
\end{figure}

\FloatBarrier

\emph{Recalibrated scores.} Both recalibrations preserve the ranking of inputs by novelty, but they assign different values to the unknown probabilities. The multiplicative recalibration corrects their average, but preserves their ratios. For some novelty scores, such as one minus the maximum softmax probability, the score is well above zero for most inputs with seen labels, so seen and novel test points still receive similar unknown probabilities after recalibration (Table~\ref{tab:app-openset-diagnostics}), and these methods perform like the naive baseline (Figure~\ref{fig:app-openset-real-recal}). The isotonic recalibration avoids this problem, because it estimates the probability of a new label as a function of the score, and thus depends on the score only through its ranking. Its performance then depends on how well the score ranks inputs by novelty: the best rankings lead to the best direct conformalizations, while the weakest do not improve on the naive baseline (Figures~\ref{fig:app-openset-real-iso} and~\ref{fig:app-openset-real-occ}). Even the best of them, however, produce larger decoded prediction sets than CGTC+, because their unknown class still stands for all calibration-only identities.

\begin{figure}[!htb]
    \centering
    \includegraphics[width=\textwidth]{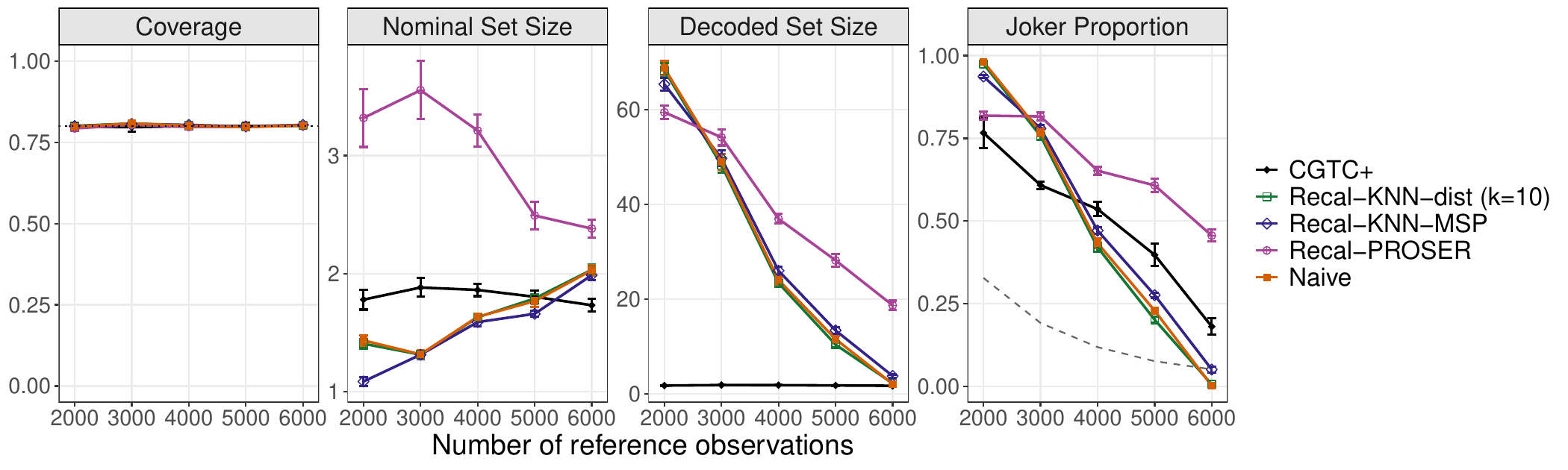}
    \caption{Same as Figure~\ref{fig:app-openset-real-main}, but for the multiplicative recalibration of the off-the-shelf scores.}
    \label{fig:app-openset-real-recal}
\end{figure}

\begin{figure}[!htb]
    \centering
    \includegraphics[width=\textwidth]{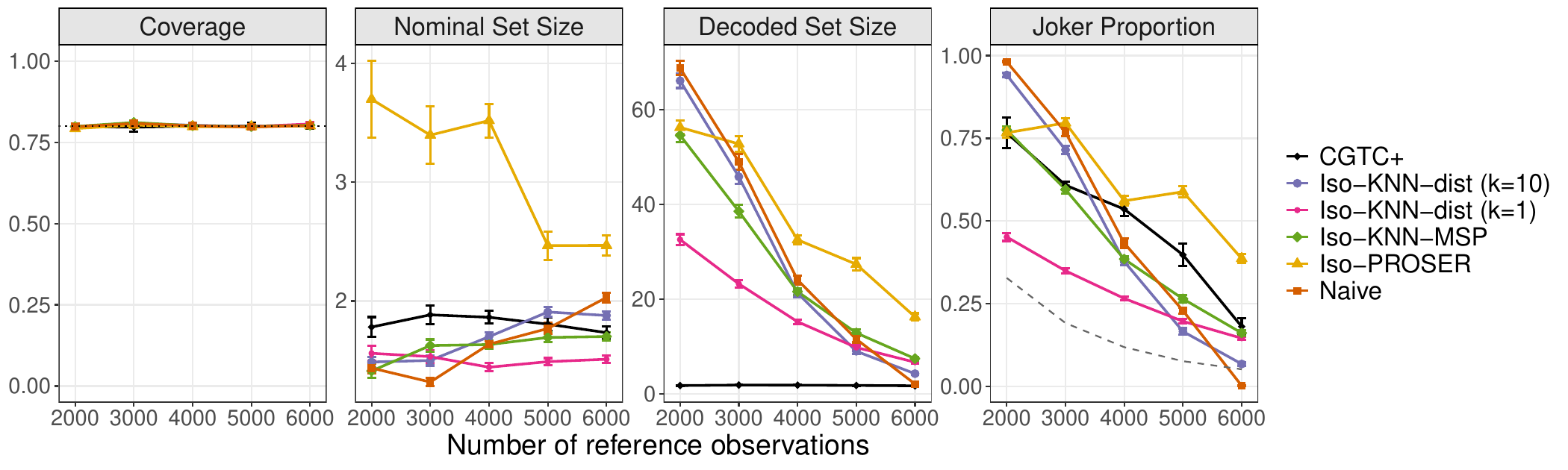}
    \caption{Same as Figure~\ref{fig:app-openset-real-main}, but for the isotonic recalibration of the off-the-shelf scores.}
    \label{fig:app-openset-real-iso}
\end{figure}

\begin{figure}[!htb]
    \centering
    \includegraphics[width=\textwidth]{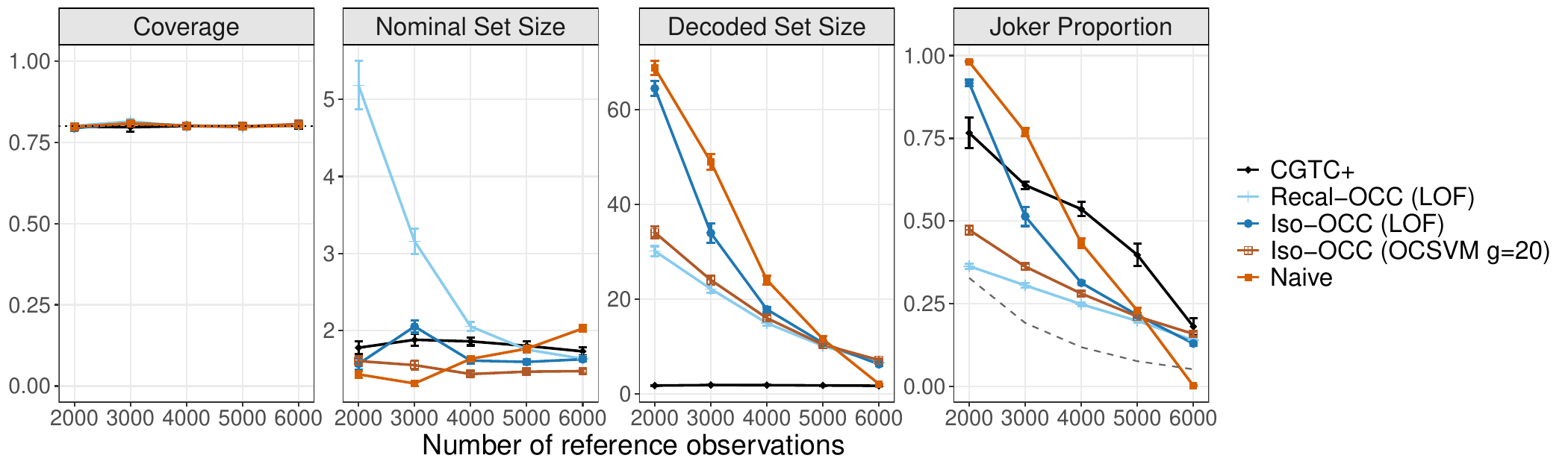}
    \caption{Same as Figure~\ref{fig:app-openset-real-main}, but for the recalibrated one-class classifiers.}
    \label{fig:app-openset-real-occ}
\end{figure}

\begin{table}[!htb]
\centering
\caption{Diagnostics of the raw unknown-class probabilities on CelebA ($\alpha = 0.2$, $2{,}000$ sampled identities, averages over the same batches as Table~\ref{tab:app-openset-benchmarks}) at the smallest and largest reference sample sizes. \emph{AUC}: area under the ROC curve of the unknown probability as a detector of novelty at the training level, that is, for separating test points whose identity is absent from the training split from those whose identity is present. \emph{$\bar{p}_{\mathrm{unk}}$}: average unknown probability over test points whose identity is present in the training split. \emph{Ratio}: average unknown probability over test points whose identity is absent from the entire reference sample, divided by the previous column; the multiplicative recalibration (Recal) preserves this ratio up to the effect of the cap and therefore shares the AUC and the ratio of its raw score, whereas the isotonic recalibration (Iso) shares only the AUC, up to ties. \emph{$p_{\mathrm{unk}} > p_{Y}$}: fraction of test points whose identity is present in the training split for which the unknown probability exceeds the probability of the true identity. \emph{Top-1}: accuracy of the seen-class component on test points whose identity is present in the training split. Monte Carlo standard errors are at most $0.01$ for the AUC and at most $0.01$ for $\bar{p}_{\mathrm{unk}}$.}
\label{tab:app-openset-diagnostics}
\small
\setlength{\tabcolsep}{3.5pt}
\begin{tabular}{lcccccccccc}
\toprule
& \multicolumn{5}{c}{$n_{\mathrm{ref}} = 2{,}000$} & \multicolumn{5}{c}{$n_{\mathrm{ref}} = 6{,}000$} \\
\cmidrule(lr){2-6} \cmidrule(lr){7-11}
Score & AUC & $\bar{p}_{\mathrm{unk}}$ & Ratio & $p_{\mathrm{unk}} > p_{Y}$ & Top-1 & AUC & $\bar{p}_{\mathrm{unk}}$ & Ratio & $p_{\mathrm{unk}} > p_{Y}$ & Top-1 \\
\midrule
OpenMax (simplified) & 0.74 & 0.69 & 1.4 & 0.71 & 0.65 & 0.61 & 0.63 & 1.3 & 0.67 & 0.64 \\
OpenMax (faithful) & 0.41 & 0.00 &  & 0.00 & 0.21 & 0.45 & 0.00 &  & 0.00 & 0.34 \\
OpenMax-KNN & 0.89 & 0.91 & 1.1 & 0.98 & 0.81 & 0.77 & 0.83 & 1.2 & 0.95 & 0.66 \\
KNN-dist ($k{=}10$) & 0.60 & 0.30 & 1.0 & 0.83 & 0.71 & 0.70 & 0.27 & 1.1 & 0.39 & 0.81 \\
KNN-dist ($k{=}1$) & 0.93 & 0.14 & 1.8 & 0.31 & 0.71 & 0.95 & 0.12 & 1.9 & 0.15 & 0.81 \\
KNN-MSP & 0.74 & 0.70 & 1.1 & 0.97 & 0.71 & 0.82 & 0.54 & 1.4 & 0.78 & 0.81 \\
PROSER & 0.73 & 0.67 & 1.3 & 0.75 & 0.58 & 0.81 & 0.60 & 1.6 & 0.64 & 0.71 \\
OCC (LOF, centered) & 0.78 & 0.03 & 4.7 & 0.07 & 0.71 & 0.82 & 0.05 & 4.1 & 0.08 & 0.81 \\
OCC (OCSVM $\gamma{=}20$) & 0.93 & 0.39 & 2.1 & 0.69 & 0.71 & 0.95 & 0.39 & 2.2 & 0.56 & 0.81 \\
\midrule
Naive (GT constant) & 0.50 & 0.37 & 1.0 & 0.95 & 0.71 & 0.50 & 0.06 & 1.0 & 0.04 & 0.81 \\
\bottomrule
\end{tabular}
\end{table}

\FloatBarrier

\emph{Nearly closed-set regime.} In the comparisons above, $7\%$ to $37\%$ of the test identities are new. Figures~\ref{fig:app-openset-real-large} and~\ref{fig:app-openset-dp-small-theta} repeat the main comparison in settings where new labels are rare: CelebA with larger reference samples, and synthetic data with small $\theta$. There, all methods produce nearly identical prediction sets, and CGTC+ coincides with the naive baseline. The estimated missing mass is so small that the joker is never included, so new labels are left uncovered, as the coverage guarantee allows. Coverage is at the nominal level throughout, except for Recal-OCC (LOF) at small $\theta$.
Thus, the methods differ only when new labels are frequent.

\begin{figure}[!htb]
    \centering
    \includegraphics[width=\textwidth]{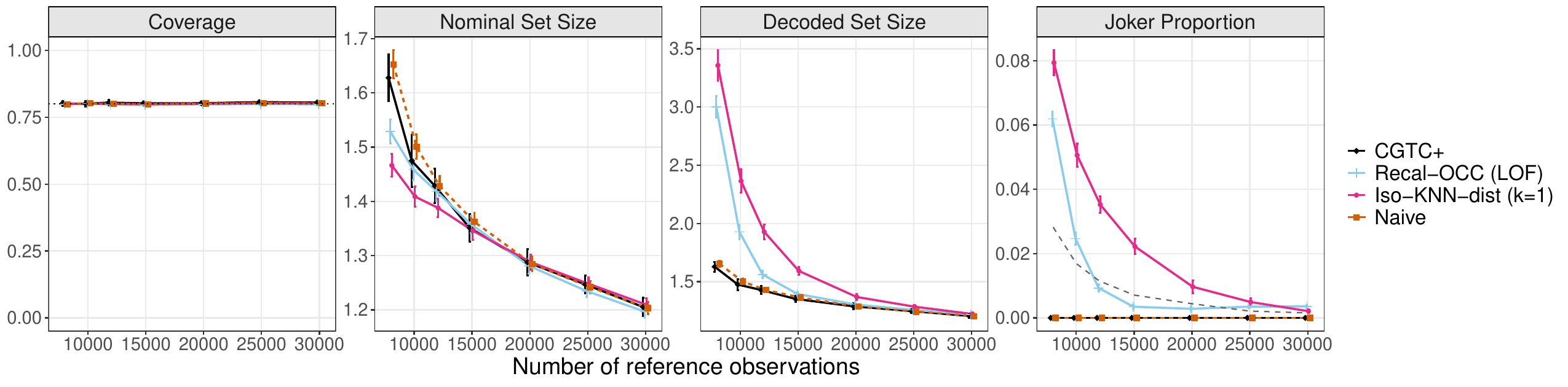}
    \caption{Same as Figure~\ref{fig:app-openset-real-main}, but for reference sample sizes from $8{,}000$ to $30{,}000$ images. Points are slightly offset horizontally, because CGTC+ and the naive baseline coincide over the whole range.}
    \label{fig:app-openset-real-large}
\end{figure}

\begin{figure}[!htb]
    \centering
    \includegraphics[width=\textwidth]{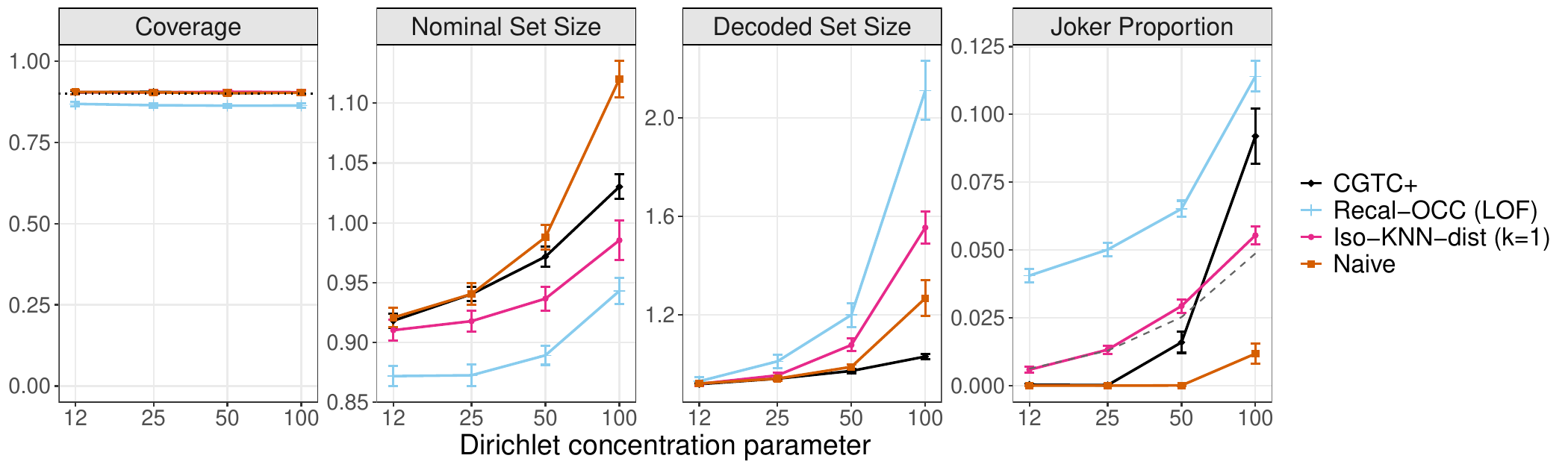}
    \caption{Same as Figure~\ref{fig:app-openset-real-main}, but on synthetic data from the Dirichlet process model ($n = 2{,}000$, $\alpha = 0.1$), as a function of $\theta$.}
    \label{fig:app-openset-dp-small-theta}
\end{figure}

\emph{Summary.} Open-set classifiers can rank inputs by novelty, but they do not estimate well the probability that a label is new. The Good--Turing estimate can fix this. Even after recalibration, however, direct conformalization produces prediction sets that are much larger than those obtainable with CGTC+, because its unknown class includes all labels observed only in the calibration set. Moreover, direct conformalization gives no control over the trade-off between prediction set size and use of the joker symbol.

\end{document}